\documentclass{article}

\usepackage[nonatbib,final]{neurips_2021}

\usepackage[utf8]{inputenc} %
\usepackage[T1]{fontenc}    %
\usepackage{hyperref}       %
\usepackage{url}            %
\usepackage{booktabs}       %
\usepackage{amsfonts}       %
\usepackage{amsmath}
\usepackage{amssymb}
\usepackage{nicefrac}       %
\usepackage{microtype}      %
\usepackage[table]{xcolor}         %
\usepackage{wrapfig}

\DeclareMathOperator{\Cat}{Cat}

\DeclareMathOperator{\KL}{KL}

\usepackage{bm}
\usepackage{amsthm}

\newtheorem{theorem}{Theorem}
\usepackage{tikz}
\usetikzlibrary{bayesnet}
\usepackage{graphicx}
\usepackage{caption}
\usepackage{subcaption}
\usepackage{paralist}

\usepackage{array,multirow,graphicx}
\usepackage{float}

\DeclareMathOperator{\expect}{\mathbb{E}}

\DeclareMathOperator{\ELBO}{\mathcal{L}}

\newcommand*\intd{\mathop{}\!\mathrm{d}}
\newcommand*\circled[1]{\tikz[baseline=(char.base)]{
            \node[shape=circle,draw,inner sep=1pt] (char) {#1};}}

\renewcommand{\vector}[1]{\boldsymbol{\mathbf{#1}}}
\renewcommand{\v}{\vector}
\newcommand{\vs}[1]{\vec{#1}}
\newcommand*\vv[1]{\vec{\v{#1}}}
\newcommand*\vvs[1]{\vec{#1}}
\newcommand{\B}[1]{\boldsymbol{#1}}

\usepackage{enumitem}

\usepackage{floatflt}

\title{Multi-Facet Clustering Variational Autoencoders}  %
\author{  %
Fabian Falck \textbf{\thanks{Equal contribution.}} $^{\ ,1,5}$ \,\,\, Haoting Zhang $^{*,2,5}$ \,\,\, Matthew Willetts $^{3,6}$ \\
\bf George Nicholson $^1$ \,\,\, Christopher Yau $^{4,5,6}$ \,\,\, Chris Holmes $^{1,5,6}$ \\
$^1$University of Oxford\, $^2$University of Cambridge\, $^3$University College London \\
$^4$University of Manchester\, $^5$Health Data Research UK\, $^6$The Alan Turing Institute \\
\texttt{fabian.falck@stats.ox.ac.uk, hz381@cl.cam.ac.uk, mwilletts@turing.ac.uk,} \\
\texttt{george.nicholson@stats.ox.ac.uk, cyau@turing.ac.uk, cholmes@stats.ox.ac.uk} \\
}

\begin{document}

\maketitle

\begin{abstract}
Work in deep clustering focuses on finding a \textit{single} partition of data. 
However, high-dimensional data, such as images, typically feature \textit{multiple} interesting characteristics one could cluster over.
For example, images of objects against a background could be clustered over the shape of the object and separately by the colour of the background.
In this paper, we introduce \textit{Multi-Facet Clustering Variational Autoencoders (MFCVAE)}, a novel class of variational autoencoders with a hierarchy of latent variables, each with a Mixture-of-Gaussians prior, that learns multiple clusterings simultaneously, and is trained fully unsupervised and end-to-end.
MFCVAE uses a progressively-trained ladder architecture which leads to highly stable performance.
We provide novel theoretical results for optimising the ELBO analytically with respect to the categorical variational posterior distribution, correcting earlier influential theoretical work.  %
On image benchmarks, we demonstrate that our approach separates out and clusters over different aspects of the data in a disentangled manner.
We also show other advantages of our model: the compositionality of its latent space and that it provides controlled generation of samples.

\end{abstract}

\section{Introduction}
\label{sec:Introduction}

\textit{Clustering} is the task of finding structure by partitioning samples in a finite, unlabeled dataset according to statistical or geometric notions of similarity~\cite{murphy2012machine, 2005clusteringSurvey, hansen1997cluster}.%
For example, we might group items along axes of empirical variation in the data, or maximise internal homogeneity and external separation of items within and between clusters with respect to a specified distance metric.
The choice of similarity measure and how one consequently validates clustering quality is fundamentally a subjective one: it depends on what is useful for a particular task~\cite{2005clusteringSurvey, von2012clustering}.
In this work, we are interested in uncovering abstract, latent characteristics/facets/aspects/levels of the data to understand and characterise the data-generative process. 
We further assume a fully exploratory, unsupervised setting without prior knowledge on the data, which could be exploited while fitting the clustering algorithm, and in particular without given ground-truth partitions at training time.

When being faced with high-dimensional data such as images, speech or electronic health records, items typically have more than one abstract characteristic.
Consider the example of the MNIST dataset \cite{lecun2010mnist}:
MNIST images possess at least two such characteristics: The digit class, which might impose the largest amount of statistical variation, and the style of the digit (e.g. stroke width).  %
This naturally raises a question: By which characteristic is a clustering algorithm supposed to partition the data?
In MNIST, both digit class and (the sub-categories of) style would be perfectly reasonable candidates to answer this question. 
In our exploratory setting described above, there is not one ``correct'' partition of the data.

Deep learning based clustering algorithms, so-called \textit{deep clustering}, were particularly successful in recent years in dealing with high-dimensional data by compressing the inputs into a lower-dimensional latent space in which clustering is computationally tractable \cite{min2018survey, aljalbout2018taxonomy}.
However, almost all of these deep clustering algorithms find only a \textit{single} partition of the data, typically the one corresponding to the given class label in a supervised dataset~\cite{dec, dcn, vade, imsat, spectralnet, yang2019deep, clustergan, acol-gar}. 
When evaluating their model, said approaches validate clustering performance by treating the one supervision label (e.g. digit class in the case of MNIST) as the de-facto ``ground truth clustering''.
We argue that restricting our view to a single facet $C_1$ rather than all or at least multiple facets $(C_1, C_2, \dots, C_J)$ is an arbitrary, incomplete choice of formulating the problem of clustering a high-dimensional dataset.

To this end, we propose \textit{Multi-Facet Clustering Variational Autoencoders (MFCVAE)}, a principled, probabilistic model which finds multiple characteristics of the data simultaneously through its multiple Mixtures-of-Gaussians (MoG) prior structure.  %
Our contributions are as follows: 
(a) \textit{Multi-Facet Clustering Variational Autoencoders (MFCVAE)}, a novel class of probabilistic deep learning models for unsupervised,  multi-facet clustering in high-dimensional data that can be optimised end-to-end. 
(b) Novel theoretical results for the optimisation of the corresponding ELBO, correcting and extending an influential, related paper for the single-facet case.
(c) Demonstrating MFCVAE's stable empirical performance in terms of multi-facet clustering of various levels of abstraction, compositionality of facets, generative, unsupervised classification, and diversity of generation.  %

\begin{figure}[t]
    \centering
    \includegraphics[width=\linewidth]{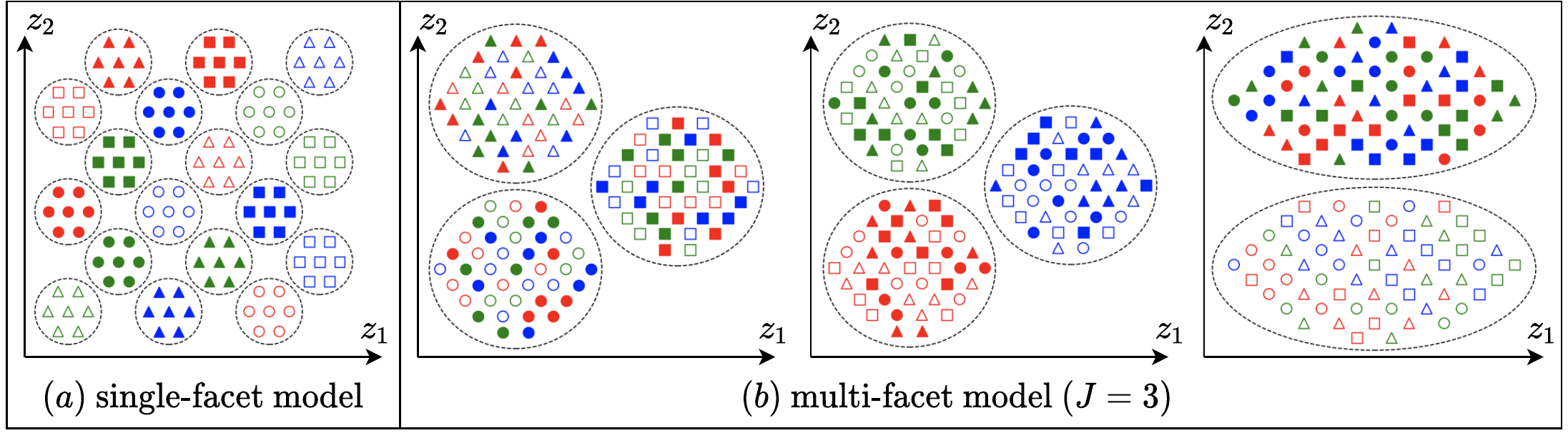}
    \caption{
    Latent space of a (a) single-facet model and a (b) multi-facet model ($J=3$) with two dimensions ($z_1$, $z_2$) per facet.
    Both models perfectly separate the abstract characteristics of the data. However, the multi-facet model disentangles them into three sensible partitions (one per facet) and its required clusters scale linearly as opposed to exponentially w.r.t. the number of aspects in the data.
    }
    \label{fig:intuition}
\end{figure}

\section{Multi-facet clustering} 
\label{sec:Multi-facet clustering}

High-dimensional data are inherently structured according to a number of abstract characteristics, and in an exploratory setting, it is clear that arbitrarily clustering by one of them is insufficient. 
However, the question remains whether these multiple facets should also be explicitly \textit{represented} by the model. 
In particular, one might argue that a single partition could be used to represent all cross-combinations\footnote{Note that in practice, not all cross-combinations of facets might be present. 
For example, in a dataset like MNIST, one might not observe `right-tilted zeros', even though we observe `right-tilted' digits and `zeros'.} of facets $\mathcal{C}= C_1 \times C_2 \times \dots \times C_J$ where $C_j = \{1,2,\ldots, K_j\}$, as in Fig.~\ref{fig:intuition} (a).
In this work, we explain that explicitly representing and clustering by multiple facets, as we do in MFCVAE and illustrated in Fig.~\ref{fig:intuition} (b), has the following four properties that are especially desirable in an unsupervised learning setting: 

\textbf{(a) Discovering a multi-facet structure. }
We adopt a probabilistically principled, unsupervised approach, specifying an independent, multiple Mixtures of Gaussians (MoG) prior on the latent space.
This induces a disentangled representation across facets, meaning that in addition to examples assigned to certain clusters being homogeneous, the facets (and their corresponding clusters) represent different abstract characteristics of the data (such as digit class or digit style).
Because of this multi-facet structure, the total number of clusters required to represent a given multi-partition structure of the data scales \textit{linearly} w.r.t. the number of data characteristics. 
In comparison, the number of clusters required in a single-facet model scales \textit{exponentially} (see Fig.~\ref{fig:intuition}, and Appendix~\ref{app:MixtureMoG} for details).
 
\textbf{(b) Compositionality of facets. }
A multi-facet model has a compositional advantage: different levels of abstraction of the data are represented in separate latent variables. 
As we will show, this allows qualitatively diverse characteristics to be meaningfully combined.

\textbf{(c) Generative, unsupervised classification. }
Our method joins a myriad of single-facet clustering models in being able to accurately identify known class structures given by the label in standard supervised image datasets.
However, in contrast to previous work, we are also able find interesting characteristics in other facets with homogeneous clusters. %
We stress that while we compare generative classification performance against other models to demonstrate statistical competitiveness, this task is \textit{not} the main motivation for our fully unsupervised model.

\textbf{(d) Diversity of generated samples. }In a generative sense, the structure of the latent space allows us to compose new, synthetic examples by a set of $J$ pairs of (continuous, discrete) latent variables. 
We can in particular intervene on each facet separately. %
This yields a rich set of options and fine-grained control for interventions and the diversity of generated examples. 

We illustrate these four properties in our experiments in Section~\ref{sec:Experiments}.

\section{Multi-Facet Clustering Variational Autoencoders}
\label{sec:Multi-Facet Clustering Variational Autoencoders}

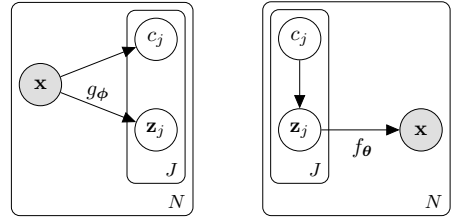
\begin{wrapfigure}{r}{0.47\linewidth}
  \centering
  \vspace{\dimexpr0.1\baselineskip-\topskip}
    \scalebox{0.8}{
    \beginpgfgraphicnamed{recognition_model}
    \begin{tikzpicture}
      \node[obs] (x) {$\mathbf{x}$}; %
      \node[latent, below right=2 of x, yshift=+1.168cm] (z) {$\mathbf{z}_j$}; %
      \node[latent, above right=2 of x, yshift=-1.168cm] (c) {$c_j$};
      \edge[] {x} {z}; %
      \edge[] {x} {c}; %
      \plate {zc} {%
        (z)(c)
      } {$J$} ;
      \plate{} {%
        (z)(c)(x)
        (zc.north)(zc.south)(zc.east)
      }{$N$} ;
      \node[const, right=.4 of x, yshift=-0.15cm] (g) {$g_{\bm{\phi}}$} ; %
    \end{tikzpicture}
    \endpgfgraphicnamed
    }
    \hspace{0.6cm}
    \scalebox{0.8}{
    \beginpgfgraphicnamed{generative_model}
    \begin{tikzpicture}
      \node[latent] (z) {$\mathbf{z}_j$}; %
      \node[latent, above=0.8 of z] (c) {$c_j$}; %
      \node[obs, right=1.3 of z] (x) {$\mathbf{x}$}; %
      \edge[] {c} {z}; %
      \edge[] {z} {x}; %
      \plate {zc} {%
        (z)(c)
      } {$J$} ;
      \plate[xshift=0cm, yshift=0.0cm] {} {
        (z)(c)(x)
        (zc.north)(zc.south)(zc.west)
      }{$N$} ;
      \node[const, right=.5 of z, yshift=-0.3cm] (f) {$f_{\bm{\theta}}$} ; %
    \end{tikzpicture}
    \endpgfgraphicnamed
    }
  \caption{
  Graphical model of MFCVAE. [Left] Variational posterior, $q_{\phi}(\vv{z},\v{c} | \v{x})$. [Right] Generative model, $p_{\theta}(\v{x},\vv{z},\v{c})$.
  }
  \label{fig:graph_model}
\end{wrapfigure}

Our model comprises $J$ latent facets, each learning its own unique clustering of samples via a Mixture-of-Gaussians (MoG) distribution:
\begin{equation}
    c_j \sim \Cat (\v{\pi}_j),\quad
    \v{z}_j \mid c_j \sim \mathcal{N}(\bm{\mu}_{c_j}, \bm{\Sigma}_{c_j})
    \label{eqn:gen1}
\end{equation}
where $\v{\pi}_j$ is the $j$th facet's $K_j$-dimensional vector of mixing weights,
and $(\bm{\mu}_{c_j}, \bm{\Sigma}_{c_j})$ are the mean and covariance of the $c_j$th mixture component in facet $j$
($\bm{\Sigma}_{c_j}$ can be either diagonal or full).

The multi-facet generative model (Fig.~\ref{fig:graph_model} [Right]) is thus structured as
\begin{align}
    p_{{\theta}}(\v{x},\vv{z},\v{c}) = p_{{\theta}}(\v{x} | \vv{z})p_{{\theta}}(\vv{z}|\v{c})p_{{\theta}}(\v c) = p_{{\theta}}(\v{x} | \vv{z})\prod_{j=1}^J p_{{\theta}}(\v{z}_j|c_j)p_{{\theta}}(c_j),
        \label{eqn:facet_independence_assumption}
\end{align}
where $\v c = \{c_1,c_2,...,c_J\}$, $\vv{z}=\{\v{z}_1,\v{z}_2,...,\v{z}_J\}$, and $p_{\theta}(\v{x} | \vv{z})$ is a user-defined likelihood model, for example a product of Bernoulli or Gaussian distributions, which is parameterised with a deep neural network $f(\vv{z}; {\theta})$. 
Importantly, this structure in Eq.~\eqref{eqn:facet_independence_assumption} encodes prior independence across facets, i.e.\  $p_{{\theta}}(\vv{z},\v{c})=\prod_j p_\theta(\v{z}_{j}, c_{j})$, thereby encouraging facets to learn clusterings that span distinct subspaces of $\vv{z}$.
The overall marginal prior $p_\theta(\vv{z})$ can be interpreted as a product of independent MoGs.

\subsection{VaDE tricks}
\label{sec:vade_tricks}

To train this model, we wish to optimise the evidence lower bound (ELBO) of the data marginal likelihood using an amortised variational posterior $q_\phi(\vv{z}, \v{c} | \v{x})$ (Fig.~\ref{fig:graph_model} [Left]), parameterised by a neural network $g(\v{x}; \phi)$, within which we will perform Monte Carlo (MC) estimation where necessary to approximate expectations  %
\begin{equation}
  \log p(\mathcal{D}) \geq \ELBO(\mathcal{D};\theta,\phi) = \expect_{\v x \sim \mathcal{D}}\left[\mathbb{E}_{q_\phi(\vv{z}, \v{c} | \v{x})}[\log \frac{p_\theta(\v{x}, \vv{z}, \v{c})}{q_\phi(\vv{z}, \v{c} | \v{x})}]\right].
  \label{eqn:jensen}
\end{equation}

What should we choose for $q_\phi(\vv{z}, \v{c} | \v{x})$?
Training deep generative models with discrete latent variables can be challenging, as reparameterisation tricks so far developed, such as the Gumbel-Softmax trick \cite{maddison2016concrete,jang2016categorical}, necessarily introduce bias into the optimisation, and become unstable when a discrete latent variable has a high cardinality.
Our setting where we have multiple discrete latent variables is even more challenging.
First, the bias from using the Gumbel-Softmax trick compounds when there is a hierarchy of dependent latent variables, leading to poor optimisation \cite{lievin2019towards}.
Second, we cannot necessarily avail ourselves of advances in obtaining good estimators for discrete latent variables as either they do not carry over to the hierarchical case \cite{grathwohl2018backprop}, or are restricted to binary latent variables \cite{pervez2020low}.
Third, we wish for light-weight optimisation, avoiding the introduction of additional neural networks whenever possible as this simplifies both training and neural specification.

Thus, we sidestep these problems, bias from relaxations of discrete variables \textit{and} the downsides of additional amortised-posterior neural networks for the discrete latent variables, by developing the hierarchical version of the \textit{VaDE trick}.
This trick was first developed for clustering VAEs with a \textit{single} Gaussian mixture in the generative model~\cite{vade}.
Informally, the idea (for a single-facet model) is to define a Bayes-optimal posterior for the discrete latent variable using the responsibilities of the constituent components of the mixture model; these responsibilities are calculated using samples taken from the amortised posterior for the continuous latent variable. 

Estimating the ELBO for models of this form does not require us to take MC samples from discrete distributions---the data likelihood is conditioned only on the continuous latent variable $\vv{z}$, which we sample using the reparameterization trick~\cite{kingma2013auto}, and the posterior for $\vv{z}$ is conditioned only on $\v x$.
Thus, when calculating the ELBO, we can cheaply marginalise out discrete latent variables where needed.
In other words, we do not have to perform multiple forward passes through the decoder as neither it nor the $\vv z$ samples we feed it depend on $c$.

As it is fundamental to our method, we now briefly recapitulate the original VaDE trick for VAEs with a single latent mixture (correcting a misapprehension in the original form of this idea) and will then cover our hierarchical extension\footnote{We note that the original VaDE paper, besides the misapprehension discussed in Section~\ref{sec:vade_tricks} and Appendix~\ref{app:misapp}, proposed a highly complex training algorithm with various pre-training heuristics which we significantly simplified while maintaining or increasing performance (details in Appendix~\ref{app:diff_to_vade}).}.

\textbf{Single-Facet VaDE Trick:}
Consider a single facet model, so the generative model is $p_\theta(\v x, \v z, c) = p_\theta(\v x | \v z)p_\theta(\v z |c)p_\theta(c)$.
Introduce a posterior $q_\phi(\v{z}, c | \v{x}) = q_\phi(\v z|\v x)q_\phi(c| \v x)$ where $q_\phi(\v z|\v x)$ is a multivariate Gaussian with diagonal covariance.
The ELBO for this model for one datapoint is
\begin{equation}
\ELBO(\v x;\theta,\phi) =\mathbb{E}_{q_\phi(\v{z}, c | \v{x})}[\log \frac{p_\theta(\v x | \v z)p_\theta(\v z |c)p_\theta(c)}{q_\phi(\v z|\v x)q_\phi(c|\v x)}]= \mathbb{E}_{q_\phi(\v{z}, c | \v{x})}[\log \frac{p_\theta(\v x | \v z)p_\theta(\v z )p_\theta(c|\v z)}{q_\phi(\v z|\v x)q_\phi(c|\v x)}],
\label{eq:elbo_initial}
\end{equation}
where we have chosen to rewrite the generative model factorisation, $p_\theta(\v z) = \sum_c p_\theta(\v z | c)p_\theta(c)$ is the marginal mixture of Gaussians, and $p_\theta(c|\v z)= p_\theta(\v z | c)p_\theta(c)/p_\theta(\v{z})$ is the Bayesian posterior for $c$.

Expanding out the ELBO, we get
\begin{equation}
\label{eq:expand_elbo_single_vade}
\ELBO(\v x;\theta,\phi) =\mathbb{E}_{q_\phi(\v{z} | \v{x})}\log p_\theta(\v x | \v z) - \KL\left[q_\phi(\v z | \v x)||p_\theta(\v z)\right] - \expect_{q_\phi(\v z | \v x)}\KL\left[q_\phi(c|\v x)||p_\theta(c|\v z)\right].
\end{equation}
We can \textit{define} $q_\phi(c|\v x)$ such that
$\expect_{q_\phi(\v z | \v x)}\KL\left[q_\phi(c|\v x)||p_\theta(c|\v z)\right]$ is \textit{minimal}, by construction, which is the case if we choose $q_\phi(c|\v x)\propto\exp\left(\expect_{q_\phi(\v z | \v x)}\log p_\theta(c|\v z)\right)$ as we will show in Theorem \ref{thm:1}.
This means that we can simply use samples from the posterior for $\v z$ to define the posterior for $c$, using Bayes' rule within the latent mixture model.

\textbf{Remark:} We note, however, that in the original description of this idea in \cite{vade}, it was claimed that $\expect_{q_\phi(\v z | \v x)}\KL\left[q_\phi(c|\v x)||p_\theta(c|\v z)\right]$ could, in general, be set to \textit{zero}, which is not the case.
Rather, this $\KL$ can be minimised, in general, to a \textit{non-zero} value.
We discuss this misapprehension in more detail and why the empirical results in~\cite{vade} are still valid in Appendix~\ref{app:misapp}.
\begin{theorem}(Single-Facet VaDE Trick)
\label{thm:1}
For any probability distribution $q_\phi(\v z | \v x)$, the distribution $q_\phi(c|\v x)$ that minimises $\expect_{q_\phi(\v z | \v x)}\KL\left[q_\phi(c|\v x)||p_\theta(c|\v z)\right]$ in (\ref{eq:expand_elbo_single_vade}) is
\begin{align}
\underset{q_\phi(c|\v x)}{\mathrm{argmin}} \expect_{q_\phi(\v z | \v x)}\KL\left[q_\phi(c|\v x)||p_\theta(c|\v z)\right] &= \v{\pi}(c| q_{\phi}(\v{z} | \v{x})) 
\end{align}
with the minimum value attained being 
\begin{align}
\underset{q_\phi(c|\v x)}{\mathrm{min}} \expect_{q_\phi(\v z | \v x)}\KL\left[q_\phi(c|\v x)||p_\theta(c|\v z)\right] &= -\log Z(q_{\phi}(\v{z} | \v{x})) \\
   \text{where} \hspace*{1.2cm} \v{\pi}(c| q_{\phi}(\v{z} | \v{x})) &:= \frac{\exp \left( \mathbb{E}_{q_{\phi}(\v{z} | \v{x})} \log p(c | \v{z})\right)}{Z(q_{\phi}(\v{z} | \v{x}))}\ \ \text{for}\ c =1, \ldots, K  \\
    Z(q_{\phi}(\v{z} | \v{x})) & := \sum_{c=1}^{K}\exp \left( \mathbb{E}_{q_{\phi}(\v{z} | \v{x})} \log p(c | \v{z})\right)\ .
\end{align}
\textit{Proof:}\hspace{3mm}See Appendix~\ref{app:proof1}.$\hspace{1mm}\square$
\end{theorem}
\textbf{Multi-facet VaDE Trick: }
In this work, we consider the case of having $J$ facets, each with its own pair of variables $(\v z_j, c_j)$. 
Perhaps surprisingly, we do \textit{not} have to make a mean-field assumption \textit{between} the $J$ facets for $\v c$ once we have made one for $\vv z$.
In other words, once we have chosen that $q_\phi(\vv{z},\v{c}|\v{x})=q_\phi(\v{c}|\v{x})\prod_{j=1}^J q_\phi(\v{z}_j|\v{x})$, where $q_\phi(\v z_j|\v x)$ is defined to be a multivariate Gaussian with diagonal covariance for each $j$, the optimal $q_\phi(\v{c}|\v{x})$ similarly factorises\footnote{We also provide the VaDE trick for the general form of the posterior for $\vv{z}$, i.e. without assuming the factorisation $q_\phi(\vv{z}|\v{x})=\prod_{j=1}^J q_\phi(\v{z}_j|\v{x})$, in Appendix~\ref{app:proof3}.}.
We formalise this:
\begin{theorem}(Multi-Facet VaDE Trick for factorized $q_\phi(\vv z | \v x)$, $p(\vv z, \v c)$)
\label{thm:2}
For any factorized probability distribution $q_\phi(\vv z | \v x) = \prod_j q_\phi(\v z_j | \v x)$, the distribution $q_\phi(\v c|\v x)$ that minimises $\expect_{q_\phi(\vv z | \v x)}\KL\left[q_\phi(\v c|\v x)||p_\theta(\v c|\vv z)\right]$ under factorized prior $p(\vv z, \v c) = \prod_j p(\v z_j, c_j)$ of (\ref{eqn:facet_independence_assumption}) is
\begin{align}
    \underset{q_\phi(\v c|\v x)}{\mathrm{argmin}} \expect_{q_\phi(\vv z | \v x)}\KL\left[q_\phi(\v c|\v x)||p_\theta(\v c|\vv z)\right] &= \prod_j \v{\pi}_j( c_j | q_\phi(\v z_j | \v x))
\end{align}
where the minimum value is attained at
\begin{align}
\underset{q_\phi(\v c|\v x)}{\mathrm{min}} \expect_{q_\phi(\vv z | \v x)}\KL\left[q_\phi(\v c|\v x)||p_\theta(\v c|\vv z)\right] &= -\sum_j \log Z_j(q_\phi(\v z_j | \v x)) \\
    \text{where} \hspace*{1.2cm}  \v{\pi}_j( c_j | q_\phi(\v z_j | \v x)) &:= \frac{\exp(\expect_{q_\phi(\v z_j | \v x)} \log p_\theta(c_j|\v z_j))}{Z_j(q_\phi(\v z_j | \v x))}  \text{, for}\ c_j =1, \ldots, K_j \label{eq:opt_q}\\
    Z_j(q_\phi(\v z_j | \v x)) & := \sum_{c_j=1}^{K_j}\exp(\expect_{q_\phi(\v z_j | \v x)} \log p_\theta(c_j|\v z_j))\ .
\end{align}
\textit{Proof:}\hspace{3mm}See Appendix~\ref{app:proof2}.$\hspace{1mm}\square$
\end{theorem}
Note that we use Eq.~\eqref{eq:opt_q} as the probability distribution of assigning input $\v x$ to clusters of facet $j$.

Armed with these theoretical results, we can now write the ELBO for our model, with the optimal posterior for $\v c$, in a form that trivially admits stochastic estimation and does not necessitate extra recognition networks for $\v c$,
\begin{align}
    \ELBO^{\mathrm{MFCVAE}}(\mathcal{D};\theta,\phi) =& \expect_{\v x \sim \mathcal{D}}\Big[ \expect_{q_\phi({\vv{z}|\v{x})}} \log p_\theta(\v x | \vv z ) \nonumber \\
    &- \sum_{j=1}^J\left[ \expect_{q_\phi(c_j|\v x)}\KL(q_\phi(\v{z}_j|\v x) ||p_\theta(\v{z}_j|c_j)) + \KL(q_\phi(c_j|\v x)||p(c_j))\right]\Big]
    \label{eqn:mfc_elbo}
\end{align}
where the optimal $q_\phi(c_j|\v x)$ is given by Eq.~\eqref{eq:opt_q} for each $j$.

To obtain the posterior distributions for $\v c$,
we take MC samples from $q_\phi(\vv{z}|\v{x})$ and use these to construct the posterior as in Eq.~\eqref{eq:opt_q}. 
We found one MC sample ($L=1$; for each facet and for each $\v x$) to be sufficient. We derive the complete MC estimator which we use as the loss function of our model and ablations on two alternative forms in Appendix~\ref{app:mc_estimator}.

\subsection{Neural implementation and training algorithm}  
\label{sec:Neural implementation and training algorithm}
It is worth pausing here to consider what neural architecture best suits our desire for learning multiple disentangled facets, and then further how we can best train our model to robustly elicit from it well-separated facets.
In the introduction, we discussed the different plausible ways to cluster high-dimensional data, such as in MNIST digits by stroke thickness and class identity. 
These different aspects intuitively correspond to different levels of abstraction about the image.
It is thus natural that these levels would be best captured by different depths of the neural networks in each amortised posterior. 
These ideas have motivated the use of \textit{ladder networks} in deep generative models that aim to learn different facets of the input data into different layers of latent variables. 
Here, we take inspiration from \textit{Variational Ladder Autoencoders (VLAEs)}~\cite{vlae}: %
A VLAE architecture has a deterministic ``backbone'' in both the recognition and generative model. 
The different layers of latent variables branch out from these at different depths along. 
This inductive bias naturally leads to stratification and does so without having to bear the computational cost of training a completely separate encoder (say) for each layer.  %
Here, we use this ladder architecture for MFCVAE, as illustrated in Fig.~\ref{fig:ladder_arch}, and refer to Appendix~\ref{app:Neural architectures and Variational Ladder Autoencoder} for further implementation details.

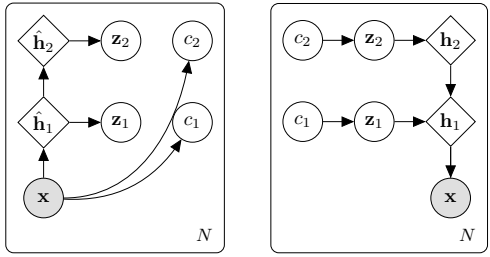
\begin{wrapfigure}{r}{0.5\linewidth}
  \centering
    \scalebox{0.75}{
        \beginpgfgraphicnamed{recognition_model_ladder}
        \begin{tikzpicture}
            \node[obs] (x) {$\v{x}$} ; %
            \node[det, above=0.52cm of x,minimum size=25pt] (h1) {$\hat{\v{h}}_1$} ; %
            \node[det, above=0.52cm of h1,minimum size=25pt] (h2) {$\hat{\v{h}}_2$} ; %
            \node[latent, right=0.55cm of h1] (z1) {$\v{z}_1$} ; %
            \node[latent, right=0.55cm of h2] (z2) {$\v{z}_2$} ; %
            \node[latent, right=0.55cm of z1] (c1) {$c_1$} ; %
            \node[latent, right=0.55cm of z2] (c2) {$c_2$} ; %
        	\edge {x} {h1} ; %
        	\edge {h1} {h2} ; %
        	\edge {h1} {z1} ; %
        	\edge {h2} {z2} ; %
        	\draw [->] (x) to[bend right] (c1);
        	\draw [->] (x) to[out=0,in=-100] (c2);
            \plate[inner sep=0.2cm] {plate1} {(x) (c1) (c2) (z1) (z2) (h1) (h2) } {\scalebox{1}{{$N$}}};
        \end{tikzpicture}
        \endpgfgraphicnamed
    }
    \hspace{0.3cm}
    \scalebox{0.75}{
        \beginpgfgraphicnamed{generative_model_ladder}
        \begin{tikzpicture}
            \node[obs] (x) {$\v{x}$} ; %
            \node[det, above=0.55cm of x,minimum size=25pt] (h1) {$\v{h}_1$} ; %
            \node[det, above=0.55cm of h1,minimum size=25pt] (h2) {$\v{h}_2$} ; %
            \node[latent, left=0.55cm of h1] (z1) {$\v{z}_1$} ; %
            \node[latent, left=0.55cm of h2] (z2) {$\v{z}_2$} ; %
            \node[latent, left=0.55cm of z1] (c1) {$c_1$} ; %
            \node[latent, left=0.55cm of z2] (c2) {$c_2$} ; %
        	\edge {c2} {z2}; %
        	\edge {c1} {z1} ;
        	\edge {z2} {h2}; %
        	\edge {z1} {h1} ;
        	\edge {h2} {h1} ; %
        	\edge {h1} {x};
            \draw (h1)  --  (x) ;
            \plate[inner sep=0.2cm] {plate1} {(x) (c1) (c2) (z1) (z2) (h1) (h2)}
            {\scalebox{1}{{$N$}}};
        \end{tikzpicture}
        \endpgfgraphicnamed
    }
  \caption{
  Ladder-MFCVAE architecture. [Left] Variational posterior. [Right] Generative model.  %
  }
  \label{fig:ladder_arch}
\end{wrapfigure}

Further, we found \textit{progressive training}~\cite{ProVLAE}, previously shown to help VLAEs learn layer-by-layer disentangled representations, to be of great use in making each facet consistently represent the same aspects of data.
The general idea of progressive training is to start with training a single facet (typically the one corresponding to the deepest recognition and generative neural networks) for a certain number of epochs, and progressively and smoothly loop in the other facets one after the other. 
We discuss the details of our progressive training schedule in Appendix~\ref{app:Progressive training algorithm}. 
We find that both the VLAE architecture and progressive training are jointly important to stabilise training and get robust qualitative and quantitative results as we show in Appendix~\ref{app:stability}.

\section{Experiments}
\label{sec:Experiments}

In the following, we demonstrate the usefulness of our model and its prior structure in four experimental analyses: (a) discovering a multi-facet structure (b) compositionality of latent facets (c) generative, unsupervised classification, and (d) diversity of generated samples from our model. 
We train our model on three image datasets: MNIST~\cite{lecun2010mnist}, 3DShapes (two configurations)~\cite{3dshapes18} and SVHN~\cite{svhn}. 
We refer to Appendices~\ref{app:Experimental details} and~\ref{app:Additional experimental results} for experimental details and further results. 
We also provide our code implementing MFCVAE, using \textit{PyTorch Distributions}~\cite{pytorch}, and reproducing our results at \textcolor{blue}{\href{https://github.com/FabianFalck/mfcvae}{\url{https://github.com/FabianFalck/mfcvae}}}.

\subsection{Discovering a multi-facet structure}
\label{sec:Discovering a multi-facet structure}

We start by demonstrating that our model can discover a multi-facet structure in data. 
Fig.~\ref{fig:orth_cluster_examples} visualises input examples representative of clusters in a two-facet ($J=2$) model.
For each facet $j$, input examples $\B{x}$ with latent variable $\B{z}_j$ are assigned to latent cluster $c_j = \mathrm{argmax}_{c_j}  \v{\pi}_j( c_j | q_\phi(\v z_j | \v x))$ according to Eq.~\eqref{eq:opt_q}. 
Surprisingly, we find that we can represent the two most striking data characteristics---digit class and style (mostly in the form of stroke width, e.g. `bold', `thin') in MNIST, object shape and floor colour in 3DShapes (configuration 1), and digit class and background colour in SVHN---in two separate facets of the data. %
In each facet, clusters are homogeneous w.r.t. a value from the represented characteristic.
When comparing our results on MNIST with LTVAE~\cite{ltvae}, the model closest to ours in its attempt to learn a clustered latent space of multiple facets, LTVAE struggles to separate data characteristics into separate facets (c.f. \cite{ltvae} Fig.~5; in particular, both facets learn digit class, i.e. this characteristic is not properly disentangled between facets), whereas MFCVAE better isolates the two.

\begin{figure}[t]
    \centering
    \includegraphics[width=.9\linewidth]{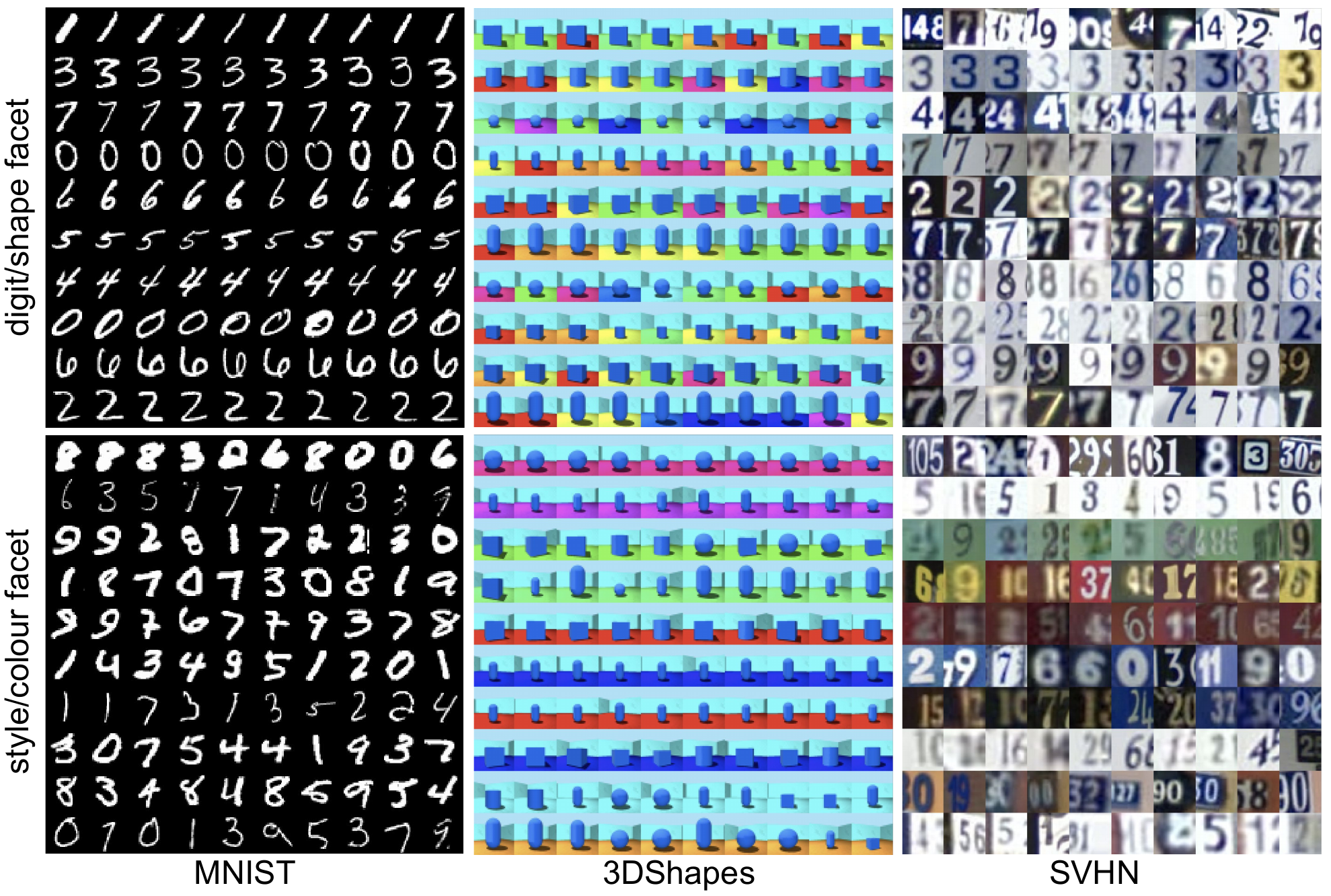}
    \caption{
    Input examples for clusters of MFCVAE with two-facets ($J=2$) trained on MNIST, 3DShapes and SVHN.
    Clusters (rows) in each facet $j$ are sorted in decreasing order by the average assignment probability of test inputs over each cluster.
    Inputs (columns) are sorted in decreasing order by their assignment probability $\mathrm{max}_{c_j}  \v{\pi}_j( c_j | q_\phi(\v z_j | \v x))$.
    We visualise the first 10 clusters and inputs from the test set (see Appendix~\ref{app:Discovering a multi-facet structure} for all clusters).
    }
    \label{fig:orth_cluster_examples}
\end{figure}

To quantitatively assess the degree of disentanglement in the learned multi-facet structure of our model, we perform a set of supervised experiments.
For each dataset, we formulate three classification tasks, for which we use latent embeddings $\v{z}_1$, $\v{z}_2$ and $\vv z$, respectively, sampled from their corresponding amortised posterior, as inputs, and the label present in the dataset (e.g. digit class in MNIST) as the target.
For each task and dataset, we train (on the training inputs) a multi-layer perceptron of one hidden layer with 100 hidden units and a ReLU activation, and an output layer followed by a softmax activation, which are the default hyperparameters in the Python package $\texttt{sklearn}$.
Table~\ref{tab:supervised_experiment} shows test accuracy of these experiments.
We find that the supervised classifiers predict the supervised label with high accuracy when presented with latent embeddings which we found to cluster the abstract characteristic corresponding to this label, or with the concatenation of both latent embeddings. 
However, when presented with latent embeddings corresponding to the ``non-label'' facet, the classifier should---if facets are strongly disentangled---not be presented with useful information to learn the supervised mapping, and this is indeed what we find, observing significantly worse performance.
This demonstrates the multi-facet structure of the latent space, which learns separate abstract characteristics of the data.

\begin{table}[t]
\caption{Supervised classification experiment to assess the disentanglement of MFCVAE's multi-facet structure on all three datasets. 
Values report test accuracy in \%. 
Error bars are the sample standard deviation across 3 runs.}
\medbreak
\label{tab:supervised_experiment}
\centering
\begin{tabular}{ ccccccc  } \toprule
  &   MNIST    & \multicolumn{2}{c}{3DShapes config. 1}    & \multicolumn{2}{c}{3DShapes config. 2}  & SVHN \\ 
  & digit class & object shape & floor colour & object shape & wall colour & digit class  \\ \midrule
$\v{z}_1$ &  17.34 (0.24) & 95.00 (0.45) & 20.00 (0.68) & 98.26 (0.16) & 73.40 (1.48) & 69.46 (0.36)  \\
$\v{z}_2$ &  94.95 (0.04) & 32.43 (1.38) & 100.00 (0.00) & 24.41 (1.34) & 100.00 (0.00) & 22.30 (0.16)  \\
$\vv{z}$ & 95.27 (0.07) & 95.18 (0.42) & 100.00 (0.00) & 98.19 (0.30) & 99.97 (0.06) & 70.39 (0.29)  \\ \bottomrule
\end{tabular}
\end{table}

\subsection{Compositionality of latent facets}
\label{sec:Compositionality of latent facets}

A unique advantage of the prior structure of MFCVAE compared to other unsupervised generative models, say a VAE with an isotropic Gaussian prior, is that it allows different abstract characteristics to be composed in the separated latent space.  
Here, we show how this enables interventions on a per-facet basis, illustrated with a two-facet model where style/colour is learned in one facet and digit/shape is learned in the other facet.
Let us have two inputs $\v{x}^{(1)}$ and $\v{x}^{(2)}$ assigned to two different style clusters according to Eq.~\eqref{eq:opt_q} (and two different digit clusters).
For both inputs, we obtain their latent representation $\tilde{\v{z}}_j$ as the modes of $q_\phi(\v{z}_j | \v{x})$, respectively.
Now, we swap the style/colour facet's representation, i.e. $\tilde{\v{z}}_1$ of both inputs for MNIST, and $\tilde{\v{z}}_2$ of both inputs for 3DShapes and SVHN, and pass these together with their unchanged digit/shape representation ($\tilde{\v{z}}_2$ for MNIST and $\tilde{\v{z}}_1$ for 3DShapes and SVHN) through the decoder $f(\vv{z}; {\theta})$ to get reconstructions $\hat{\v{x}}^{(1)} = f(\{ \tilde{\v{z}}_1^{(1)}, \tilde{\v{z}}_2^{(2)} \}; {\theta})$ and $\hat{\v{x}}^{(2)} = f(\{ \tilde{\v{z}}_1^{(2)}, \tilde{\v{z}}_2^{(1)} \}; {\theta})$ which we visualise in Fig.~\ref{fig:compositionality} (see Appendix~\ref{app:Compositionality of facets} for a more rigorous explanation of this swapping procedure).

Surprisingly, by construction of this intervention in our multi-facet model, we observe reconstructions that ``swap'' their style/background colour, yet in most cases preserve their digit/shape.
This intervention is successful across a wide set of clusters on MNIST and 3DShapes.
It works less so on SVHN where we hypothesise that this is due to the much more diverse dataset and (consequently) the model reaching a lower fit (see Section~\ref{sec:Generative classification}).
We show further examples including failure cases in Appendix~\ref{app:Compositionality of facets} which show that our model learns a multi-facet structure allowing complex inventions.

\begin{figure}[t]
    \centering
    \includegraphics[width=0.6\linewidth]{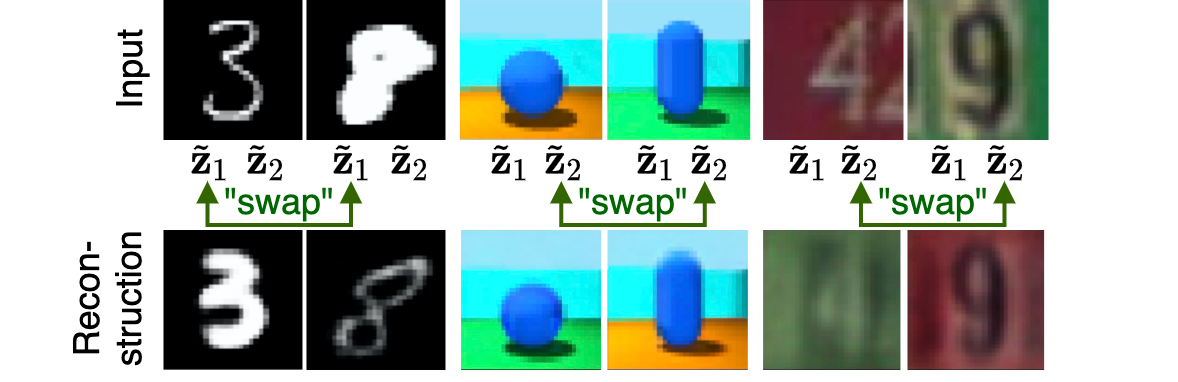}
    \caption{Reconstructions of two input examples when swapping their latent style/colour.}
    \label{fig:compositionality}
\end{figure}

\subsection{Generative, unsupervised classification}
\label{sec:Generative classification}

Recall our fully unsupervised, exploratory setting of clustering where the goal is to identify and characterise multiple meaningful latent structures \textit{de novo}.
In practice, we have no ground-truth data partition---if labels were available, the task would be better formulated as a supervised classification in the first place. 
That said, it is often reasonable to assume that the class label in a supervised dataset represents a semantically meaningful latent structure that contributes to observed variation in the data. 
Indeed, this assumption underlies the common approach for benchmarking clustering models on labelled data: the class label is hidden during training; afterwards it is revealed as a pseudo ground-truth partitioning of the data for assessing clustering ``accuracy''. 
MFCVAE aims to capture multiple latent structures and can be deployed as a multi-facet generative classifier, as distinct from standard single-facet discriminative classifiers~\cite[p.30]{murphy2012machine}. 
But we emphasise that high classification accuracy is attained as a by-product, and is \textit{not} our core goal---we do not explicitly target label accuracy, nor does high label accuracy necessarily correspond to the ``best'' multi-facet clustering.

Following earlier work, in Table \ref{tab:classification_table}, we report classification performance on MNIST and SVHN in terms of \textit{unsupervised clustering accuracy} on the test set, which intuitively measures homogeneity w.r.t. a set of ground-truth clusters in each facet (see Appendix~\ref{app:Generative, unsupervised classification} for a formal definition). 
We compare our method against commonly used single-facet (SF) and multi-facet (MF), generative (G) and non-generative (NG) deep clustering approaches (we use results as reported) of both deterministic and probabilistic nature.
We report the mean and standard deviation (if available) of accuracy over $T$ runs with different random seeds, where $T=10$ for MFCVAE.
For VaDE \cite{vade}, we report results from the original paper, and our two implementations, one with a multi-layer perceptron encoder and decoder architecture, one using convolutional layers.
Models marked with $^\eta$ explicitly state that they instead report the best result obtained from $R$ restarts with different random seeds (DEC: $R=20$, VaDE: $R=10$).
Both of these types of reporting in previous work---not providing error bars over several runs and picking the best run (while not providing error bars)---ignore stability of the model w.r.t. initialisation.  %
We further discuss this issue and the importance of stability in deep clustering approaches in Appendix~\ref{app:stability}.

\begin{table}
\caption{Unsupervised clustering accuracy ($\%$) of single-facet (SF) and multi-facet (MF), generative (G) and non-generative (NG) models on the test set.
Error bars (if available) are the sample standard deviation across multiple runs. 
Results marked with $^{\eta}$ do not provide error bars.}
\medbreak
\label{tab:classification_table}
\centering
\begin{tabular}{ccc} \toprule
Method               & MNIST            & SVHN             \\ \midrule
DEC (\cite{dec}; SF; NG)         & 84.3 $^{\eta}$ & 11.9 (0.4)   \\
VaDE (\cite{vade}; MLP; SF; G)    & 94.46 $^{\eta}$; 89.09 (3.32) & 27.03 (1.53) \\
VaDE (\cite{vade}; conv.; SF; G)  & 92.65 (1.14) & 30.80 (1.99) \\
IMSAT (\cite{imsat}; SF; NG)       & 98.4  (0.4)  & 57.3  (3.9)  \\
ACOL-GAR (\cite{acol-gar}; SF; NG)    & 98.32 (0.08) & 76.80 (1.30) \\
VLAC (\cite{willetts2019disentangling}; MF; G)         & -                & 37.8 (2.2)   \\
LTVAE (\cite{ltvae}; MF; G)        & 86.3             & -                \\
MFCVAE (ours; MF; G) & 92.02 (3.18) & 56.25 (0.93) \\ \bottomrule
\end{tabular}
\end{table}

MFCVAE is able to recover the assumed ground-truth clustering stably. 
It achieves competitive performance compared to other probabilistic deep clustering models, but is clearly outperformed by ACOL-GAR on SVHN, a single-facet, non-generative and deterministic model which does not possess three of the four properties demonstrated in Sections \ref{sec:Discovering a multi-facet structure}, \ref{sec:Diversity of generated samples}) and \ref{sec:Compositionality of latent facets}). 
Besides the results presented in the table, we also note that MFCVAE performs strongly on 3DShapes, obtaining $99.46\% \pm 1.10\%$ for floor colour and $88.47\% \pm 1.82\%$ for object shape on configuration 1, and $100.00\% \pm 0.00\%$ for wall colour and $90.05\% \pm 2.65\%$ for object shape on configuration 2.
Lastly, it is worth noting that we report classification performance for the same hyperparameter configurations and training runs of our model that are used in all experimental sections and in particular for Fig.~\ref{fig:orth_cluster_examples}, \ref{fig:compositionality} and \ref{fig:sample_generation}, i.e. our trained model has a pronounced multi-facet characteristic.
In contrast, while it is somewhat unclear, LTVAE seems to report its clustering performance when trained with only a single facet, not when performing multi-facet clustering~\cite{ltvae}.

\begin{figure}[t]
    \centering
    \includegraphics[width=.9\linewidth]{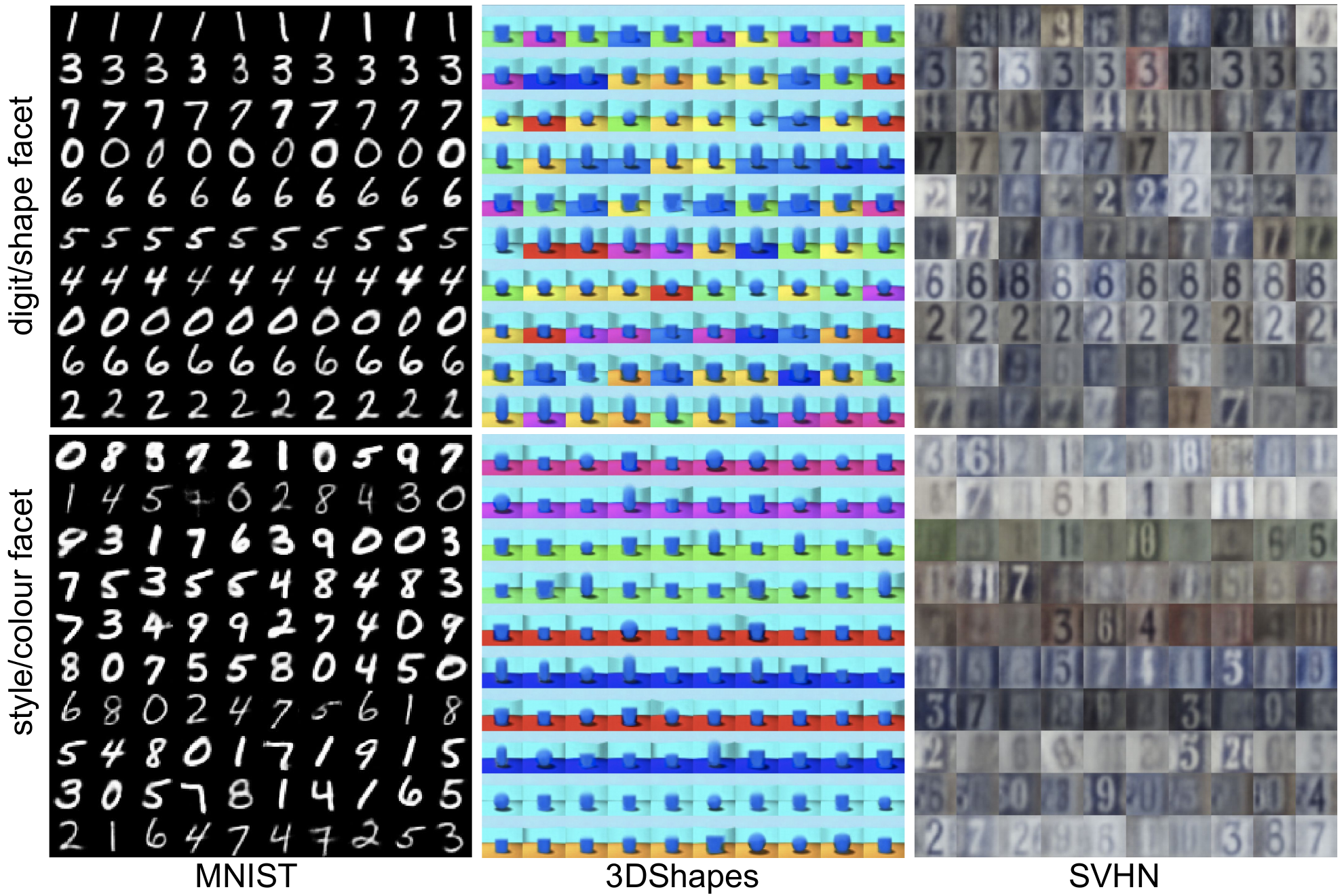}
    \caption{
    Synthetic samples generated from MFCVAE with two facets ($J=2$) trained on MNIST, 3DShapes, and SVHN.
    For each cluster $c_j$ in facet $j$, $\v{z}_j$ is sampled from $p(\v{z}_j | c_j)$ and $\v{z}_{j'}$ is sampled from $p(\v{z}_{j'})$ for the other facet $j' \neq j$.
    Each row corresponds to 10 random samples from a cluster.
    Clusters (rows) are sorted and selected (and are from the same trained model) as in Fig.~\ref{fig:orth_cluster_examples} (see Appendix~\ref{app:Diversity of generated samples} for visualisation of all clusters and comparison with LTVAE).}
    \label{fig:sample_generation}
\end{figure}

\subsection{Diversity of generated samples}
\label{sec:Diversity of generated samples}

We lastly show that MFCVAE enables diverse generation of synthetic examples for each given cluster in the different facets, as a downstream task in addition to clustering.
To obtain synthetic examples for a cluster $c_j$ in facet $j$, we sample $\v{z}_j$ from $p(\v{z}_j | c_j)$, and sample $\v{z}_{j'}$ from $p(\v{z}_{j'})$ for all other facets $j' \neq j$.
We then take the modes of $p_{\theta}(\v x | \vv z)$ where $\vv z = (\v z_1, \dots, \v z_j, \dots, \v z_J)$ as the generated images.
Fig.~\ref{fig:sample_generation} shows synthetic examples generated from the models ($J=2$) trained on MNIST, 3DShapes and SVHN.

For all three datasets, we observe synthetic samples that are homogeneous w.r.t. the characteristic value (e.g. `red background') of a cluster in the chosen facet (as we sample this continuous latent variable from the conditional distribution), but heterogeneous and diverse w.r.t. all other facets (as we sample all other continuous latent variables from their marginal distribution).
For example, on MNIST, when fixing a cluster in the digit facet, we observe generated samples that have the same digit class (e.g. all `1'), but are diverse in style (e.g. different `stroke width').
Conversely, when fixing a cluster in the style facet, we get samples homogeneous in style, but heterogeneous in digit class.
Likewise, on 3DShapes, fixing a cluster in the wall colour facet produces generations diverse in shape, but having the same wall color, and conversely when fixing the shape facet.
Besides, in all clusters, generated samples are diverse w.r.t. other factors of variation on 3DShapes, such as orientation and scale.
On SVHN, while less strong than in Fig.~\ref{fig:orth_cluster_examples}, these patterns extend here to the two facets style (background colour is particularly distinct) and digit class. 
These results are consistent with and underline the observed disentanglement of facets that we found in our previous experimental analyses.
We also compare sample generation performance between MFCVAE and LTVAE and assess the diversity of generations quantitatively in Appendix~\ref{app:Diversity of generated samples}.

\section{Related work}
\label{sec:related_work}

Within the deep generative framework, various deep clustering models have been proposed.
VaDE~\cite{vade} is the most important prior work, a probabilistic model which has been highly influential in deep clustering.
Related approaches, GM-VAEs~\cite{gmvae} and GM-DGMs~\cite{Nalisnick2016_gm, Willetts2018_sus, Willetts2020}, have similar overall performance and explicitly represent the discrete clustering latent variable during training.
Non-parametric approaches include DLDPMMs~\cite{Nalisnick2016_gm}, and HDP-VAEs~\cite{hdpvae}. 
Further, many non-generative methods for clustering have been proposed that use neural components~\cite{dec, dcn, clustergan, imsat, acol-gar, gmvae, spectralnet}.
All these approaches, however, propose single-facet models. %

Hierarchical VAEs can both be a way to learn more powerful models~\cite{Kingma2016,lvae,biva,Vahdat2020,Child2020}, but can also enable to separate out representations where each layer of latent variables represents a different aspect of the data.
Variational Ladder Autoencoders (VLAEs)~\cite{vlae} aim to do the latter: to learn independent sets of latent variables, each representing some part of the data; but each group of latent variables within this set has a $\mathcal{N}(\v 0, \v 1)$ prior, so it does not perform clustering.
Recently, progressive training for VLAEs has been proposed~\cite{ProVLAE} which sharpens the separation between layers.
Here, we also mention disentanglement methods \cite{higgins2016beta,mathieu2019disentangling,kim2018disentangling,chen2018isolating} which likewise attempt to find separated latent variables. 
However, rather than discovering facets through the prior and a hierarchical structure, these techniques attempt to find statistically-independent representations via regularisation, leading the loss to deviate from the ELBO. 
Unfortunately, these methods require lucky selection of hyperparameters to work \cite{rolinek2019variational, locatello2019challenging}, and do not provide a clustered latent space.

Learning multiple clusterings simultaneously has been studied in the case of very low dimensional datasets~\cite{Cui2007, Davidson2008, Qi2009, muller2012discovering} under the names \textit{alternative clusterings} and \textit{non-redundant clustering}.
However, when it comes to clustering high-dimensional data like images, approaches are rare.
The recently proposed LTVAE~\cite{ltvae} aims to perform this task, proposing a variational autoencoder with a latent tree model prior for a set of continuous latent variables $\vv z$, of which each $\v z_j$ has a GMM prior.
The neural components are trained via stochastic gradient ascent under the ELBO; this is interleaved with a heuristic (hill-climbing) search algorithm to grow or prune the tree structure and message-passing to learn its nodes' GMM parameters of the current structure of the tree prior in a manner reminiscent of SVAEs~\cite{svae}, rendering the entire training algorithm \textit{not} end-to-end differentiable (in contrast to MFCVAE).
LTVAE learns multiple clusterings over the data, however, lacks a proper disentanglement of facets, as discussed in Section~\ref{sec:Discovering a multi-facet structure}.

\section{Conclusion}
\label{sec:Conclusion}

We introduced Multi-Facet Clustering Variational Autoencoders (MFCVAE), a novel class of probabilistic deep learning models for unsupervised, multi-partition clustering in high-dimensional data which is end-to-end differentiable.
We provided novel theoretical results for optimising its ELBO, correcting and extending an influential related paper for the single-facet case.
We demonstrated MFCVAE's empirical performance in terms of multi-facet clustering of various levels of abstraction, and the usefulness of its prior structure for composing, classifying and generating samples, achieving state-of-the-art performance among deep probabilistic multi-facet models.

An important limitation of our work shared with many other deep clustering algorithms is the lack of a procedure to find good hyperparameters through a metric known at training time.  %
Future work should explore: MFCVAE with $J>2$; automatic tuning of hyperparameters $J$ and $K_j$; application to large-scale datasets of other modalities; and regularising the model facet-wise to further enforce disentangled representations in the latent space~\cite{peebles2020hessian}.
While we successfully stabilised model training, further work will be key to harness the full potential of deep clustering models.

\begin{ack}

FF and HZ acknowledge the receipt of studentship awards from the Health Data Research UK-The Alan Turing Institute Wellcome PhD Programme in Health Data Science (Grant Ref: 218529/Z/19/Z). 
HZ acknowledges the receipt of Wellcome Cambridge Trust Scholarship. 
MW is grateful for the support of UCL Computer Science and The Alan Turing Institute.
GN acknowledges support from the Medical Research Council Programme Leaders award MC\_UP\_A390\_1107.
CY is funded by a UKRI Turing AI Fellowship (Ref: EP/V023233/1).
CH acknowledges support from the Medical Research Council Programme Leaders award MC\_UP\_A390\_1107, The Alan Turing Institute, Health Data Research, U.K., and the U.K. Engineering and Physical Sciences Research Council through the Bayes4Health programme grant.

The authors report no competing interests.

We thank Tomas Lazauskas, Jim Madge and Oscar Giles from the Alan Turing Institute’s Research Engineering team for their help and support.
We thank Adam Huffman, Jonathan Diprose, Geoffrey Ferrari and Colin Freeman from the Biomedical Research Computing team at the University of Oxford for their help and support.
We thank Angela Wood and Ben Cairns for their support and useful discussions.

\end{ack}

\bibliographystyle{unsrt}
\bibliography{2_references.bib}

\clearpage
\appendix

\section{\texorpdfstring{$J$}{J} Independent Mixture of Gaussians prior on \texorpdfstring{$z$}{z}}
\label{app:MixtureMoG}

Let $p(\bm{z}_j)$ be the marginal distribution of $\bm{z}_j$ as follows%
\begin{align}
    p(\bm{z}_j) &= \sum_{c_j = 1}^{K_j} p(c_j, \bm{z}_j) \\
                &= \sum_{c_j = 1}^{K_j} p(c_j) p(\bm{z}_j | c_j) \\
                &= \sum_{c_j = 1}^{K_j} p(c_j) \mathcal{N}(\bm{z}_j | \bm{\mu}_{c_j}, \bm{\Sigma}_{c_j}) \label{eq:above}
\end{align} 
where $p(c_j)$ is a categorical distribution. Thus, $p(\bm{z}_j)$ is a Mixture-of-Gaussians (MoG).

Let us now derive $p(\vv{z})$, the marginal distribution of $\vv{z}$, as follows
\begin{align}
    p(\vv{z}) &= p(\bm{z}_1, \bm{z}_2, \dots, \bm{z}_J) \\
              &= \prod_{j=1}^{J} p(\bm{z}_j) \label{eq:ass2Used} \\ %
              &= \prod_{j=1}^{J} \sum_{c_j = 1}^{K_j} p(c_j) \mathcal{N}(\bm{z}_j | \bm{\mu}_{c_j}, \bm{\Sigma}_{c_j}) \label{eq:aboveUsed}    %
\end{align}
where Eq.~\eqref{eq:ass2Used} follows from the indepdendence assumption of facets, and Eq.~\eqref{eq:aboveUsed} uses Eq.~\eqref{eq:above}. 
The resulting marginal of $\vv z$ is our prior of $J$ independent MoGs.  %

\textbf{Linear (rather than exponential) complexity of number of clusters. }
Besides its representational advantages, the multi-facet prior structure features a computational advantage: Given multiple known partitions of a dataset, the total number of clusters over all facets required to represent these partitions scales \textit{linearly} w.r.t. the number of such partitions.
In comparison, the number of clusters required in a single-facet model suffers from combinatorial explosion and scales \textit{exponentially}. 

To understand this, let us consider a hypothetical multi-partition image dataset of (rather standardised) hotel rooms which features $J$ facets $C_1, C_2, \dots, C_J$ with $K_j$ possible discrete values for each characteristic, for example, the colour of the bed sheets, walls, interiors, whether a phone is present or not, the view of the room (beach, forest, city, \dots). 
We now attempt to find reasonable clusters in this dataset. In principle, a single-partition model could learn all cross-combinations $C_1 \times C_2 \times \dots \times C_J$. 
In general, this requires to learn ``at least''  $\mathcal{O}(\prod_{j=1}^{J} K_j)$ latent clusters~\footnote{This assumes that every cluster in the latent space corresponds to exactly one cross-combination of the data facets. 
Empirically, we find that for statistical reasons (``having more shots''), it can be desirable to have more latent clusters per facet than values possible for each facet.}.
Compare this with a multi-partition model such as MFCVAE. Here, we need to learn ``at least'' $\mathcal{O}(\sum_{j=1}^{J} K_j)$. If $K_j = K$ is equally large for all facets, the number of latent clusters to learn is $\mathcal{O}(K^J)$ for a single-partition model and $\mathcal{O}(K \cdot J)$ for a multi-partition model.

\section{VaDE Trick Proofs}
\label{app:proofs}
\subsection{Single-Facet VaDE Trick}
\label{app:proof1}

\begin{proof}(Theorem~\ref{thm:1}: Single-Facet VaDE Trick)
\begin{align}
    \expect_{q_\phi(\v z | \v x)}\KL\left[q_\phi(c|\v x)||p_\theta(c|\v z)\right] &= \expect_{q_\phi(\v z | \v x)} \expect_{q_\phi(c|\v x)} \log \frac{q_\phi(c|\v x)}{p_\theta(c|\v z)}\\
    &= \expect_{q_\phi(c|\v x)} \log \frac{q_\phi(c|\v x)}{\exp( \expect_{q_\phi(\v z | \v x)} \log p_\theta(c|\v z))}\\
    &= \KL\left[q_\phi(c|\v x)|| \v{\pi}( c | q_\phi(\v z | \v x))\right] - \log Z(q_\phi(\v z | \v x))
\end{align}
which is minimised w.r.t. $q_\phi(c|\v x)$ by setting the KL term to zero by $q_\phi(c|\v x) = \v{\pi}(c | q_\phi(\v z | \v x))$, where
\begin{align}
   \v{\pi}(c| q_{\phi}(\v{z} | \v{x})) &:= \frac{\exp \left( \mathbb{E}_{q_{\phi}(\v{z} | \v{x})} \log p(c | \v{z})\right)}{Z(q_{\phi}(\v{z} | \v{x}))}\ \ \text{for}\ c =1, \ldots, K \label{eq:pi} \\
    Z(q_{\phi}(\v{z} | \v{x})) & := \sum_{c=1}^{K}\exp \left( \mathbb{E}_{q_{\phi}(\v{z} | \v{x})} \log p(c | \v{z})\right)\ \label{eq:normConst}.
\end{align}
as required. Here, $Z(q_{\phi}(\v{z} | \v{x}))$ is the appropriate normalization constant for Eq.~\eqref{eq:pi} to define a probability mass function.
\end{proof}

\subsubsection{Misapprehension in original statement}
\label{app:misapp}
In the original paper, \cite{vade}, they reach Eq.~\eqref{eq:expand_elbo_single_vade}:
\[\ELBO(\v x;\theta,\phi) =\mathbb{E}_{q_\phi(\v{z} | \v{x})}\log p_\theta(\v x | \v z) - \KL\left[q_\phi(\v z | \v x)||p_\theta(\v z)\right] - \underbrace{\expect_{q_\phi(\v z | \v x)}\KL\left[q_\phi(c|\v x)||p_\theta(c|\v z)\right]}_{\circled{A}}.\]

The claim made (appendix A of~\cite{vade}) is that $q(c|\v x) = \expect_{q_\phi(\v{z}'|\v x)}p(c|\v{z}')$ makes the final term, $\circled{A}=\expect_{q_\phi(\v z | \v x)}\KL\left[q_\phi(c|\v x)||p_\theta(c|\v z)\right]$, equal to zero.
Substituting this form for $q_\phi(c|\v x)$ in Eq.~\eqref{eq:substVaDE}, we get:
\begin{align}
\circled{A}&=\expect_{q_\phi(\v z | \v x)}\KL\left[q_\phi(c|\v x)||p_\theta(c|\v z)\right]\nonumber \\
&=\int \intd \v z\, q_\phi(\v z | \v x) \sum_{c=1}^{K} q_\phi(c|\v x)\log\frac{q_\phi(c|\v x)}{ p_\theta(c | \v z)}\nonumber \\
    &=\int \intd \v z\, q_\phi(\v z | \v x) \sum_{c=1}^{K} \expect_{q_\phi(\v{z}'|\v x)}p(c|\v{z}')\log\frac{ \expect_{q_\phi(\v{z}''|\v x)}p(c|\v{z}'')}{p_\theta(c | \v z)} \label{eq:substVaDE} \\
    &=\int \intd \v z\, q_\phi(\v z | \v x) \sum_{c=1}^{K} \left(\int \intd \v z'\,{q_\phi(\v{z}'|\v x)}p(c|\v{z}')\right)\log\frac{ \int \intd \v z''\,{q_\phi(\v{z}''|\v x)}p(c|\v{z}'')}{p_\theta(c | \v z)}\nonumber \\
    &=\sum_{c=1}^{K} \underbrace{\left(\int \intd \v z'\,{q_\phi(\v{z}'|\v x)}p(c|\v{z}')\right)}_{\neq 0}\underbrace{\left[\log \int \intd \v z''\,{q_\phi(\v{z}''|\v x)}p(c|\v{z}'') - \int \intd \v z\, q_\phi(\v z | \v x) \log {p_\theta(c | \v z)}\right]}_{\stackrel{?}{=}0} \label{eq:misappr_final} \\
    &\stackrel{?}{=}0 \nonumber
\end{align}
In the above derivations, we use $\v{z}$, $\v{z}'$ and $\v{z}''$ to mark separate occurrences of the variable $\v{z}$ in different integrals.
The first term in Eq.~\eqref{eq:misappr_final} is strictly positive. To satisfy the claim, the second term in Eq.~\eqref{eq:misappr_final} would have to be equal to zero for all $c \in \{ 1, \dots, K \}$, which in general does not hold. We note that in the original codebase for~\cite{vade}, training is not done using the form of the ELBO as above, Eq.~\eqref{eq:expand_elbo_single_vade} with $\circled{A}=0$, instead, using the general form where all terms are calculated.

\textbf{Monte Carlo Sampling of the VaDE-Trick objective: }
These results and analysis raise the natural question: how is it that this misapprehension has lasted?
Perhaps this is because of the following lucky accident when performing MC sampling.

If one substitutes the optimal forms of Theorem~\ref{thm:1} back into the initial $\ELBO$ and then estimates the resulting objective using a \textit{single} MC sample from $\v z$, then the resulting estimator \textit{looks} like is an estimator of Eq~\eqref{eq:expand_elbo_single_vade} with the final term set to zero, that is:
\[\ELBO(\v x;\theta,\phi) \stackrel{?}{=}\mathbb{E}_{q_\phi(\v{z} | \v{x})}\log p_\theta(\v x | \v z) - \KL\left[q_\phi(\v z | \v x)||p_\theta(\v z)\right] .\]
Equivalently, in reverse, taking a \textit{single} MC sample for $\v z$ and using the above misapprehension as the training objective results in the same estimator as one gets from taking one MC sample for the true objective. 

Let us push through the former of these, constructing the objective and MC estimator for the correct optimal objective:
\begin{align}
    \ELBO(\v x;\theta,\phi) &=\mathbb{E}_{q_\phi(\v z|\v x)q_\phi(c|\v x)}\left[\log \frac{p_\theta(\v x | \v z)p_\theta(\v z , c)}{q_\phi(\v z|\v x)q_\phi(c|\v x)}\right]\\
    \label{eq:second_line_general}
    &= \mathbb{E}_{q_\phi(\v z | \v x)}[\log p_\theta(\v x | \v z) - \log q_\phi(\v z|\v x)] - \expect_{q_\phi(\v z | \v x)} \expect_{q_\phi(c|\v x)} \log \frac{q_\phi(c|\v x)}{p_\theta(\v z, c)}\\
    &= \mathbb{E}_{q_\phi(\v z | \v x)}[\log p_\theta(\v x | \v z) - \log q_\phi(\v z|\v x)] - \expect_{q_\phi(c|\v x)} \log \frac{q_\phi(c|\v x)}{\exp(\expect_{q_\phi(\v z | \v x)} \log p_\theta(\v z, c))}\\
    &= \mathbb{E}_{q_\phi(\v z | \v x)}[\log p_\theta(\v x | \v z) - \log q_\phi(\v z|\v x)] \nonumber \\
    &  - \KL\left[q_\phi(c|\v x)|| \pi( c | q_\phi(\v z | \v x))\right] + \log \vs Z (q_\phi(\v z | \v x)) 
\end{align}
where
\begin{align}
  \v{\pi}( c | q_\phi(\v z | \v x)) &:= \frac{\exp(\expect_{q_\phi(\v z | \v x)} \log p_\theta(\v z, c))}{\vs Z(q_\phi(\v z | \v x))}\ \ \text{for}\ c \in \mathcal{C} \\
  \label{eq:zarrow_def}
    \vs Z (q_\phi(\v z | \v x))  & := \sum_{c \in \mathcal{C}} \exp(\expect_{q_\phi(\v z | \v x)} \log p_\theta(\v z, c))\ .
\end{align}
Setting $q_\phi(c|\v x) = \v{\pi}( c | q_\phi(\v z | \v x))$ and substituting $\v z^{(l)}$ for $l=1,\ldots,L$ Monte Carlo samples from $q_\phi(\v z | \v x)$:
\begin{align}
    \ELBO(\v x;\theta,\phi) &\approx \frac{1}{L}\sum_{l=1}^L\log p_\theta(\v x | \v z^{(l)}) - \log q_\phi(\v z^{(l)}|\v x) \nonumber \\
    & \hspace{1cm}  + \log \sum_{c \in \mathcal{C}} \exp\left(\frac{1}{L}\sum_{l=1}^L \log p_\theta(\v z^{(l)}, c)\right)
    \label{eqn:vade_estimator}
\end{align}
which reduces for $L=1$ to
\begin{align}
\label{eq:MC_est_L_equals_1_general}
    \ELBO(\v x;\theta,\phi) &\approx \log p_\theta(\v x | \v z^{(1)}) - \log q_\phi(\v z^{(1)}|\v x) + \log p_\theta(\v z^{(1)}).
\end{align}
This \textit{\textbf{appears}} to be a MC estimator for
\begin{align}
    \ELBO(\v x;\theta,\phi) &=\mathbb{E}_{q_\phi(\v{z} | \v{x})}\log p_\theta(\v x | \v z) - \KL\left[q_\phi(\v z | \v x)||p_\theta(\v z)\right].    
\end{align}
This appearance is purely because the $\log \sum \exp \sum \log$ in Eq~\eqref{eqn:vade_estimator} luckily simplifies when $L=1$. Because of this lucky coincidence, all empirical results in~\cite{vade} are valid.

\subsection{Multi-Facet VaDE Trick (factorised distribution)}
\label{app:proof2}

\begin{proof}(Theorem~\ref{thm:2}: Multi-Facet VaDE Trick for factorised distribution $q_\phi(\vv{z} | \v{x}) = \prod_j q_\phi(\v{z}_j | \v{x})$)
\begin{align}
\label{eq:starting_point_factorized}
    \expect_{q_\phi(\vv z | \v x)}\KL\left[q_\phi(\v c|\v x)||p_\theta(\v c|\vv z)\right] &= \expect_{q_\phi(\vv z | \v x)} \expect_{q_\phi(\v c|\v x)} \log \frac{q_\phi(\v c|\v x)}{p_\theta(\v c|\vv z)}\\
    &= \expect_{q_\phi(\v c|\v x)} \log \frac{q_\phi(\v c|\v x)}{\exp(\expect_{q_\phi(\vv z | \v x)} \log p_\theta(\v c|\vv z))}\\
    &= \expect_{q_\phi(\v c|\v x)} \log \frac{q_\phi(\v c|\v x)}{\exp(\sum_j \expect_{q_\phi(\v z_j | \v x)} \log p_\theta(c_j|\v z_j))}\\
    &= \expect_{q_\phi(\v c|\v x)} \log \frac{q_\phi(\v c|\v x)}{\prod_j \exp(\expect_{q_\phi(\v z_j | \v x)} \log p_\theta(c_j|\v z_j))}\\
    &= \KL\left[q_\phi(\v c|\v x)|| \prod_j \v{\pi}_j( c_j | q_\phi(\v z_j | \v x))\right] - \sum_j \log Z_j(q_\phi(\v z_j | \v x))
\end{align}
which is minimised w.r.t. $q_\phi(c|\v x)$ by setting the KL term to zero by $q_\phi(c|\v x) = \prod_j \v{\pi}_j( c_j | q_\phi(\v z_j | \v x))$, where 
\begin{align}
   \v{\pi}_j( c_j | q_\phi(\v z_j | \v x)) &:= \frac{\exp(\expect_{q_\phi(\v z_j | \v x)} \log p_\theta(c_j|\v z_j))}{Z_j(q_\phi(\v z_j | \v x))}\ \ \text{for}\ c_j =1, \ldots, K_j \\
   \label{eq:Z_in_factorized_derivation}
    Z_j(q_\phi(\v z_j | \v x)) & := \sum_{c_j=1}^{K_j}\exp(\expect_{q_\phi(\v z_j | \v x)} \log p_\theta(c_j|\v z_j))\ .
\end{align}

where $Z_j(q_\phi(\v z_j | \v x))$ is a normalisation constant for $\v{\pi}_j( c_j | q_\phi(\v z_j | \v x))$, and we have used the relations
\begin{align}
    q_\phi(\vv z | \v x) &= \prod_j q_\phi(\v z_j | \v x)\\
    p_\theta(\v c|\vv z) &= \prod_j p_\theta(c_j|\v z_j) 
\end{align}
as required. 
\end{proof}

\subsection{Multi-Facet VaDE Trick (general distribution)}
\label{app:proof3}
While we use a posterior over $\vv{z}$ that factorises between facets, $q_\phi(\vv{z}|\v{x})=\prod_{j=1}^J q_\phi(\v{z}_j|\v{x})$, there is the question as to whether one can use a VaDE trick in the case where the posterior for $\vv{z}$ has a general factorisation (e.g. an autoregressive factorisation over facets).
An example would be $q_\phi(\vv{z}|\v{x})=\prod_{j=1}^J q_\phi(\v{z}_j|\v{z}_{<j}, \v{x})$, the posterior factorisation used in many hierarchical VAEs~\cite{Kingma2016,Vahdat2020,Child2020}.
We answer this question in the affirmative:
\begin{theorem}(Multi-Facet VaDE Trick for general $q_\phi(\vv z | \v x)$, $p(\vv z, \v c)$)
\label{thm:3}
For any probability distribution $q_\phi(\vv z | \v x)$, the distribution $q_\phi(\v c|\v x)$ that minimises $\expect_{q_\phi(\vv z | \v x)}\KL\left[q_\phi(\v c|\v x)||p_\theta(\v c|\vv z)\right]$ is 
\begin{align}
    \underset{q_\phi(\v c|\v x)}{\mathrm{argmin}} \expect_{q_\phi(\vv z | \v x)}\KL\left[q_\phi(\v c|\v x)||p_\theta(\v c|\vv z)\right] &= \vvs{\pi}( \v c | q_\phi(\vv z | \v x))
\end{align}
where the minimum value is attained at
\begin{align}
\underset{q_\phi(\v c|\v x)}{\mathrm{min}} \expect_{q_\phi(\vv z | \v x)}\KL\left[q_\phi(\v c|\v x)||p_\theta(\v c|\vv z)\right] &= -\log \vvs Z (q_\phi(\vv z | \v x)) 
\end{align}
where
\begin{align}
  \vvs{\pi}( \v c | q_\phi(\vv z | \v x)) &:= \frac{\exp(\expect_{q_\phi(\vv z | \v x)} \log p_\theta(\v c|\vv z))}{\vvs Z(q_\phi(\vv z | \v x))}\ \ \text{for}\ \v c \in \mathcal{C} \\
    \vvs Z (q_\phi(\vv z | \v x))  & := \sum_{\v c \in \mathcal{C}} \exp(\expect_{q_\phi(\vv z | \v x)} \log p_\theta(\v c|\vv z))\ .
\end{align}
\end{theorem}
\begin{proof}
\begin{align}
    \expect_{q_\phi(\vv z | \v x)}\KL\left[q_\phi(\v c|\v x)||p_\theta(\v c|\vv z)\right] &= \expect_{q_\phi(\vv z | \v x)} \expect_{q_\phi(\v c|\v x)} \log \frac{q_\phi(\v c|\v x)}{p_\theta(\v c|\vv z)}\\
    &= \expect_{q_\phi(\v c|\v x)} \log \frac{q_\phi(\v c|\v x)}{\exp(\expect_{q_\phi(\vv z | \v x)} \log p_\theta(\v c|\vv z))}\\
    &=  \KL\left[q_\phi(\v c|\v x)|| \vvs{\pi}( \v c | q_\phi(\vv z | \v x))\right] - \log \vvs Z (q_\phi(\vv z | \v x)) 
\end{align}
which is minimised by setting the KL term to zero as required.
\end{proof}

\section{Monte Carlo estimator of Evidence Lower Bound}
\label{app:mc_estimator}

\subsection{Primary form}

We start the derivation from the ELBO in Eq.~\eqref{eq:elbo_initial}: 
\begin{align}
    \ELBO(\v x;\theta,\phi) &=\mathbb{E}_{q_\phi(\vv z|\v x)q_\phi(\v c|\v x)}[\log \frac{p_\theta(\v x | \vv z)p_\theta(\vv z , \v c)}{q_\phi(\vv z|\v x)q_\phi(\v c|\v x)}]\\
    &=\mathbb{E}_{q_\phi(\vv z|\v x)q_\phi(\v c|\v x)}[\log \frac{p_\theta(\v x | \vv z)p_\theta(\vv z |\v c)p_\theta(\v c)}{q_\phi(\vv z|\v x)q_\phi(\v c|\v x)}] \\
    &= \mathbb{E}_{q_\phi(\vv z | \v x)}\left[\log p_\theta(\v x | \vv z) \right] + \expect_{q_\phi(\vv z | \v x)q_\phi(\v c|\v x)}\left[ \log p_\theta(\vv z \mid \v c)\right] \nonumber \\
    & \hspace{0.2cm} + \expect_{q_\phi(\v c|\v x)} \left[ \log p_\theta(\v c)\right] - \mathbb{E}_{q_\phi(\vv z | \v x)}\left[\log q_\phi(\vv z|\v x)\right] - \expect_{q_\phi(\v c|\v x)}\left[\log q_\phi(\v c|\v x)\right] \label{eqn:elbo_5terms}
\end{align}

Next, we note that Theorem 2 has the optimal value of $q(\v c | \v x)$ taking the factorised form
\begin{equation}
    q(\v c | \v x) = \prod_{j=1}^J q(c_j | \v x).
\end{equation}

Combining this with the factorised prior introduced in Eq. \eqref{eqn:facet_independence_assumption}, the loss can then be simplified and approximated as
\begin{align}
    \mathcal{L}(\v{x};\theta,\phi) &\approx  \mathbb{E}_{q_\phi(\vv z | \v x)}\left[\log p_\theta(\v x | \vv z) \right] + \sum_{j=1}^J \expect_{q_\phi(\v z_j | \v x)q_\phi(c_j|\v x)}\left[ \log p_\theta(\v z_j \mid c_j)\right] \nonumber \\
    & \hspace{-0.5cm} + \sum_{j=1}^J \expect_{q_\phi(c_j|\v x)} \left[ \log p_\theta(c_j)\right] - \sum_{j=1}^{J} \mathbb{E}_{q_\phi(\v z_j | \v x)}\left[\log q_\phi(\v z_j |\v x)\right] - \sum_{j=1}^J \expect_{q_\phi(c_j |\v x)}\left[\log q_\phi(c_j|\v x)\right]
\end{align}

Then we approximate the ELBO using MC estimation by drawing samples from $q_\phi(\vv z | \v x)$: 
\begin{align}
    \tilde{\mathcal{L}}(\v{x};\theta,\phi) &= \frac{1}{L} \sum^L_{l=1} \log p_{\bm{\theta}}(\v{x} | \vv z^{(l)}) + \frac{1}{L} \sum^L_{l=1} \sum_{j=1}^J \sum_{c_j=1}^{K_j} q_\phi(c_j |\v{x}) \log p_\theta(\v z_j^{(l)} | c_j) \nonumber \\
    & \hspace{-0.5cm} + \sum_{j=1}^J \sum_{c_j=1}^{K_j} q_\phi(c_j | \v{x}) \log p_\theta(c_j) - \frac{1}{L}\sum^L_{l=1} \sum_{j=1}^{J} \log q_\phi(\v z_j^{(l)}| \v{x}) - \sum_{j=1}^J \sum_{c_j=1}^{K_j} q_\phi(c_j |\v{x}) \log q_\phi(c_j | \v{x}) 
\end{align}

where the optimal value of $q_\phi(c_j | \v x)$ is obtained from Theorem~\ref{thm:2}: 
\begin{align}
   q_\phi(c_j | \v x) &:= \frac{\exp\left[\frac{1}{L}\sum^L_{l=1}\log p_\theta(c_j|\v z_j^{(l)})\right]}{Z_j(q_\phi(\v z_j | \v x))}\ \ \text{for}\ c_j =1, \ldots, K_j \label{eqn:opt_q_mc}\\
    Z_j(q_\phi(\v z_j | \v x)) & := \sum_{c_j=1}^{K_j}\exp\left[\frac{1}{L}\sum^L_{l=1}\log p_\theta(c_j|\v z_j^{(l)})\right]\ .
\end{align}
which reduces to
\begin{align}
    q_\phi(c_j | \v x) &= p_\theta(c_j|\v z_j^{(1)})\ \ \text{for}\ j = 1, \ldots, J
\end{align}
when $L=1$.

As a result, we obtain the loss for $L=1$:
\begin{align}
    \tilde{\mathcal{L}}(\v{x};\theta,\phi) &= \log p_\theta(\v{x} | \vv z^{(1)}) + \sum_{j=1}^J \sum_{c_j=1}^{K_j} p_\theta(c_j|\v z_j^{(1)}) \big( \log p_\theta(\v z_j^{(1)} | c_j) + \log p_\theta(c_j) \big) \nonumber \\
    & \hspace{0.2cm} - \sum_{j=1}^J \log q_\phi(\v z_j^{(1)}| \v{x}) - \sum_{j=1}^J \sum_{c_j=1}^{K_j} p_\theta(c_j|\v z_j^{(1)}) \log p_\theta(c_j|\v z_j^{(1)})
    \label{eq:5term_loss}
\end{align}

\subsection{Alternate form} 

Again, we start the derivation of an MC estimator from Eq.~\eqref{eq:elbo_initial}: 

\begin{align}
    \ELBO(\v x;\theta,\phi) &=\mathbb{E}_{q_\phi(\vv z|\v x)q_\phi(\v c|\v x)}[\log \frac{p_\theta(\v x | \vv z)p_\theta(\vv z , \v c)}{q_\phi(\vv z|\v x)q_\phi(\v c|\v x)}]\\
    &= \mathbb{E}_{q_\phi(\vv z|\v x)q_\phi(\v c|\v x)}[\log \frac{p_\theta(\v x | \vv z)p_\theta(\vv z )p_\theta(\v c|\vv z)}{q_\phi(\vv z|\v x)q_\phi(\v c|\v x)}] \label{eq:factorisation_joint} \\
    &= \mathbb{E}_{q_\phi(\vv z | \v x)}\left[\log p_\theta(\v x | \vv z) - \log q_\phi(\vv z|\v x) + \log p_\theta(\vv z)\right] \nonumber \\
    & \hspace{0.5cm} + \expect_{q_\phi(\v c|\v x)}\left[ \expect_{q_\phi(\vv z | \v x)} \log p_\theta(\v c \mid \vv z) - \log q_\phi(\v c|\v x)\right]
\end{align}
where an alternative factorisation of $p_{\theta}(\vv z, \v c)$ is used in Eq.~\eqref{eq:factorisation_joint}, resulting in a different, but equivalent formulation of the ELBO in Eq.~\eqref{eqn:elbo_5terms}.

Next, we draw MC samples from $q_\phi(\vv z | \v x)$:
\begin{align}
    \tilde{\mathcal{L}}(\v{x};\theta,\phi) &= \frac{1}{L} \sum^L_{l=1} \left[\log p_\theta(\v{x} | \vv z^{(l)}) - \log q_\phi(\vv z^{(l)} | \v{x}) + \log p_\theta(\vv z^{(l)}) \right]\\
    & \hspace{0.5cm}  + \sum_{\B{c} \in \mathcal{C}}\left\{ q_\phi(\v c | \v x)\left[  \frac{1}{L}\sum^L_{l=1}\log p_\theta(\v c \mid \vv z^{(l)}) -  \log q_\phi(\v c | \v x) \right]\right\} \label{eqn:elbo_3terms_last_term}
\end{align}
where the optimal value of $q_\phi(\v c | \v x)$ when $L=1$ is similarly obtained from Theorem~\ref{thm:2}:
\begin{align}
    q_\phi(\v c | \v x) &= \prod_{j=1}^J q_\phi(c_j | \v x)\\
    q_\phi(c_j | \v x) &= p_\theta(c_j|\v z_j^{(1)})\ \ \text{for}\ j = 1, \ldots, J
\end{align}

In this case, the term in \eqref{eqn:elbo_3terms_last_term} evaluates to zero, because from Theorem 2
\begin{equation}
    \log p_\theta(\v c \mid \vv z^{(1)}) = \sum_{j=1}^J \log p_\theta(c_j \mid \v z_j^{(1)})  = \sum_{j=1}^J \log q_\phi(c_j | \v x) = \log q_\phi(\v c | \v x).
\end{equation}

Consequently, we obtain the loss for $L=1$:
\begin{align}
    \tilde{\mathcal{L}}(\v x;\theta,\phi) &= \log p_\theta(\v x | \vv z^{(1)}) - \log q_\phi(\vv z^{(1)} | \v x) + \log p_\theta(\vv z^{(1)})
\end{align}
where
\begin{align}
    & q_\phi(\vv z^{(1)} | \v x) = \prod_{j=1}^{J} q_\phi(\v z_j^{(1)} | \v x) \\
    & p_\theta(\vv z^{(1)}) = \prod_{j=1}^J p_\theta(\v z_j^{(1)}) = \prod_{j=1}^J \sum_{c_j=1}^{K_j} p_\theta(\v z^{(1)}_j | c_j) p_\theta(c_j)
\end{align}

\subsection{Empirical comparison of primary and alternate form}

Here, we empirically compare the primary and alternate form with five and three terms, respectively. 
Each loss comes from a different factorization of $p_\theta(\vv z , \v c)$ as we show above, but are equivalent.

We verified in our implementation that both losses yield the exact same loss values on the same mini-batch, but gradients computed during optimisation are different as both losses have non-overlapping terms and consequently convergence behavior may differ during training. 
We are interested in whether these differences are substantial. 
In particular, \cite{vade} used a 5-term loss function (similar to the primary loss in Eq.~\eqref{eq:5term_loss}, even though a 3-term loss function (similar to the alternate loss) could also be obtained and is arguably more compact. 

We investigate this question with the following experimental setup, which is close to the one in~\cite{vade}: 
On MNIST, we conduct 10 training runs of our model with varying random seeds for both the primary and alternate form of the loss. 
We use the following hyperparameters with a shared architecture and refer to Appendix~\ref{app:Experimental details} for a more detailed understanding on these configurations: 

\begin{itemize}[topsep=0pt,parsep=0pt,partopsep=0pt,leftmargin=*]
    \item Number of facets: $J=1$
    \item Batch size: 512
    \item Learning rate: 0.002
    \item Dimension of $\v z$: 10
    \item Number of $c$ (number of clusters): 50
    \item Covariance structure of $p_{\theta}(\v z | c)$: diagonal
    \item Output dimensions for layers in $g(\v x; \phi)$: $\left[500, 500, 2000\right]$
    \item Output dimensions for layers in $f(\v z; \theta)$: $\left[2000, 500, 500\right]$
\end{itemize}

In Fig.~\ref{fig:loss_comp}, we show unsupervised clustering accuracy on the test set over training epochs for the 10 runs and the primary (left) and alternate (right) form of the loss. 
Our results indicate that there is no significant difference in performance between the two loss forms. 
We decide to use the primary form in all our experiments going forward as it is simpler for our implementation, e.g. when combined with progressive training, and is also more intuitively following the generative process of our model. 

\begin{figure}[h]
    \centering
    \includegraphics[width=.47\linewidth]{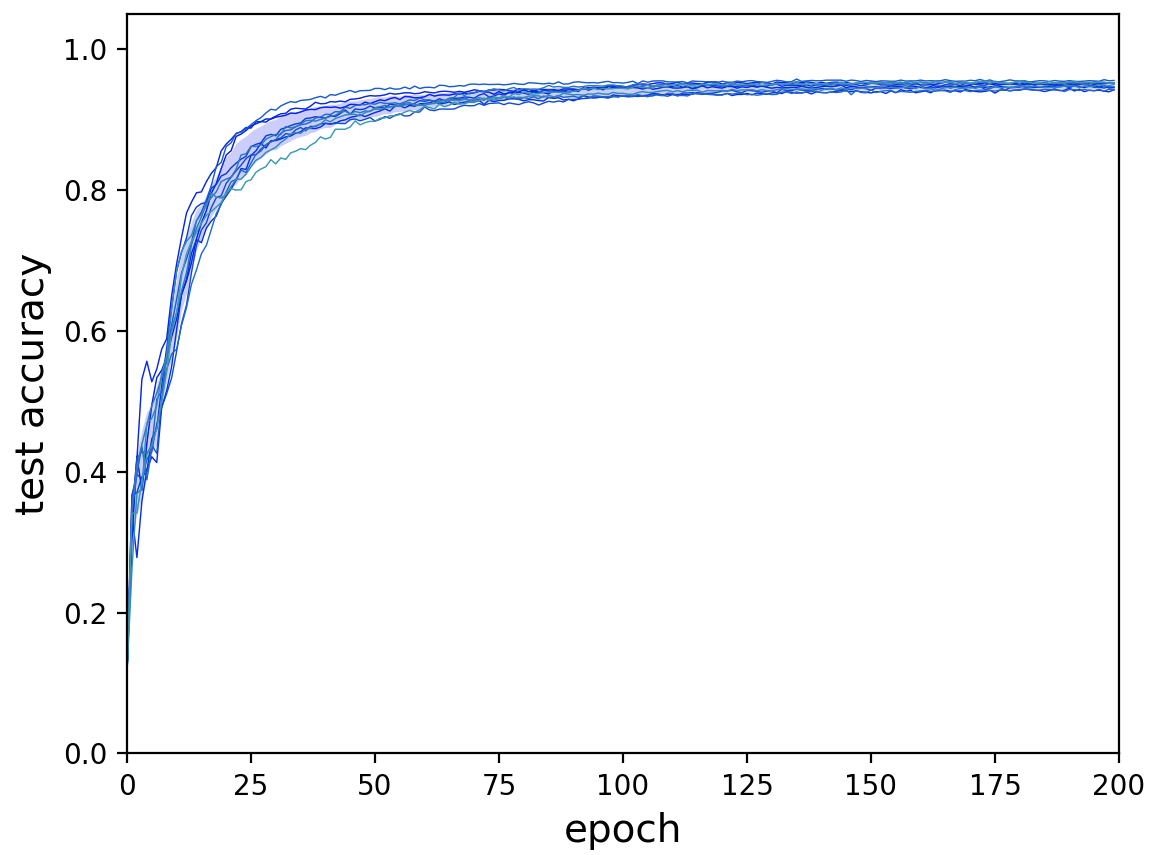}
    \includegraphics[width=.47\linewidth]{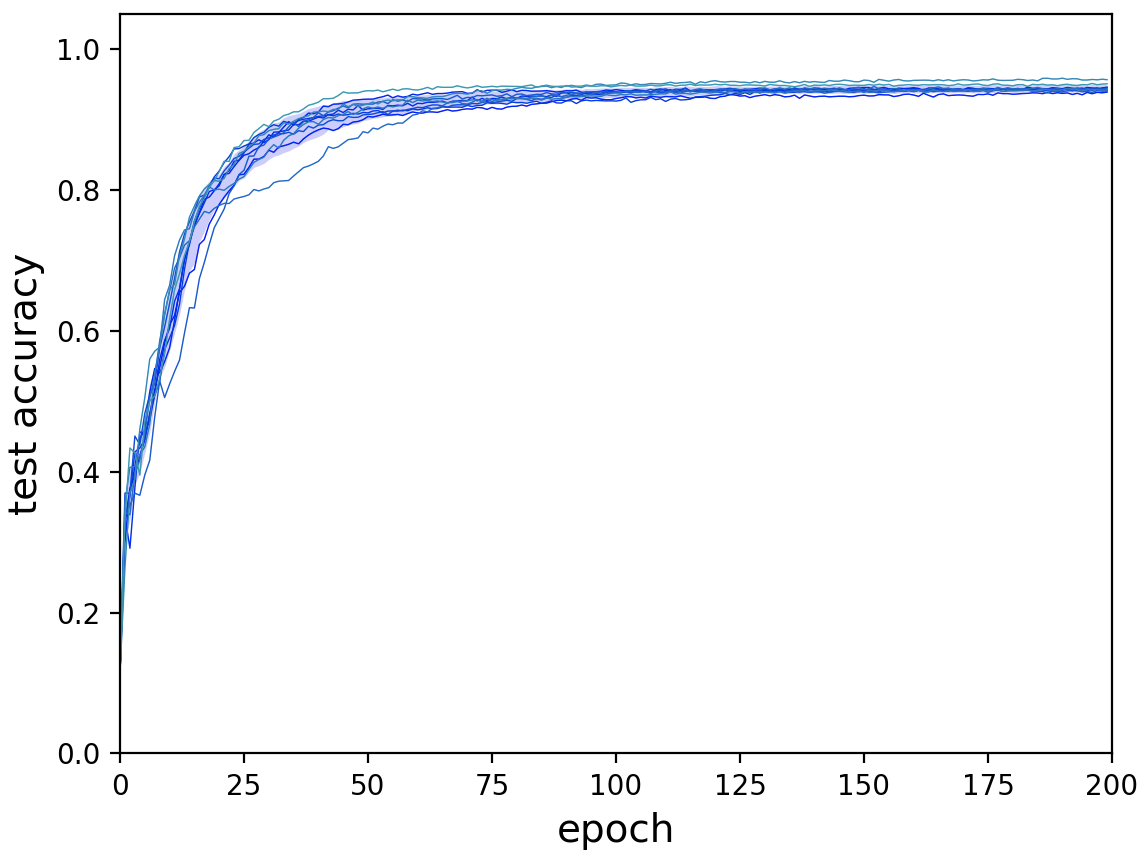}
    \caption{Unsupervised clustering accuracy on the test set for 10 runs, comparing the primary (left) and alternate (right) loss form.
    Each run is illustrated by one curve.
    The blue shade is bounded by the mean accuracy plus and minus one standard deviation across the ten runs.}
    \label{fig:loss_comp}
\end{figure}

\section{Experimental details}
\label{app:Experimental details}

We provide our code implementing MFCVAE, using \textit{PyTorch Distributions}~\cite{pytorch}, together with detailed instructions setup, training and evaluating our model, as well as reproducing the main results of this paper via shell scripts at \textcolor{blue}{\href{https://github.com/FabianFalck/mfcvae}{\url{https://github.com/FabianFalck/mfcvae}}}.

\subsection{Datasets and preprocessing}
\label{app:Datasets}

Throughout our experiments, we use three datasets: MNIST~\cite{lecun2010mnist}, 3DShapes~\cite{3dshapes18}, and SVHN~\cite{svhn}. 
In the following, we briefly introduce these datasets, particularly focusing on their abstract characteristics which might be separated out by a multi-facet clustering model, as well as ethical considerations with regards to their collection.
For MNIST and SVHN, we use their implementations as PyTorch Dataset classes as part of the torchvision package to process~\cite{pytorch}.
For 3DShapes, we provide a custom PyTorch Dataset class which contains several preprocessing steps (detailed below) and the selection of arbitrary combinations of factors.

\textbf{MNIST. } The MNIST database~\cite{lecun2010mnist} consists of grey-scale (almost binary) handwritten digits from 10 classes  ('0' to '9'). 
There are 60,000 training examples and 10,000 test examples. 
The handwritten digits are written by 500 writers (250 writers for training and test set, respectively), introducing a large variation in terms of style of these characters. 
The most prominent characteristics of MNIST are 1) the digit class, given as a supervised label 2) stroke width (e.g. 'bold', 'thin', \dots) 3) the slant of the digits (e.g. 'right-tilted', 'left-tilted', 'upright', \dots).
During preprocessing, we transform the images to a 0 to 1 scale using min-max scaling.

To the best of our knowledge, the dataset is highly curated and cropped to individual digits, so that we can exclude offensive content or important personally identifiable information in these images.
However, we note that as the images are handwritten, there is a possibility that they can be linked to these individuals.

MNIST ``was constructed from NIST's Special Database 3 [SD-3] and Special Database 1 [SD-1]''~\cite{lecun2010mnist}. 
To the best of our knowledge, SD-3 and SD-1 are no longer available for download (see \url{https://www.nist.gov/srd/shop/special-database-catalog}), as opposed to other Special Databases. 
We thus cannot comment on whether and if so in what form consent was obtained from subjects providing the handwritten digits.
 
\textbf{3DShapes. }The 3DShapes dataset~\cite{3dshapes18} consists of images of three-dimensional shapes in front of a background, generated from six independent ground truth latent factors. 
These latent factors are floor colour (10 values), wall colour (10 values), object colour (10 values), scale (8 values), shape (4 values), and orientation (15 values). 
Since all ground truth latent factors are discrete, the nature of this dataset makes it particularly suited for a multi-facet clustering task.

The dataset is preprocessed as follows: 
We transform each factor's values to a scale of integers between 0 and the number of values of that factor minus one.
Then, to be consistent with the SVHN dataset, we resize the original $64 \times 64$ images to the size $32 \times 32$ using bilinear interpolation.
Lastly, we transform the images to a 0 to 1 scale using min-max scaling.

From this 3DShapes dataset, we extract the following 2 configurations which are used during our experiments (note that other configurations can be easily created using our provided Dataset class): 
\begin{itemize}[noitemsep,topsep=0pt,parsep=0pt,partopsep=0pt,leftmargin=*]
    \item \textit{Configuration 1} (4,800 images): 10 values for floor colour, 1 value for wall colour, 1 value for object colour, 8 values for scale, 4 values for shape, 15 values for orientation
    \item \textit{Configuration 2} (4,800 images): 1 value for floor colour, 10 values for wall colour, 1 value for object colour, 8 values for scale, 4 values for shape, 15 values for orientation
\end{itemize}

3DShapes is a simulated dataset. 
The dataset was generated using the QUery Networks Mujocu environment~\cite{gqn}. 

\textbf{SVHN. }SVHN~\cite{svhn} is a real-world image dataset of cropped digits obtained of house numbers in Google Street View images. 
We focus on the 73,257 training digits and the 26,032 test digits available in the $\texttt{torchvision}$ package in PyTorch. 
SVHN is similar to MNIST in the sense that it is a labelled digit dataset, however, was collected with the aim of being significantly more complex and diverse: 
As the images were extracted from random Google Street View images in various countries, for example they have varying backgrounds, different number of digits per image (the central digit is used as the label), varying resolutions, and different digit styles, rendering them a challenging dataset for supervised and unsupervised learning tasks, and an interesting test bed for multi-facet clustering, as for some of these characteristics, it might be possible to separate them out.
During preprocessing, we transform the images to a 0 to 1 scale using min-max scaling.

\subsection{Neural architectures and Variational Ladder Autoencoder}
\label{app:Neural architectures and Variational Ladder Autoencoder}

In the following, we define two architectures which we implemented in our experiments: 
\begin{itemize}[noitemsep,topsep=0pt,parsep=0pt,partopsep=0pt,leftmargin=*]
    \item A Variational Ladder Autoencoder (VLAE) architecture, as defined in~\cite{vlae} and illustrated in Fig.~\ref{fig:ladder_architecture}.
    \item A shared encoder and decoder architecture (we refer to it as ``shared architecture'' in the following), illustrated in Fig.~\ref{fig:shared_enc_dec}.
\end{itemize}

\textbf{VLAE architecture. }
The VLAE architecture consists of an encoder (recognition model) and a decoder (generative model) which are symmetric to each other.
Both encoder and decoder have a set of backbone layers $b_j^\text{enc}$ and $b_j^\text{dec}$, respectively, which share parameters across the layers of latent variables, and thus naturally build a hierarchy of abstractions.
From and into these backbones, a set of rung layers $r_j^\text{enc}$ and $r_j^\text{dec}$ emerge which parameterise the latent variables $\vv{z}$ (encoder) and process their samples towards reconstructions (decoder).

Formally, following the notation in~\cite{vlae}, we define the recognition model as
\begin{align}
    \hat{\v{h}}_j &= \v{b}_j^\text{enc}(\hat{\v{{h}}}_{j-1}) \\ 
    \v{z}_j &\sim \mathcal{N} \left(\v{z}_j ;   r_{j,\mu}^\text{enc}(\hat{\v{h}}_j), r_{j,\sigma^2}^\text{enc}(\hat{\v{h}}_j)\right)
\end{align}
where $j = 1, \dots, J$; $b_j^\text{enc}$ and $r_j^\text{enc}$ are neural networks, $r_{j,\mu}^\text{enc}$ and $r_{j,\sigma^2}^\text{enc}$ refer to the elements of the output vector of $r_{j}$ corresponding to the mean and variance of the parameterised Gaussian distribution $q(\v{z}_j \mid \v{x})$ with diagonal covariance matrix, and $h_0 \equiv \v{x}$.

For each $j$, $q(c_j \mid \v{x})$ is not directly parameterised by neural networks, but instead computed by Theorem~\ref{thm:2} as described in Section~\ref{sec:vade_tricks}.

We define the generative model as
\begin{align}
    c_j &\sim p(c_j) \text{, for }j=1, \dots,J \\
    \v{z}_j &\sim p(\v{z}_j | c_j) \text{, for }j=1, \dots,J \\
    \tilde{\v{z}}_J &= b_J^\text{dec} \circ r_J^\text{dec}(\v{z}_J) \\
    \tilde{\v{z}}_j &= b_j^\text{dec}  \left( \left[ \tilde{\v{z}}_{j+1}, r_j^\text{dec}(\v{z}_j) \right] \right) \text{, for }j=1, \dots, J-1 \label{eq:proTraining} \\
    \v{x} &\sim u (\v{x} ;  \tilde{\v{z}}_1 )
\end{align}
where $b_j^\text{enc}$ and $r_j^\text{enc}$ are neural networks, $\left[ \cdot, \cdot \right]$ denotes concatenation of two vectors, and $u(\v{x})$ is the likelihood model of $\v{x}$.

We refer to Appendix~\ref{app:Implementation details and hyperparameters} for the exact implementation of all neural networks $b_j^\text{enc}$, $r_j^\text{enc}$, $b_j^\text{dec}$, and $r_j^\text{dec}$ and the likelihood model $u(\cdot)$ for each of the three datasets.

\begin{figure}[h!]
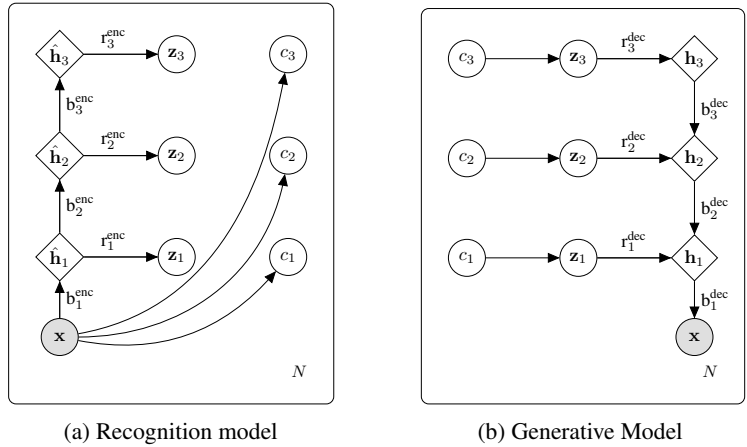

\noindent\makebox[\linewidth]{%
\subfloat[Recognition model]{\scalebox{0.7}{\raisebox{0ex}{
\tikz{ %
      \node[obs] (x) {$\v{x}$} ; %
    \node[det, above=0.7cm of x,minimum size=25pt] (h1) {$\hat{\v{h}}_1$} ; %
    \node[det, above=of h1,minimum size=25pt] (h2) {$\hat{\v{h}}_2$} ; %
    \node[det, above=of h2,minimum size=25pt] (h3) {$\hat{\v{h}}_3$} ; %
    \node[latent, right=1.4cm of h1] (z1) {$\v{z}_1$} ; %
    \node[latent, right=1.4cm of h2] (z2) {$\v{z}_2$} ; %
    \node[latent, right=1.4cm of h3] (z3) {$\v{z}_3$} ; %
    \node[latent, right=1.4cm of z1] (c1) {$c_1$} ; %
    \node[latent, right=1.4cm of z2] (c2) {$c_2$} ; %
    \node[latent, right=1.4cm of z3] (c3) {$c_3$} ; %
	\edge {x} {h1} ; %
	\edge {h1} {h2} ; %
	\edge {h2} {h3} ; %
	\edge {h3} {z3} ; %
	\edge {h1} {z1} ; %
	\edge {h2} {z2} ; %
	\draw [->] (x) to[bend right] (c1);
	\draw [->] (x) to[out=0,in=-100] (c2);
	\draw [->] (x) to[out=10,in=-100] (c3);
    \draw (x)  --  (h1) 
        node [midway,right](b_1^{enc}){$\text{b}_1^{\text{enc}}$};
    \draw (h1)  --  (h2) 
        node [midway,right](b_2^{enc}){$\text{b}_2^{\text{enc}}$};
    \draw (h2)  --  (h3) 
        node [midway,right](b_3^{enc}){$\text{b}_3^{\text{enc}}$};
    \draw (h1)  --  (z1) 
        node [pos=0.4,above](r_1^enc){$\text{r}_1^\text{enc}$};
    \draw (h2)  --  (z2) 
        node [pos=0.4,above](r_2^enc){$\text{r}_2^\text{enc}$};
    \draw (h3)  --  (z3) 
        node [pos=0.4,above](r_3^enc){$\text{r}_3^\text{enc}$};
    \plate[inner sep=0.5cm] {plate1} {(x) (c1) (c2) (c3) (z1) (z2) (z3) (h1) (h2) (h3) } {\scalebox{1}{{$N$}}}; %
  }}
  }}
\hspace*{9mm}
\subfloat[Generative Model]{\scalebox{0.7}{
\tikz{
    \node[obs] (x) {$\v{x}$} ; %
    \node[det, above=0.7cm of x,minimum size=25pt] (h1) {$\v{h}_1$} ; %
    \node[det, above=of h1,minimum size=25pt] (h2) {$\v{h}_2$} ; %
    \node[det, above=of h2,minimum size=25pt] (h3) {$\v{h}_3$} ; %
    \node[latent, left=1.4cm of h1] (z1) {$\v{z}_1$} ; %
    \node[latent, left=1.4cm of h2] (z2) {$\v{z}_2$} ; %
    \node[latent, left=1.4cm of h3] (z3) {$\v{z}_3$} ; %
    \node[latent, left=1.4cm of z1] (c1) {$c_1$} ; %
    \node[latent, left=1.4cm of z2] (c2) {$c_2$} ; %
    \node[latent, left=1.4cm of z3] (c3) {$c_3$} ; %
    \edge {c3} {z3} ; %
	\edge {c2} {z2}; %
	\edge {c1} {z1} ;
	\edge {z3} {h3} ; %
	\edge {z2} {h2}; %
	\edge {z1} {h1} ;
	\edge {h3} {h2} ; %
	\edge {h2} {h1} ; %
	\edge {h1} {x};
    \draw (h1)  --  (x) 
        node [midway,right](b_1^dec){$\text{b}_1^\text{dec}$};
    \draw (h2)  --  (h1) 
        node [midway,right](b_2^dec){$\text{b}_2^\text{dec}$};
    \draw (h3)  --  (h2) 
        node [midway,right](b_3^dec){$\text{b}_3^\text{dec}$};
    \draw (z3)  --  (h3) 
        node [pos=0.5,above,sloped](r_1^dec){$\text{r}_3^\text{dec}$};
    \draw (z2)  --  (h2) 
        node [pos=0.5,above,sloped](r_2^dec){$\text{r}_2^\text{dec}$};
    \draw (z1)  --  (h1) 
        node [pos=0.5,above,sloped](r_3^dec){$\text{r}_1^\text{dec}$};
    \plate[inner sep=0.5cm] {plate1} {(x) (c1) (c2) (c3) (z1) (z2) (z3) (h1) (h2) (h3)}
    {\scalebox{1}{{$N$}}};
    }}
    }}
 \caption{Ladder-MFCVAE architecture with $J=3$ as an example. (a) The recognition model and (b) generative model. Each \emph{labelled} arrow corresponds to a neural network. The posterior for each $c_j$ is defined using the multi-facet VaDE trick.
}
 \label{fig:ladder_architecture}
\end{figure}

\textbf{Shared architecture. }
We use the shared architecture as a simple comparison to test our hypothesis that a VLAE helps stabilise training.
In the shared architecture, each facet has an equal depth of neural networks and shares the parameters.
The encoder and decoder are both fully shared, except for the last hidden layers in both. 

More precisely, the recognition model is defined as
\begin{align}
    \hat{\v{h}} &= s^\text{enc}(\v{x}) \\
    \v{z}_j &\sim \mathcal{N} \left( \v{z}_j; t_{j,\mu}^\text{enc}(\hat{\v{h}}), t_{j,\sigma^2}^\text{enc}(\hat{\v{h}}) \right)
\end{align}
where $s^\text{enc}$ and each $t_j^\text{enc}$ are neural networks, and $t_{j,\mu}^\text{dec}$ and $t_{j,\sigma^2}^\text{dec}$ again refer to those elements of the output vector of $t_{j}$ corresponding to the mean and variance of the parameterised Gaussian distribution $q(\v{z}_j \mid \v{x})$ with diagonal covariance matrix.  $q(c_j \mid \v{x})$ is computed by Theorem~\ref{thm:2} as described in Section~\ref{sec:vade_tricks}.

The generative model is defined as
\begin{align}
    c_j &\sim p(c_j) \text{, for }j=1, \dots,J \\
    \v{z}_j &\sim p(\v{z}_j | c_j) \text{, for }j=1, \dots,J \\
    \v{h} &= t^\text{dec}(\left[ \v{z}_1, \dots, \v{z}_J \right]) \\
    \tilde{\v{x}} &= s^\text{dec}(\v{h}) \\
    \v{x} &\sim u(\tilde{\v{x}})
\end{align}
where $t^\text{dec}$ and $s^\text{dec}$ are neural networks, $\left[ \cdot, \dots, \cdot \right]$ refers to vector concatenation, and $u(\cdot)$ is the likelihood model of $\v{x}$.

Again, we refer to Appendix~\ref{app:Implementation details and hyperparameters} for the exact implementation of all neural networks $s^\text{enc}$, $t_j^\text{enc}$, $s^\text{dec}$, and $t^\text{dec}$, as well as for the likelihood model $u(\cdot)$ for each of the three datasets.

\begin{figure}[h!]
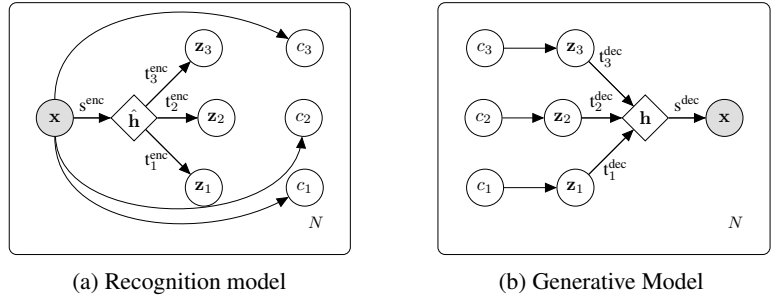

\noindent\makebox[\linewidth]{%
\subfloat[Recognition model]{\scalebox{0.7}{\raisebox{0ex}{
\tikz{ %
      \node[obs] (x) {$\v{x}$} ; %
    \node[det, right=0.7cm of x,minimum size=25pt] (h) {$\hat{\v{h}}$} ; %
    \node[latent, below right=1.2cm of h] (z1) {$\v{z}_1$} ; %
    \node[latent, right=0.76cm of h] (z2) {$\v{z}_2$} ; %
    \node[latent, above right=1.2cm of h] (z3) {$\v{z}_3$} ; %
    \node[latent, right=1.2cm of z1] (c1) {$c_1$} ; %
    \node[latent, right=0.94cm of z2] (c2) {$c_2$} ; %
    \node[latent, right=1.2cm of z3] (c3) {$c_3$} ; %
	\edge {x} {h} ; %
	\edge {h} {z1} ; %
	\edge {h} {z2} ; %
	\edge {h} {z3} ; %
	\draw [->] (x) to[out=-90, in=-150] (c1);
	\draw [->] (x) to[out=-88,in=-95] (c2);
	\draw [->] (x) to[out=90,in=150] (c3);
    \draw (x)  --  (h) 
        node [midway,above](b^{enc}){$\text{s}^{\text{enc}}$};
    \draw (h)  --  (z1) 
        node [pos=0.25,below=.1cm](t_1^enc){$\text{t}_1^\text{enc}$};
    \draw (h)  --  (z2) 
        node [pos=0.5,above](t_2^enc){$\text{t}_2^\text{enc}$};
    \draw (h)  --  (z3) 
        node [pos=0.25,above=.1cm](t_3^enc){$\text{t}_3^\text{enc}$};
    \plate[inner sep=0.5cm] {plate1} {(x) (c1) (c2) (c3) (z1) (z2) (z3) (h)} {\scalebox{1}{{$N$}}}; %
  }}
  }}
\hspace*{9mm}
\subfloat[Generative Model]{\scalebox{0.7}{
\tikz{
    \node[obs] (x) {$\v{x}$} ; %
    \node[det, left=0.7cm of x,minimum size=25pt] (h) {$\v{h}$} ; %
    \node[latent, below left=1.2cm of h] (z1) {$\v{z}_1$} ; %
    \node[latent, left=1cm of z1] (c1) {$c_1$} ; %
    \node[latent, left=0.76cm of h] (z2) {$\v{z}_2$} ; %
    \node[latent, left=0.79cm of z2] (c2) {$c_2$} ; %
    \node[latent, above left=1.2cm of h] (z3) {$\v{z}_3$} ; %
    \node[latent, left=1cm of z3] (c3) {$c_3$} ; %
	\edge {h} {x} ; %
	\edge {z1} {h} ; %
	\edge {z2} {h} ; %
	\edge {z3} {h} ; %
	\edge {c1} {z1} ; %
	\edge {c2} {z2} ; %
	\edge {c3} {z3} ; %
    \draw (x)  --  (h) 
        node [midway,above](s^{dec}){$\text{s}^{\text{dec}}$};
    \draw (h)  --  (z1) 
        node [pos=0.4,below=.1cm](t_1^dec){$\text{t}_1^\text{dec}$};
    \draw (h)  --  (z2) 
        node [pos=0.5,above](t_2^dec){$\text{t}_2^\text{dec}$};
    \draw (h)  --  (z3) 
        node [pos=0.5,above=.2cm](t_3^dec){$\text{t}_3^\text{dec}$};
    \plate[inner sep=0.5cm] {plate1} {(x) (z1) (z2) (z3) (c1) (c2) (c3) (h)} {\scalebox{1}{{$N$}}}; 
    }}
    }}
 \caption{Shared encoder and decoder MFCVAE architecture with $J=3$ as an example. (a) The recognition model and (b) generative model. Each labelled arrow corresponds to a neural network. The posterior for each $c_j$ is defined using the multi-facet VaDE trick.
}
 \label{fig:shared_enc_dec}
\end{figure}

\subsection{Progressive training algorithm}
\label{app:Progressive training algorithm}

We use a progressive training algorithm to train our VLAE architectures. 
We strongly base its implementation on~\cite{ProVLAE} and refer to this source for a more complete introduction, but will point out differences to this formulation below.

The idea of progressive training is to start with training a single facet (typically the one of highest depth in the VLAE architecture), and then progressively loop in the other facets one after the other in a smooth manner.
To formalise this, we define a progressive step $s = 1, 2, \dots, J-1$, where in step $s$, facets $J-s+1$ to $J$ (both including) are contributing to the network (and might be currently looped in), and $\alpha_j$, the fade-in coefficient of layer $j$.
$\alpha_j$ linearly increases from 0.0 to 1.0 during the first 15,000 (for MNIST and SVHN) or 2,000 (for 3DShapes) batches of a progressive step (except for $s=1$), is 0.0 if the facet has not yet been looped in, and is 1.0, otherwise.
\cite{ProVLAE} used 5,000, but we increased this number for MNIST and SVHN to have a smoother loop-in of facets.

In contrast to the formulation in~\cite{ProVLAE} which excludes from the model latent facets that are not looped in yet in a certain progressive step, in our formulation, all latent facets are part of the model throughout all progressive steps, yet do not contribute to the KL-divergences or the reconstruction term in Eq.~\eqref{eqn:mfc_elbo}. 
We achieve this by applying the fade-in coefficient only to the decoder rungs, not the encoder rungs (compare Eq.~(9) in~\cite{ProVLAE}, where both the encoder and decoder rungs are faded in); and to weigh the KL-divergences in $\v{z}_j$ and $c_j$, as before.
In other words, this is similar to the implementation of progressive training in~\cite{ProVLAE} with $\alpha_j=1.0$ for the encoder rungs throughout all progressive steps, and the regular, smoothly increasing $\alpha_j$ value for the decoder rungs.
Precisely, to implement the progressive training algorithm, we amend Eq.~\eqref{eq:proTraining} as follows: 
\begin{align}
    \tilde{\v{z}}_j &= b_j^\text{dec}  \left( \left[ \tilde{\v{z}}_{j+1}, \alpha_j r_j^\text{dec}(\v{z}_j) \right] \right) \ \ \text{, for} \ \ j=1, \dots, J,  \\
    \ELBO^{\mathrm{MFCVAE}}(\mathcal{D};\theta,\phi) &= \expect_{\v x \sim \mathcal{D}}\Big[ \expect_{q_\phi({\vv{z}|\v{x})}} \log p_\theta(\v x | \vv z ) \nonumber \\
    &- \sum_{j=1}^J \alpha_j \left[ \expect_{q_\phi(c_j|\v x)}\KL(q_\phi(\v{z}_j|x) ||p_\theta(\v{z}_j|c_j)) + \KL(q_\phi(c_j|\v x)||p(c_j))\right]\Big]
\end{align}
where 
\begin{align}
    \alpha_j &= 1.0, \ \ \text{for} \ \ j = (J-s+1), \dots, J \\
    \alpha_{J-s} &\in [0, 1] \ \ \text{(looped in)} \ \ \\
    \alpha_j &= 0.0, \ \ \text{for all} \ \ j = 1, \dots, (J-s-1)
\end{align}
and all other equations of the VLAE remain unchanged.
Thus, when $\alpha_j = 0.0$, the gradient w.r.t. any parameters in $b_j^\text{enc}, r_j^\text{enc}, b_j^\text{dec}$, $r_j^\text{dec}$, as well as the parameters of the priors $p(c_j)$ and $p(\v{z}_j)$ are 0, and we achieve the same effect as if those components would not be part of the model.

Lastly, while we have tested ``pretraining'' the latent facets which are not looped in yet through a KL-regularisation terms in $\v{z}_j$ and $c_j$ (see Eq.~(10) in~\cite{ProVLAE}), we could not see a beneficial effect on stability of model training in our model. 
As this would add complexity to the training algorithm, we do not pursue this type of pretraining here.

\subsection{Implementation details and hyperparameters}
\label{app:Implementation details and hyperparameters}

\textbf{New assets. }
We publish the following new assets accompanying this paper:
\begin{itemize}[noitemsep,topsep=0pt,parsep=0pt,partopsep=0pt,leftmargin=*]
    \item \textit{Code}: We provide our source code with detailed instructions on setup, reproducing our results via shell scripts, training, evaluation in a README file at \textcolor{blue}{\href{https://github.com/FabianFalck/mfcvae}{\url{https://github.com/FabianFalck/mfcvae}}}. 
    We follow the code templates and NeurIPS guidelines for code submissions.
    Our code is provided under MIT license.
    The initial implementation of our model are inspired by the codebase of VaDE in \url{https://github.com/eelxpeng/UnsupervisedDeepLearning-Pytorch/blob/master/udlp/clustering/vade.py}, as well as the official VLAE implementation (\cite{vlae} and \url{https://github.com/ermongroup/Variational-Ladder-Autoencoder}), and the official ProVLAE implementation (\cite{ProVLAE} and \url{https://github.com/Zhiyuan1991/proVLAE/blob/master/model_ladder_progress.py}).  %
     Our convolutional VLAE architecture is largely based on the neural architecture for the CelebA dataset in the ProVLAE codebase mentioned above.
    \item \textit{Pretrained models}: We further provide pretrained models with the hyperparameters reported in this section as part of the folder \texttt{pretrained\_models/}. 
\end{itemize}

\textbf{Existing assets used. }
Our work uses the following Python software packages with accompanying licenses (if known): 
PyTorch~\cite{pytorch} (in particular the PyTorch Distributions and Torchvision packages; custom license), 
Numpy~\cite{harris2020array} (BSD 3-Clause License), 
Weights\&Biases~\cite{wandb} (MIT License), 
Matplotlib~\cite{matplotlib} (PSF License), 
Seaborn~\cite{seaborn} (BSD 3-Clause License), 
Pickle~\cite{pickle} (N/A), 
H5Py~\cite{h5py} (BSD 3-Clause License), 
OpenCV 2~\cite{opencv} (Apache License), 
Scikit-learn~\cite{scikit-learn} (BSD 3-Clause License), 
boilr (\url{https://github.com/addtt/boiler-pytorch}) (MIT License).
Regarding data assets used, we refer to Section~\ref{app:Datasets}.

\textbf{Data splits. }
For all three datasets, we split data into training and test dataset (no validation dataset used). 
For MNIST and SVHN, we use the standard data splits as provided with these datasets and in the TorchVision PyTorch package.
For 3DShapes, in both configurations, we use 80\% for training and the remaining 20\% for testing. 
Here, we sample the images uniformly at random and without replacement.
We also refer to Appendix~\ref{app:Datasets} for a more detailed discussion on preprocessing of these datasets.

\textbf{Likelihood models. }For MNIST data, we define its likelihood $p(\v x|\vv z)$ as a product of independent Bernoulli likelihoods, where each dimension is Bernoulli-distributed with respect to some learnt parameter and independent of other dimensions.
Bernoulli likelihood is a reasonable assumption, because most pixels in MNIST images have values close or equal to 0 and 1.

For 3DShapes and SVHN data, we define their likelihood $p(\v x|\vv z)$ to be a product of independent Gaussian likelihoods, where each dimension is Gaussian-distributed with its mean learnt as a parameter and its variance fixed as a hyperparameter.

\textbf{Other design choices. } Apart from the neural architectures, hyperparameters and likelihood models of MFCVAE, there were other design choices which were made on a per-dataset basis: 
\begin{itemize}[noitemsep,topsep=0pt,parsep=0pt,partopsep=0pt,leftmargin=*]
    \item \textit{Covariance structure of the Gaussian $p(\v z_j | c_j)$ for each $j$ and $c_j$}: The covariance matrices can be set to either diagonal or full. In this paper, diagonal covariance is found to be sufficient for MNIST. 
    Full covariance is chosen for 3DShapes and SVHN as it results in a stronger disentanglement of facets.
    \item \textit{Whether to fix $\bm{\pi}_j$ or train them as parameters}: In order to encourage clusters to have similar sizes in each facet, one option is to fix $\bm{\pi}_j$ to be $1/K_j$ componentwise for each facet $j$. 
    We fix $\bm{\pi}_j$ in models trained on 3DShapes and SVHN.
    \item \textit{Activation functions: } To avoid vanishing gradients and encourage a more stable training, we tested three activation functions, in particular, ReLU, leaky ReLU and ELU. 
    For MNIST, we found ReLU to be sufficient. 
    For 3DShapes and SVHN, where convolutional neural networks are involved, we sometimes observed vanishing gradients in training runs, which is why we used the ELU activation function where we no longer observed this problem (leaky ReLU likewise worked, but we chose ELU for consistency with previous VLAE implementations mentioned above).
\end{itemize}

\textbf{Hyperparameter tuning. } 
We performed several large exploratory hyperparameter sweeps over wide grids of possible hyperparameter values, looking at the qualitative improvement in facet disentanglement and, where available, the training accuracy of supervised labels (often only one label of the two facets of interest available).
In these exploratory hyperparameter sweeps, we observed the following hyperparameter patterns that generalise across all datasets: 

We noticed that results are stable w.r.t. to a large set of hyperparameters and ranges of possible values. 
In particular, this applies to batch size, learning rate, the number of batches used during fade-in, and to some degree the number of clusters in both facets.
However, we noticed that some hyperparameters must be set rather carefully to achieve strong disentanglement between facets. 
In particular, we find that the style/colour facet's latent dimension must be rather precisely set to a narrow range of values yielding strong disentanglement of facets: 
For MNIST, $\dim(\v z_1)$ has to be around 5.
For SVHN, $\dim(\v z_2)$ has to be around 5.
For 3DShapes, $\dim(\v z_2)$ has to be around 2.

\textbf{Hyperparameters of the reported results. } 
In the following, we report the hyperparameters for training our models reported and presented in Section~\ref{sec:Experiments} (note that the hyperparameters for the models trained on the two different 3DShapes configurations are the same). 
In Table~\ref{tab:hyperparameters_final_models}, we report chosen values of scalar hyperparameters.
For full details on each of these hyperparameters, we refer to our code and in particular the help message of the respective command line arguments in the training script.

\begin{table}[h!]
\caption{Scalar hyperparameters and design choices for our three model configurations on MNIST, 3DShapes and SVHN, with results presented in Section~\ref{sec:Experiments} and Appendix~\ref{app:Additional experimental results}.}
\medbreak
\centering
\begin{tabular}{cccc} \toprule
 & MNIST & 3DShapes & SVHN \\ \midrule
Batch size & 512 & 150 & 150 \\
Learning rate  &  0.0005 & 0.0003 & 0.0005 \\
Latent dimension of the first facet, $\dim(\v z_1)$  & 5 & 20 & 22 \\
Number of clusters in the first facet, $\dim(c_1)$ & 25 & 60 & 200 \\
Latent dimension of the second facet, $\dim(\v z_2)$ & 5 & 2 & 7 \\
Number of clusters in the second facet, $\dim(c_2)$  & 25 & 20 & 50 \\
Number of training batches for each fade-in  & 15000 & 2000  & 15000 \\
Likelihood model for $p(\v x | \vv z)$ & Bernoulli & Gaussian & Gaussian \\
Standard deviation of $p(\v x | \vv z)$ componentwise (if Gaussian) & N/A & 0.6 & 0.3 \\
Data dependent initialisation for $g(\v x; \phi)$ and $f(\vv z; \theta)$ & No & Yes & Yes \\
Covariance structure of $p(\v z_j | c_j)$  & diagonal  & full  & full \\
Fix $\bm{\pi}_j$  & No  &  Yes &  Yes \\
Diagonal entries during initialisation for covariance of $p(\v z_j | c_j)$ & 0.01 & 0.01 & 0.01 \\
Activation function  & ReLU  & ELU & ELU \\  \bottomrule
\end{tabular}
\label{tab:hyperparameters_final_models}
\end{table}

\textbf{Neural architectures.} Following up Appendix~\ref{app:Neural architectures and Variational Ladder Autoencoder}, we here provide the detailed initialisation of hidden layers of our neural architectures.

We first discuss the fully-connected ladder architecture which we use to train MFCVAE on MNIST, with results presented in Section~\ref{sec:Experiments} and Appendix~\ref{app:Additional experimental results}.
We initialise the VLAE architecture as detailed in Table~\ref{tab:fc_vlae_architecture}.

\begin{table}[h!]
\caption{Details of the fully-connected ladder architecture for our model trained on MNIST and reported in Section~\ref{sec:Experiments} and Appendix~\ref{app:Additional experimental results}.}
\medbreak
\centering
\begin{tabular}{cc} \toprule
Recognition Network & Generative Network \\ \midrule
$b_1^{enc}$: $\dim(\v x) \times 500$ linear layer & $r_2^{enc}$: $\dim(\v z_2) \times 2000$ linear layer \\ 
ReLU activation & ReLU activation \\
$r_1^{enc}$: $500 \times (2\cdot \dim(\v z_1))$ linear layer & $b_2^{enc}$: $2000 \times 500$ linear layer \\
& ReLU activation\\
$b_2^{enc}$: $500 \times 2000$ linear layer & $r_1^{enc}$: $\dim(\v z_1) \times 500$ linear layer \\
ReLU activation & ReLU activation\\
$r_2^{enc}$: $2000 \times (2\cdot \dim(\v z_2))$ linear layer & $b_1^{enc}$: $500 \times \dim(\v x)$ linear layer\\
 & Sigmoid activation\\ \bottomrule
\end{tabular}
\label{tab:fc_vlae_architecture}
\end{table}

Next, we describe the convolutional ladder architecture which we use to train MFCVAE on 3DShapes and SVHN, with results presented in Section~\ref{sec:Experiments} and Appendix~\ref{app:Additional experimental results}.
We initialise the VLAE architecture as detailed in Table~\ref{tab:conv_vlae_architecture}.

\begin{table}[h!]
\caption{Details of the convolutional ladder architecture trained on 3DShapes and SVHN.
Conv2d is the 2D convolutional operation, and ConvTranspose2d is the 2D transposed convolutional operation.
We implement both operations using the \texttt{torch.nn} package in PyTorch.
For both operations, the four numbers represent output channels, input channels, kernel size (height) and kernel size (width) respectively ($C_{out}, C_{in}, H, W$).
Convolutional operations marked with (*) have stride 1 to ensure valid dimensions. All remaining convolutional operations have stride 2.
For experimental results of models from this architecture, see Section~\ref{sec:Experiments} and Appendix~\ref{app:Additional experimental results}.}
\medbreak
\centering
\begin{tabular}{cc} \toprule
Recognition Network & Generative Network \\ \midrule
$b_1^{enc}$: $64 \times \dim(\v x) \times 4 \times 4$ Conv2d &  $r_2^{enc}$: $\dim(\v z_2) \times 16384$ linear layer  \\
ELU activation, batch norm & ELU activation, batch norm \\
$r_1^{enc}$: $64 \times 64 \times 4 \times 4$ Conv2d  & $b_2^{enc}$: $128 \times 256 \times 4 \times 4$ ConvTranspose2d  \\
ELU activation, batch norm & ELU activation, batch norm \\
$64 \times 64 \times 4 \times 4$ Conv2d (*)   & $64 \times 128 \times 4 \times 4$ ConvTranspose2d (*)  \\
ELU activation, batch norm  & ELU activation, batch norm \\
$1024 \times (2\cdot \dim(\v z_1))$ linear layer             & $r_1^{enc}$: $\dim(\v z_1) \times 16384$ linear layer  \\
$b_2^{enc}$: $128 \times 64 \times 4 \times 4$ Conv2d                             & ELU activation, batch norm \\
ELU activation, batch norm      & $b_1^{enc}$: $\dim(\v x) \times 128 \times 4 \times 4$ ConvTranspose2d  \\
$r_2^{enc}$: $128 \times 128 \times 4 \times 4$ Conv2d                              & ELU activation, batch norm \\ 
ELU activation, batch norm            & Sigmoid activation (only for SVHN) \\
$256 \times 128 \times 4 \times 4$ Conv2d                      &  \\ 
ELU activation, batch norm         &   \\
$3136 \times (2\cdot \dim(\v z_2))$ linear layer    & \\  \bottomrule
\end{tabular}
\label{tab:conv_vlae_architecture}
\end{table}

Lastly, in Table~\ref{tab:shared_architecture}, we provide details on the shared architecture for MNIST training, with its results presented in Appendix~\ref{app:stability}.

\begin{table}[h!]
\caption{Details of the shared architecture trained on MNIST. For its results, see Appendix~\ref{app:stability}.}
\medbreak
\centering
\begin{tabular}{cc} \toprule
Recognition Network & Generative Network \\ \midrule
$s^{enc}$: $\dim(\v x) \times 500$ linear layer & $t_1^{dec}$: $ \dim(\v z_1) \times 2000$ linear layer  \\
ReLU activation & ReLU activation \\
$500 \times 2000$ linear layer & $t_2^{dec}$: $ \dim(\v z_2) \times 2000$ linear layer \\
ReLU activation & ReLU activation \\
$t_1^{enc}$: $2000\times (2\cdot \dim(\v z_1))$ linear layer & $s^{dec}$: $2000 \times 500$ linear layer \\
 & ReLU activation \\
$t_2^{enc}$: $2000\times (2\cdot \dim(\v z_2))$ linear layer & $500 \times \dim(\v x)$ linear layer \\  
& Sigmoid activation \\ \bottomrule
\end{tabular}
\label{tab:shared_architecture}
\end{table}

\textbf{Initialisation. }
In MFCVAE, all parameters in the deep neural networks $g(\v x; \phi)$ and $f(\vv z; \theta)$ are initialised using either Glorot normal initialisation \cite{glorot_init} for MNIST, and using a data-dependent initialisation method~\cite{data_init} for 3DShapes and SVHN. 
The idea of the data-dependent initialisation is to set the parameters in the deep neural network such that all layers in the network are encouraged to train at roughly the same rate, with the aim of avoiding vanishing or exploding gradients.
Data-dependent initialisation is particularly useful for convolutional neural networks. 
Therefore, we use it as a starting point for model training of 3DShapes and SVHN datasets, where convolutional neural networks are used.

For the parameters of the MoGs, we initialise them facet-wise as follows:
\begin{itemize}[noitemsep,topsep=0pt,parsep=0pt,partopsep=0pt,leftmargin=*]
    \item Mixing weights $\bm{\pi}_j$ are initialised to be $1/K_j$ component-wise.  %
    \item For each $k_j \in \{1,...,K_j\}$, means $\bm{\mu}_{j, k_j}$ of the Gaussians are initialised with the means on an MoG (implemented with the package \texttt{sklearn.mixture.GaussianMixture}) fitted on a dataset consisting of latent observations $\v z_j$ sampled from $q(\v z_j | \v x)$, where $\v x$ are all batches from the corresponding training dataset of MFCVAE.
    Note that encoder parameterising $q(\v z_j | \v x)$ is not trained at this point.
    The aim is to encourage a smoother and faster learning of the multiple MoG prior by starting from an MoG fitted to the initial state instead of one initialised at random.
    \item Covariance matrices $\Sigma_{j, k_j}$ are not initialised by the outputs from the fitted MoG above.
    Instead, a fixed value is assigned to all diagonal entries of the covariance matrices.
    The fixed value is the same across all clusters in all facets, which is a hyperparameter set to be much larger than the output variances from the trained fitted MoGs.
    This initialisation is favoured because at the start of the training, the MoGs do not contain much useful information as they were fitted on latent observations obtained from a randomly initialised model.
    An overly small variance at the start of the training could result in the model being stuck in a local optimum prematurely.
\end{itemize}
We note that we found these prior initialisations to be reasonable choices, but have not extensively explored alternatives.

\textbf{Potential negative societal impacts. } 
Our work is mainly of theoretical and methodological nature, thus, we do not have a direct application of our model on which it could cause immediate negative societal impacts. 
Since we provide a general clustering algorithm, MFCVAE can be used in malicious or potentially unethical ways for any clustering task at hand, suited particularly for high-dimensional data.
Our model does not account for fairness of clusters, which should be taken particular care of when dealing with data from human subjects.
As our model is a generative model by nature, we mention the possibility to abuse our model for the generation of deepfakes for disinformation.
Further, we have not investigated the vulnerability of our model to aversarial attacks, which might cause a significant security problem when applied in front-end applications and tasks.

\textbf{Compute resources. }
We had access to two GPU clusters: 
One internal cluster with 12 Nvidia GeForce GTX 1080 graphic cards each with 8GB of memory that was shared with many other users (access for 5 months ongoing), and one Microsoft Azure cluster with initially two, later four Nvidia Tesla M60 each with 8GB of memory that was used only by the authors (access for approximately 4 months). 

To train one model on each of the 4 dataset (configurations) on the Azure cluster detailed above, it takes approximately 31 min for MNIST, 36 min for 3DShapes (in both configurations), and 5h 54 min for SVHN.
Since we performed a seed sweep over 10 runs on each of the 4 dataset (configurations), the total computational time to reproduce the main results in this paper is (31 min + 2 $\cdot$ 36 min + 354 min) $\cdot$ 10 = 4570 min $\approx$ 76 hours of GPU time.

\subsection{Differences between VaDE and (\texorpdfstring{$J=1$}{J=1}) MFCVAE} 
\label{app:diff_to_vade}

In the following, we describe the (pre-)training algorithm of VaDE~\cite{vade} and compare it with the training of MFCVAE with one facet ($J=1$). 
Throughout, VaDE uses a symmetric, shared encoder and decoder architecture (see Appendix~\ref{app:Neural architectures and Variational Ladder Autoencoder}).
VaDE has the following two differences compared to MFCVAE ($J=1$): 
\begin{itemize}[noitemsep,topsep=0pt,parsep=0pt,partopsep=0pt,leftmargin=*]
\item \textit{Stacked Denoising Autoencoder (SAE)}~\cite{vincent2010stacked} pretraining: 
VaDE uses a two-stage SAE deterministic pretraining algorithm which is in detail described in~\cite{dec} to find good initialisations for the parameters $\theta$ of the decoder and $\phi$ of the encoder.
During the first stage, denoising autoencoders, which are two-layer neural networks of symmetric shapes, are trained using a least-squares reconstruction loss (i.e. deterministically as a plain autoencoder). 
In every iteration of this first stage of the pretraining routine, the outermost layers (at the front of the encoder and the back of the decoder), which have been trained in previous iterations, are frozen, and the next denoising autoencoder towards the centre of the architecture is trained.
Then, in the second stage of training, the entire architecture, which has been trained in this sequential fashion, is fine-tuned, again using a deterministic reconstruction loss. 
For a detailed description of this pretraining algorithm, we refer to~\cite{dec}.
\newline
Once the pretraining routine is complete, VaDE uses the weights $\theta$ and $\phi$ obtained as the initialisation of regular VaDE training, maximizing the ELBO with Monte Carlo sampling.---
We note that MFCVAE does not require any SAE pretraining.
Instead, we either initialise these weights randomly (MNIST) or using data-dependent initialization (3DShapes and SVHN; details see Appendix~\ref{app:Implementation details and hyperparameters}).
\item VaDE restricts the covariance matrices $\Sigma_c$ of the conditional Gaussian distributions $p(\v{z} \mid c) = \mathcal{N} (\mu_c, \Sigma_c)$ to be \textit{diagonal}, i.e. each $p(\v{z} \mid c)$ is a product of $\dim(\v{z})$ independent univariate Gaussian distributions.
In contrast, MFCVAE allows $\Sigma_c$ to be \textit{full} and enables MFCVAE to express more complex (facet-wise) dependencies in the prior.
\end{itemize}
The SAE pretraining routine adds significant complexity to the training algorithm which MFCVAE does not require in order to obtain comparable performance.
Once the encoder and decoder are initialised, both VaDE and MFCVAE use a Gaussian-Mixture model to initialise the prior $p(\v{z})$ and its parameters $\pi, \mu_c$ and $\Sigma_c$, fitted with an EM-algorithm.
We note that in the single-facet case, training MFCVAE simplifies to only one progressive step, i.e. the main training stage of VaDE is equivalent to that of MFCVAE (but with different initialisation).

\clearpage

\section{Additional experimental results}
\label{app:Additional experimental results}

\subsection{On the stability of training}
\label{app:stability}

In this appendix, we analyse the stability of MFCVAE with respect to different neural architectures and discuss the stability of deep clustering models in this context.

A natural starting point for a neural architecture of MFCVAE is a shared encoder and decoder architecture (as detailed in Appendix~\ref{app:Neural architectures and Variational Ladder Autoencoder}), which was previously used in VaDE \cite{vade} and other deep clustering models.
When using this architecture, we observe a high variation between runs which only vary in their (partly) random initialisation (determined by the random seed; see Fig.~\ref{fig:stability} [Top left]). 
However, when being lucky, drawing the right lottery ticket, this architecture can yield excellent disentanglement of facets, just like our progressively trained VLAE architecture can do (but in a stable manner).
We visualise input examples assigned to clusters from such a lucky run (which is not part of Fig.~\ref{fig:stability} [Top left]) in Fig.~\ref{fig:serene_cluster_examples}.
We point out that this run is cherry-picked from over 100 runs with different random initialisation. 
We could not produce stable results with a shared encoder and decoder architecture, neither for $J=1$ nor $J>1$.
Given these stability issues of deep clustering models, it is not only crucial to address them (which we do next), but this even more highlights the importance of providing error bars, and that picking the best run of many (as has been common practice among several deep clustering papers) is particularly here not acceptable.

To overcome these stability issues, we used a combination of a progressive training algorithm and a VLAE architecture.
We found that only using a VLAE architecture (Fig.~\ref{fig:stability} [Top right]) significantly improves the performance of disentangling facets (here only measured in terms of accuracy w.r.t. the supervised label), but is not sufficient to fully stabilise the runs over different random seeds.
Only by additionally using a progressive training schedule (Fig.~\ref{fig:stability} [Bottom]), we achieve very good disentanglement of facets and at the same time stable performance.
While we here report this for one configuration of hyperparameters only and on MNIST, we made this observation throughout all datasets and in diverse hyperparameter settings.

\begin{figure}[p]
    \centering
    \includegraphics[width=.47\linewidth]{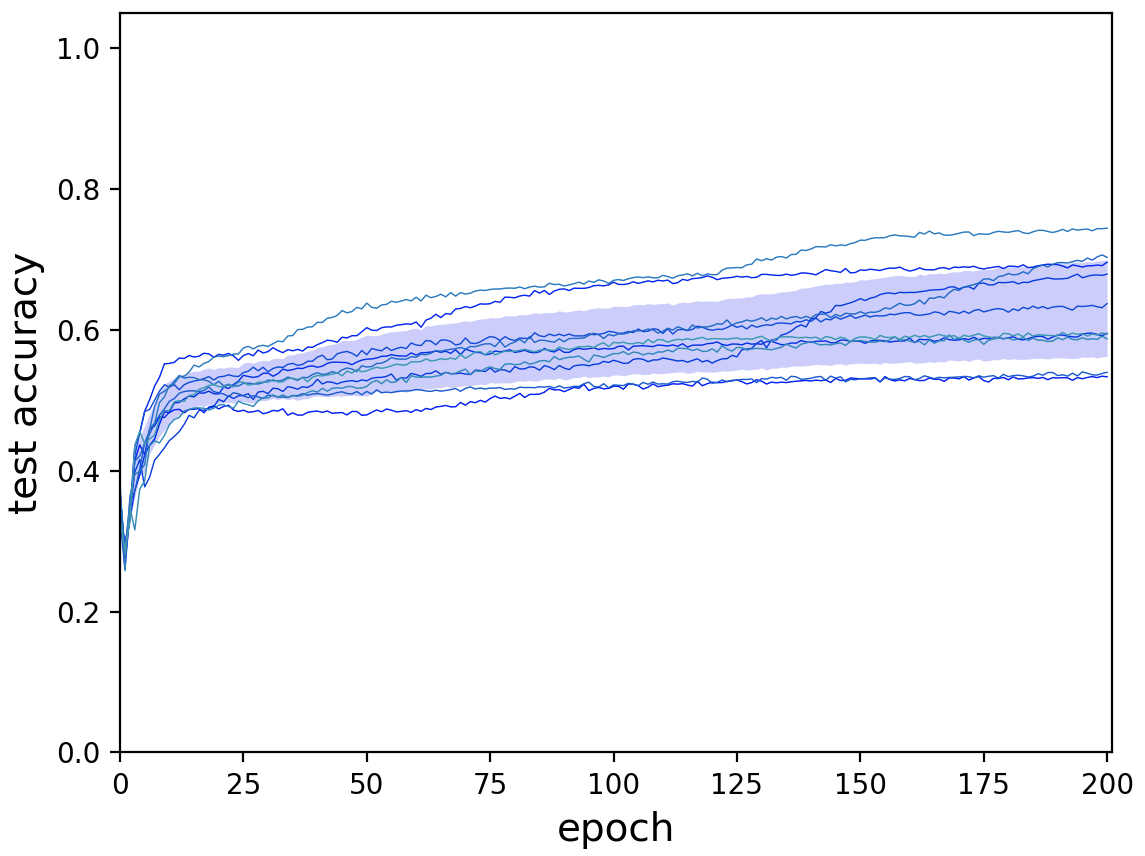}
    \includegraphics[width=.47\linewidth]{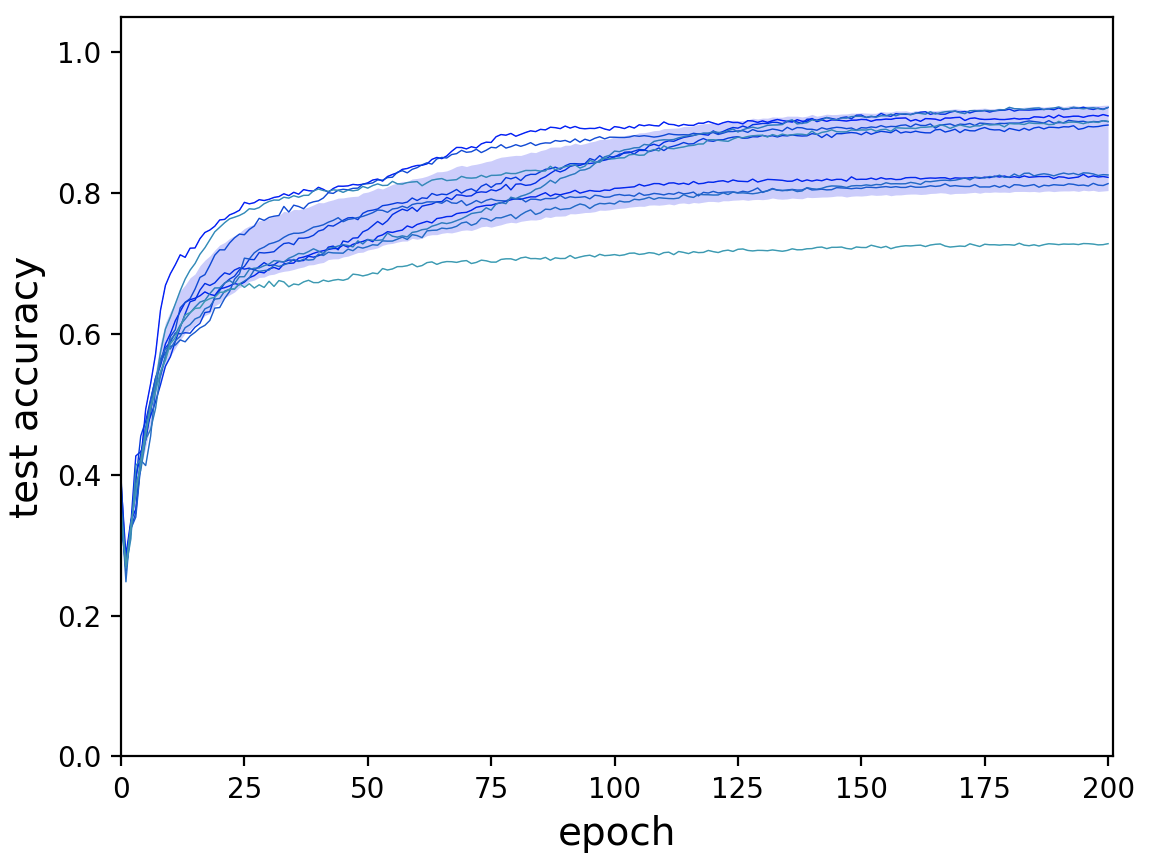}
    \includegraphics[width=.47\linewidth]{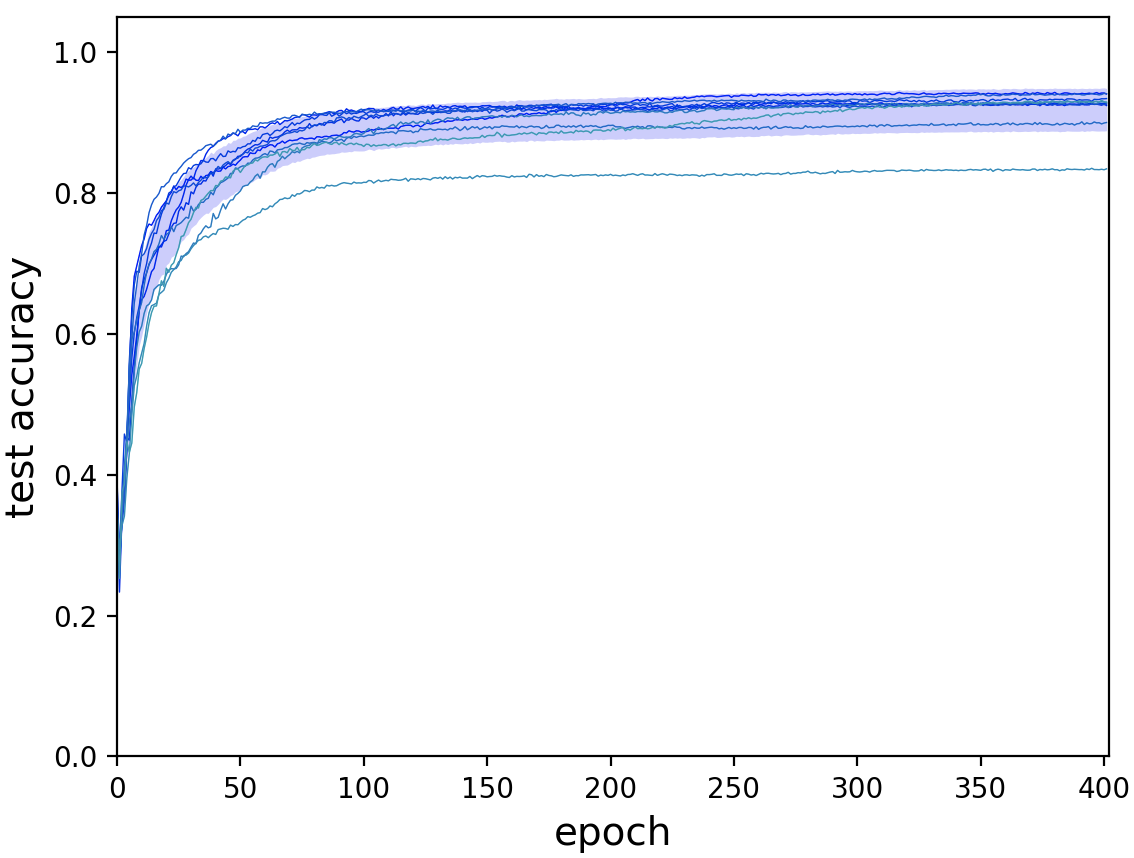}
    \caption{Test accuracy over training epochs for models trained on MNIST for three different architectures.
    Ten runs are performed for each architecture.
    Each run is illustrated by one curve.
    The blue shade is bounded by the mean accuracy plus and minus one standard deviation across the ten runs.
    [Top left] Shared architecture [Top right] VLAE architecture \textit{without} progressive training schedule [Bottom] VLAE architecture \textit{with} progressive training schedule}
    \label{fig:stability}
\end{figure}

\begin{figure}[p]
    \centering
    \includegraphics[width=.8\linewidth]{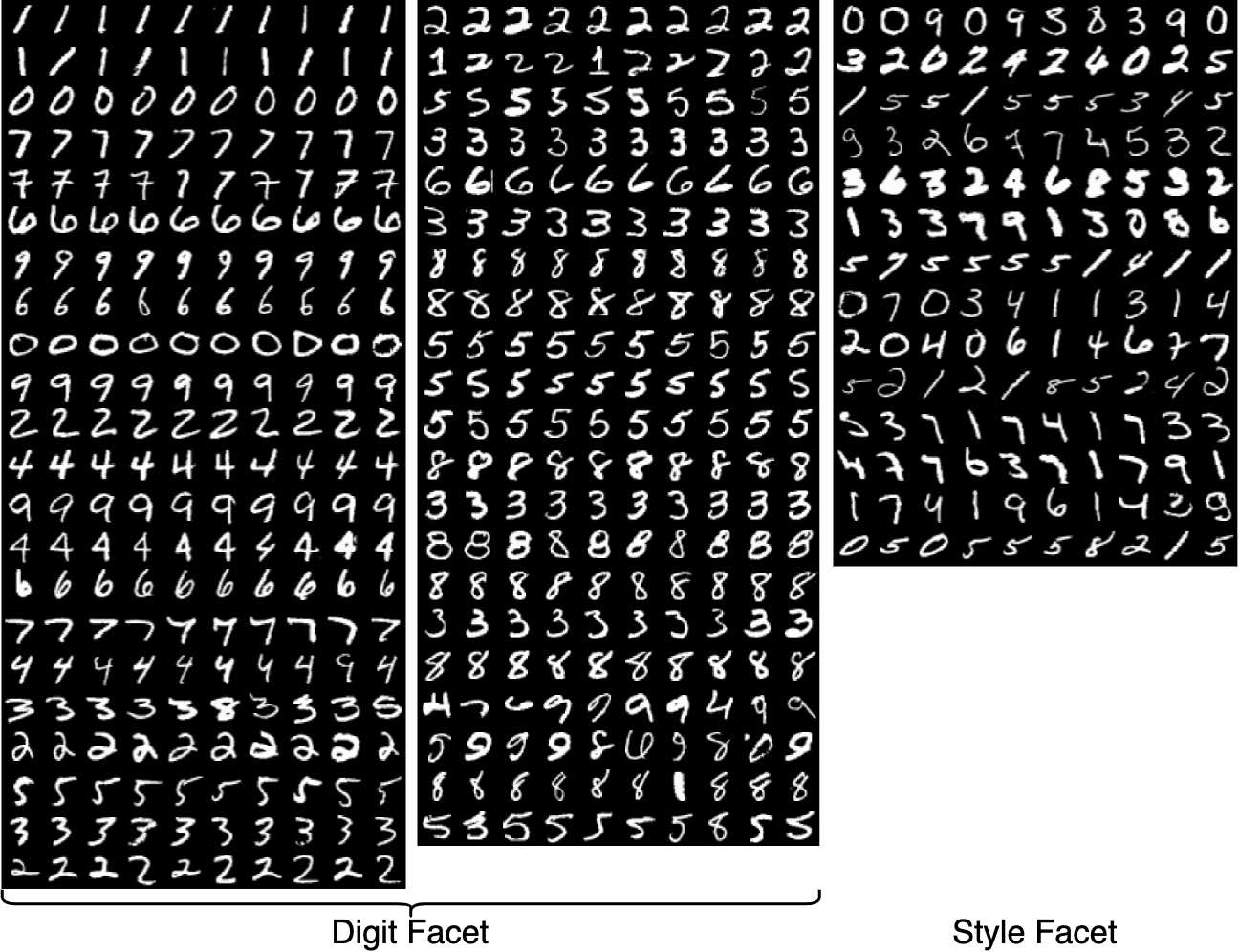}
    \caption{Input examples of a cherry-picked MFCVAE model with a shared architecture, with a lucky lottery ticket drawn as the random initialisation, and two-facets ($J=2$), trained on MNIST.
    Sorting is performed in the same way as in Fig.~\ref{fig:orth_cluster_examples}.}
    \label{fig:serene_cluster_examples}
\end{figure}

\clearpage

\subsection{Generalisation between training and test set}
\label{app:generalisation}

An important question in an exploratory setting is to what degree clustering results generalise from a training to a test set. 
In supervised machine learning, it is common to see a generalisation gap: 
Performance of a model is generally better on the training set than on the test set, often because the model overfits on the training set, and the goal is to minimise this gap, while actually being interested in test set performance. 
Perhaps surprisingly, we observe that MFCVAE has a negligible generalisation gap, i.e. typically performs almost equally well on the training and test set. 

To analyse this, we use the exact experimental setup of our main results as detailed in Appendix~\ref{app:Implementation details and hyperparameters} with $J=2$ facets, training on MNIST. 
Fig.~\ref{fig:generalisation} shows the unsupervised clustering accuracy over training epochs, evaluated both on the training set (left) and the test set (right).
When evaluating unsupervised clustering accuracy after the model is fully trained, a mean training accuracy of $91.85\% \pm 3.09 \%$ is achieved across ten runs, which is slightly lower than the mean test accuracy of $92.02\% \pm 3.02\%$ presented in Section \ref{sec:Generative classification}.
Thus, while we observe small differences between performance on training and test set, also when considering individual runs, these are not significant. 
In summary, MFCVAE generalises well between training and test set.

\begin{figure}[h!]
    \centering
    \includegraphics[width=.47\linewidth]{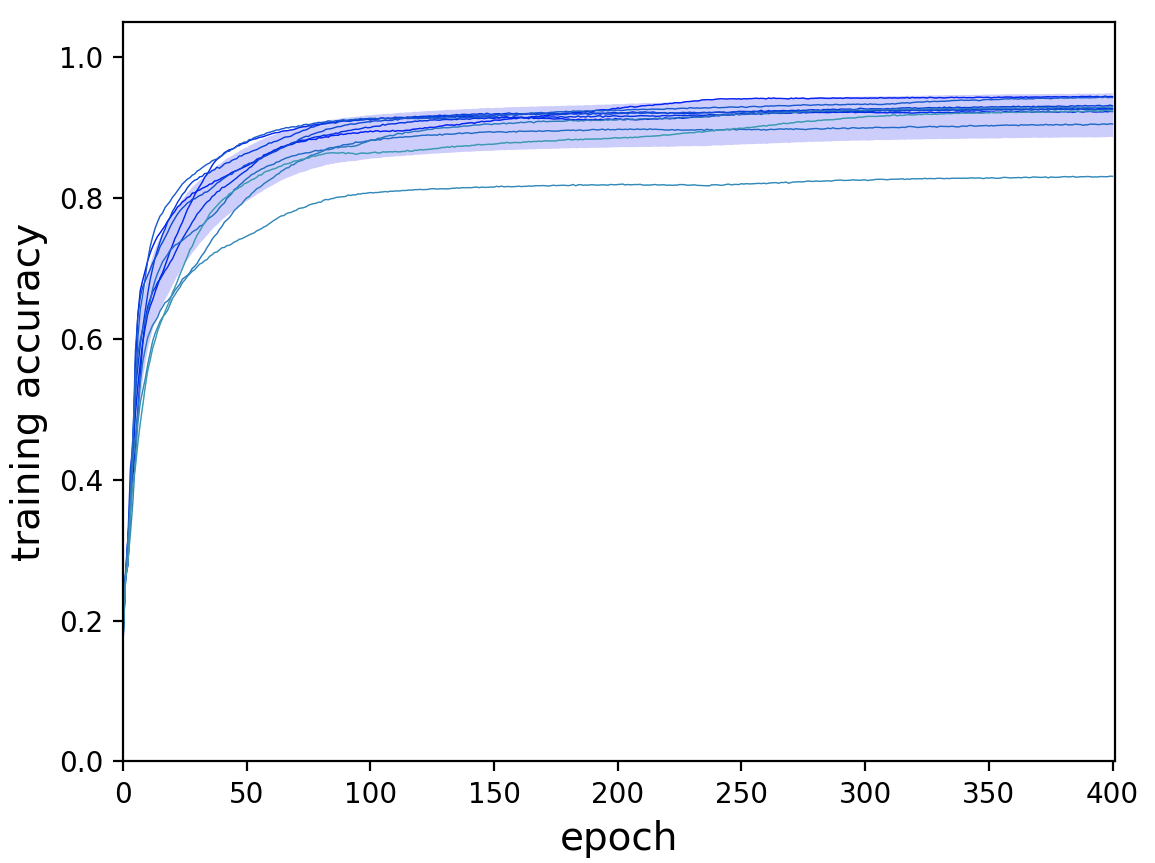}
    \includegraphics[width=.46\linewidth]{Figures/Appendix/test_acc_MNIST_final_sweep.png}
    \caption{Unsupervised clustering accuracy over training epochs for MFCVAE trained on (the training set of) MNIST as detailed in Appendix ~\ref{app:Implementation details and hyperparameters}, and evaluated on the training set [Left] and the test set [Right], respectively.
    Ten runs are performed, with each run being illustrated by one curve on the left and right, respectively.
    The blue shade is bounded by the mean accuracy plus and minus one standard deviation across the ten runs.}
    \label{fig:generalisation}
\end{figure}

\subsection{Discovering a multi-facet structure}
\label{app:Discovering a multi-facet structure}

This appendix provides the complete results of Section~\ref{sec:Discovering a multi-facet structure}.
As before, we visualise input examples from the test set for clusters of MFCVAE with two-facets ($J=2$) trained on MNIST, 3DShapes and SVHN.
Here, we show all clusters of our results in Fig.~\ref{fig:orth_cluster_examples}, visualised in Figs.~\ref{fig:orth_clusters_full_mnist} to \ref{fig:orth_clusters_full_svhn}.
In all figures, inputs (columns) are sorted in decreasing order by their assignment probability $\mathrm{max}_{c_j}  \v{\pi}_j( c_j | q_\phi(\v z_j | \v x))$.

In particular, in Fig.~\ref{fig:orth_clusters_full_mnist}, we directly compare our model to the results shown in LTVAE~\cite{ltvae}, the model closest to ours in its attempt to learn a clustered latent space with multiple disentangled facets.
As can be seen, LTVAE struggles to separate data characteristics into separate facets (see Fig.~\ref{fig:orth_clusters_full_mnist} (b)). 
In particular, LTVAE learns digit class in both facets, i.e. this characteristic is not properly disentangled between facets. 
In comparison, MFCVAE better isolates the two characteristics, and does not learn digit class in the style facet.

\begin{figure}[h!]
    \centering
    \includegraphics[width=.8\linewidth]{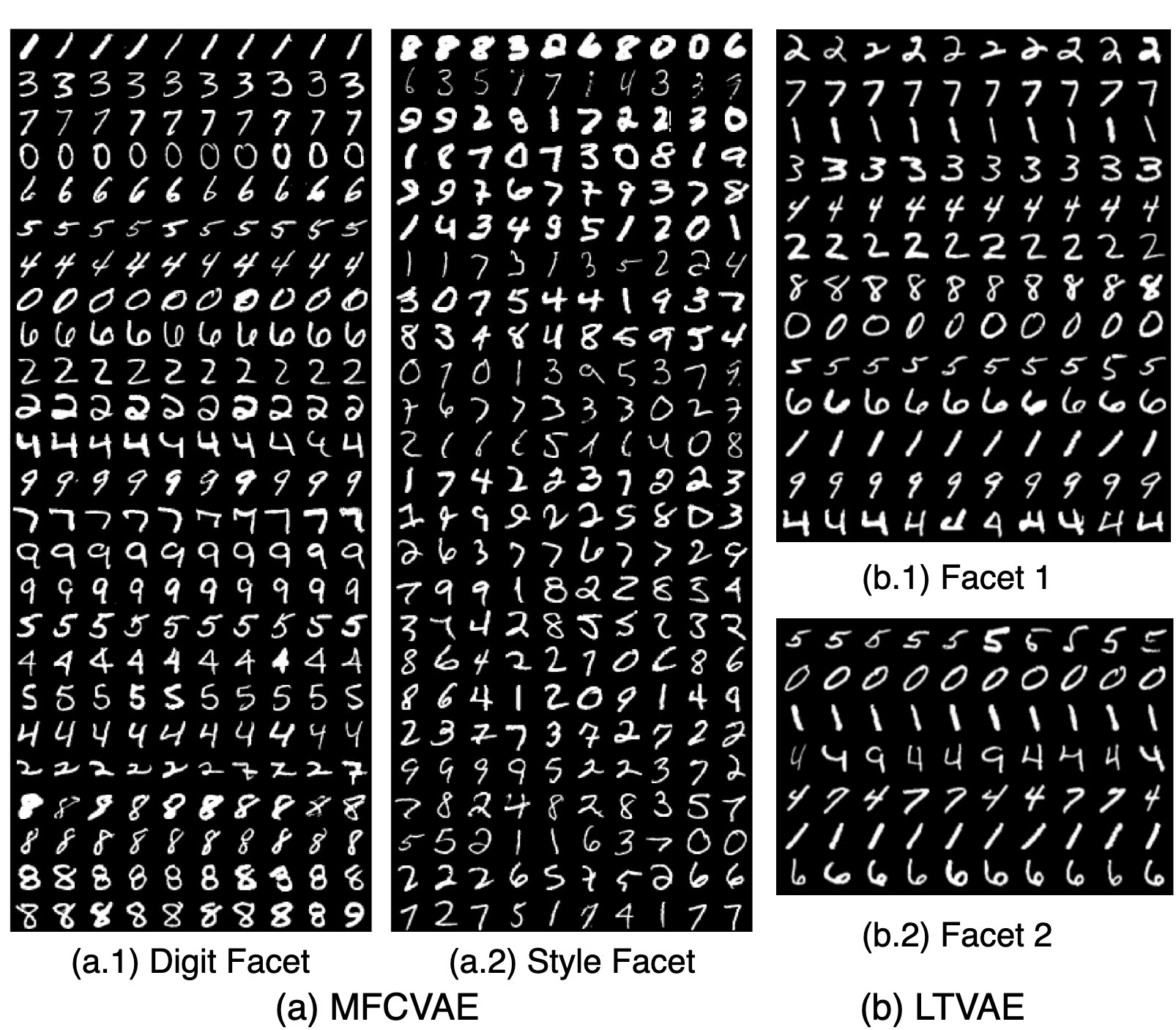}
    \caption{(a) Input examples for clusters of MFCVAE with two-facets ($J=2$) trained on MNIST. Rows and columns are sorted as in Fig.~\ref{fig:orth_cluster_examples}. (b) Input examples for clusters of LTVAE with two-facets, likewise trained on MNIST. 
    Plot is taken as reported in \cite{ltvae}, Fig.~5.}
    \label{fig:orth_clusters_full_mnist}
\end{figure}

\begin{figure}[h!]
    \centering
    \includegraphics[width=\linewidth]{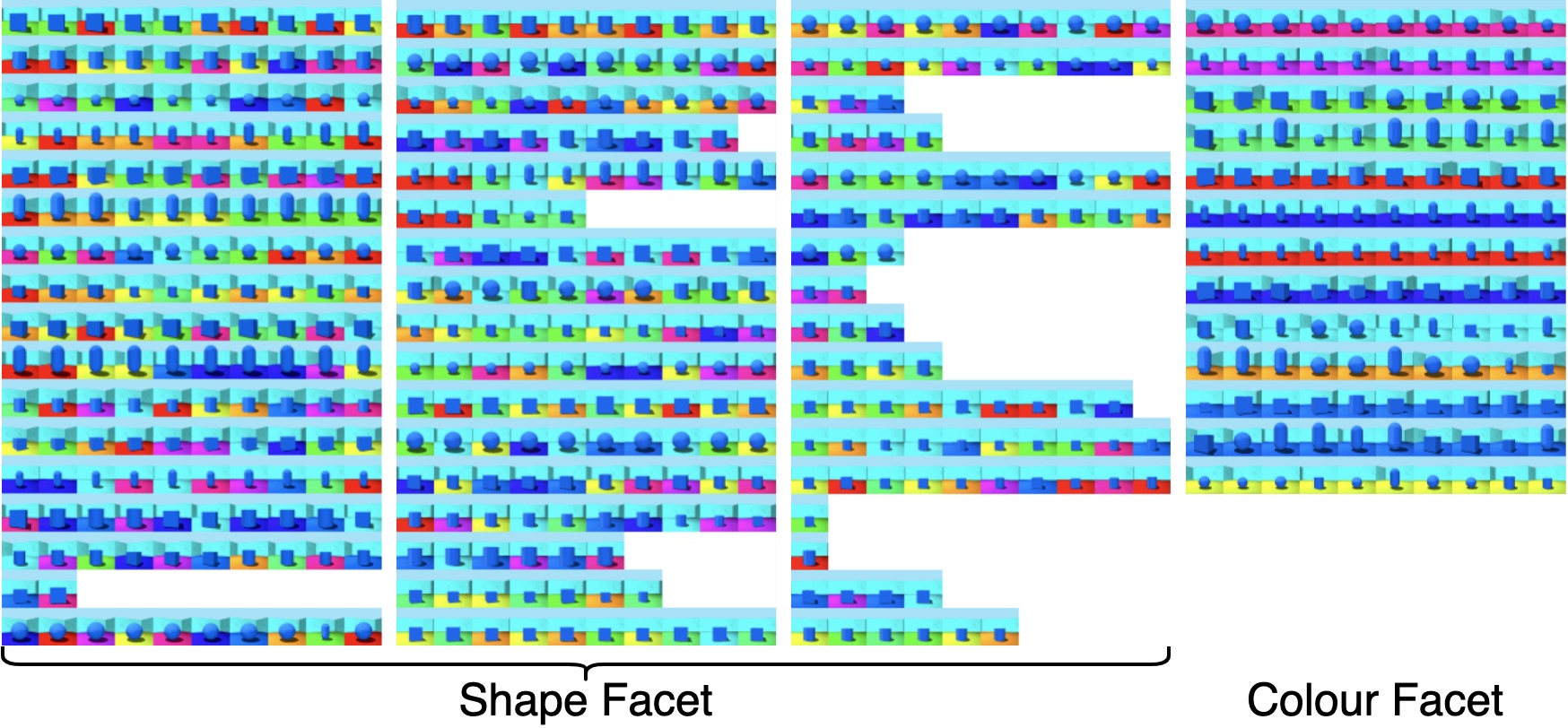}
    \caption{Input examples for clusters of MFCVAE with two-facets ($J=2$) trained on 3DShapes (configuration 1). Rows and columns are sorted as in Fig.~\ref{fig:orth_cluster_examples}.}
    \label{fig:orth_clusters_full_3dshapes_1}
\end{figure}

\begin{figure}[h!]
    \centering
    \includegraphics[width=\linewidth]{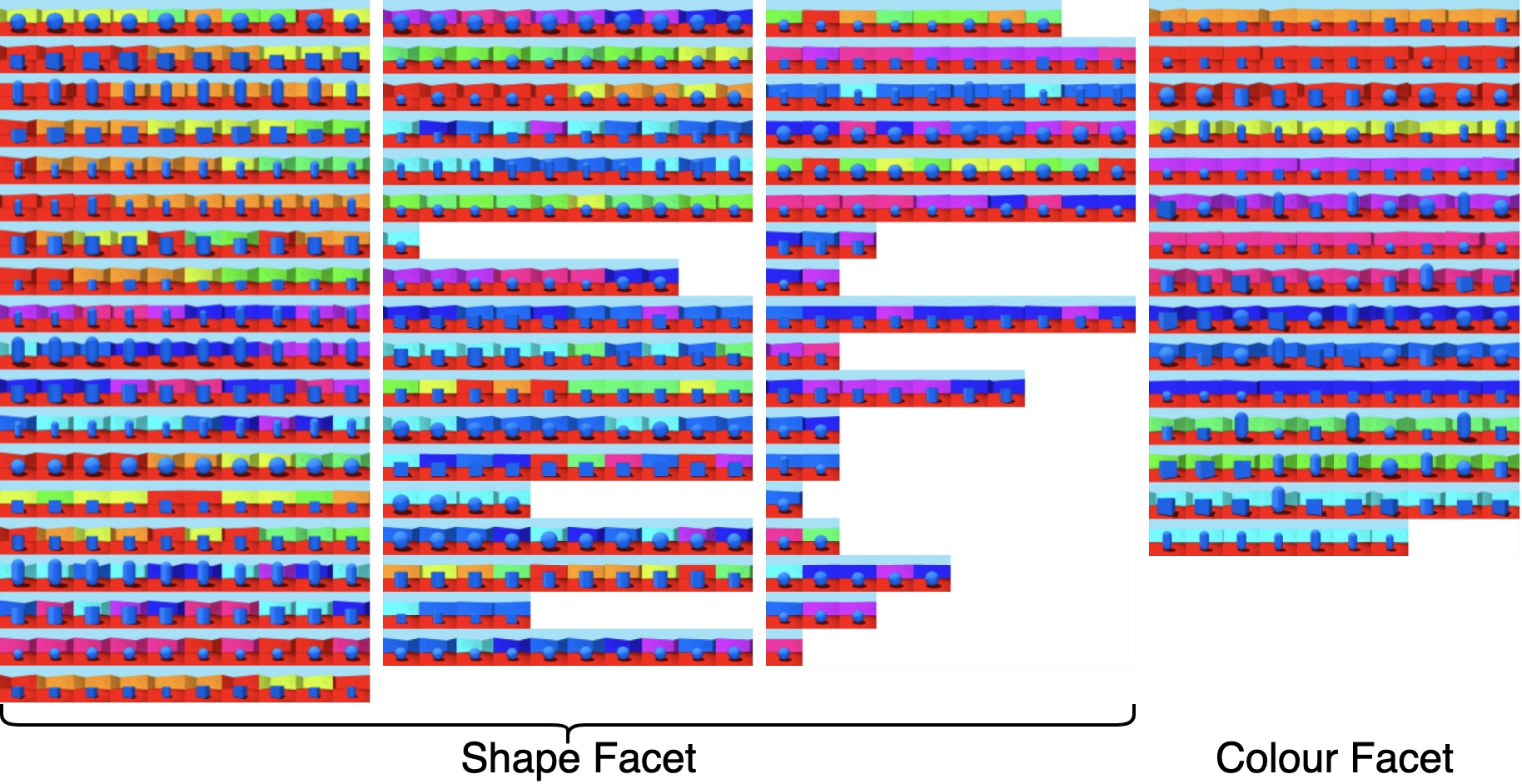}
    \caption{Input examples for clusters of MFCVAE with two-facets ($J=2$) trained on 3DShapes (configuration 2). Rows and columns are sorted as in Fig.~\ref{fig:orth_cluster_examples}.}
    \label{fig:orth_clusters_full_3dshapes_2}
\end{figure}

\begin{figure}[h!]
    \centering
    \includegraphics[width=\linewidth]{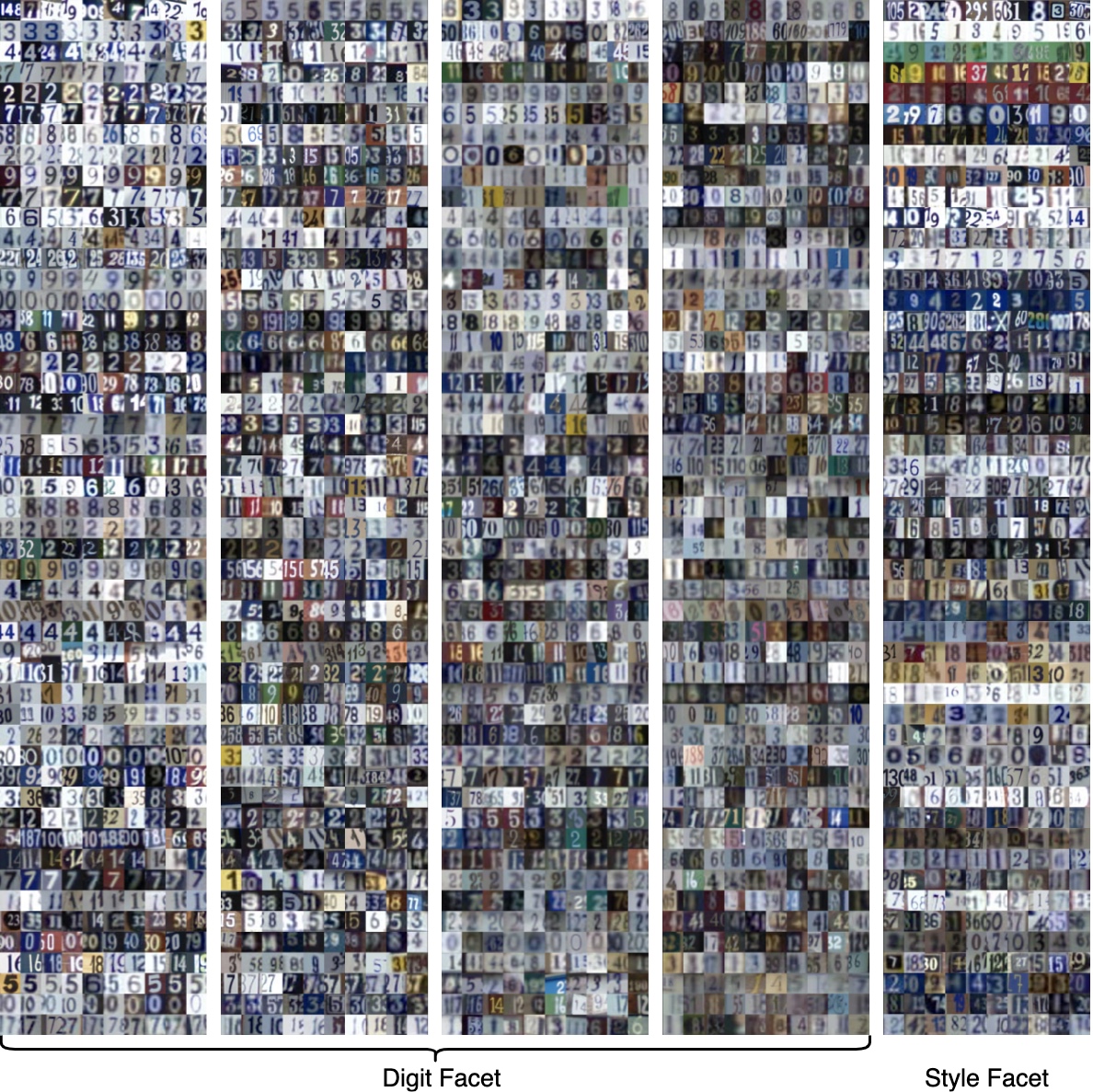}
    \caption{Input examples for clusters of MFCVAE with two-facets ($J=2$) trained on SVHN. Rows and columns are sorted as in Fig.~\ref{fig:orth_cluster_examples}.}
    \label{fig:orth_clusters_full_svhn}
\end{figure}

\clearpage
\subsection{Compositionality of facets}
\label{app:Compositionality of facets}

In this appendix, we provide further combinations of clusters of which the style facet is swapped and give a more rigorous explanation of the swapping procedure applied.

Let us have two input examples $\v{x}^{(1)}$ and $\v{x}^{(2)}$ assigned to two different style clusters according to Eq.~\eqref{eq:opt_q} (and typically two different digit clusters), i.e. $c_j^{(i)} = \mathrm{argmax}_{c_j}  \v{\pi}_j( c_j | q_\phi(\v z_j | \v x^{(i)}))$ and $c_j^{(1)} \neq c_j^{(2)}$, where $j=1$ for MNIST and $j=2$ for 3DShapes and SVHN. 
We can obtain the latent representations $\tilde{\v{z}}_j^{(i)}$ of input examples for both facets, taking the mode of $q_\phi(\v{z}_j | \v{x}^{(i)})$, respectively, which is parameterised via a forward pass.
Now, we swap the style/colour facet's latent representation ($\tilde{\v{z}}_1^{(i)}$ for MNIST, and $\tilde{\v{z}}_2^{(i)}$ for 3DShapes and SVHN) between the two inputs, while fixing the digit/shape facet's latent representation ($\tilde{\v{z}}_2^{(i)}$ for MNIST, and $\tilde{\v{z}}_1^{(i)}$ for 3DShapes and SVHN). 
Once the swapping is complete, we pass these latent representations through the decoder of our model to obtain reconstructions of our model from ``swapped style'' latent representations.
Note that this swapping operation is symmetric in the two-facet case, in the sense that the resulting two reconstructions are the same regardless of whether we swap $\tilde{\v{z}}_1$ or $\tilde{\v{z}}_2$.
Formally, we obtain reconstructions $\hat{\v{x}}^{(1)} = f(\{ \tilde{\v{z}}_1^{(1)}, \tilde{\v{z}}_2^{(2)} \}; {\theta})$ and $\hat{\v{x}}^{(2)} = f(\{ \tilde{\v{z}}_1^{(2)}, \tilde{\v{z}}_2^{(1)} \}; {\theta})$. 
It is such 4-tuples of inputs $\v{x}^{(1)}$ and $\v{x}^{(2)}$ and reconstructions $\hat{\v{x}}^{(1)}$ and $\hat{\v{x}}^{(2)}$ that we visualise. 

In Fig.~\ref{app_fig:compositionality}, we visualise further 4-tuples from all datasets (and configurations), in addition to the results presented in Section~\ref{sec:Compositionality of latent facets}.
For each dataset, the first 4-tuple consists of the first element of both input rows (left and right) and both reconstruction rows (left and right).
Note that we limit ourselves here to few digit/shape and colour/style cluster combinations. 
However, we note that for MNIST and 3DShapes, many other combinations could be found where a similar ``style swapping effect'' can be observed.
For the much richer dataset SVHN on which our model is less well fitted, while we can find cluster combinations where compositionality of facets works somewhat well (see Fig.~\ref{app_fig:compositionality} (d)), we can also find failure cases (see e.g. Fig.~\ref{app_fig:compositionality} (e)). 
Here, the style, represented as a white background, shall be composed with a white digit. 
We hypothesise that as this combination is particularly rare in the real world, the model struggles to reconstruct these latent combinations and as a consequence introduces undesired artefacts into the reconstructions.

\begin{figure}[h]
    \centering
    \includegraphics[width=\linewidth]{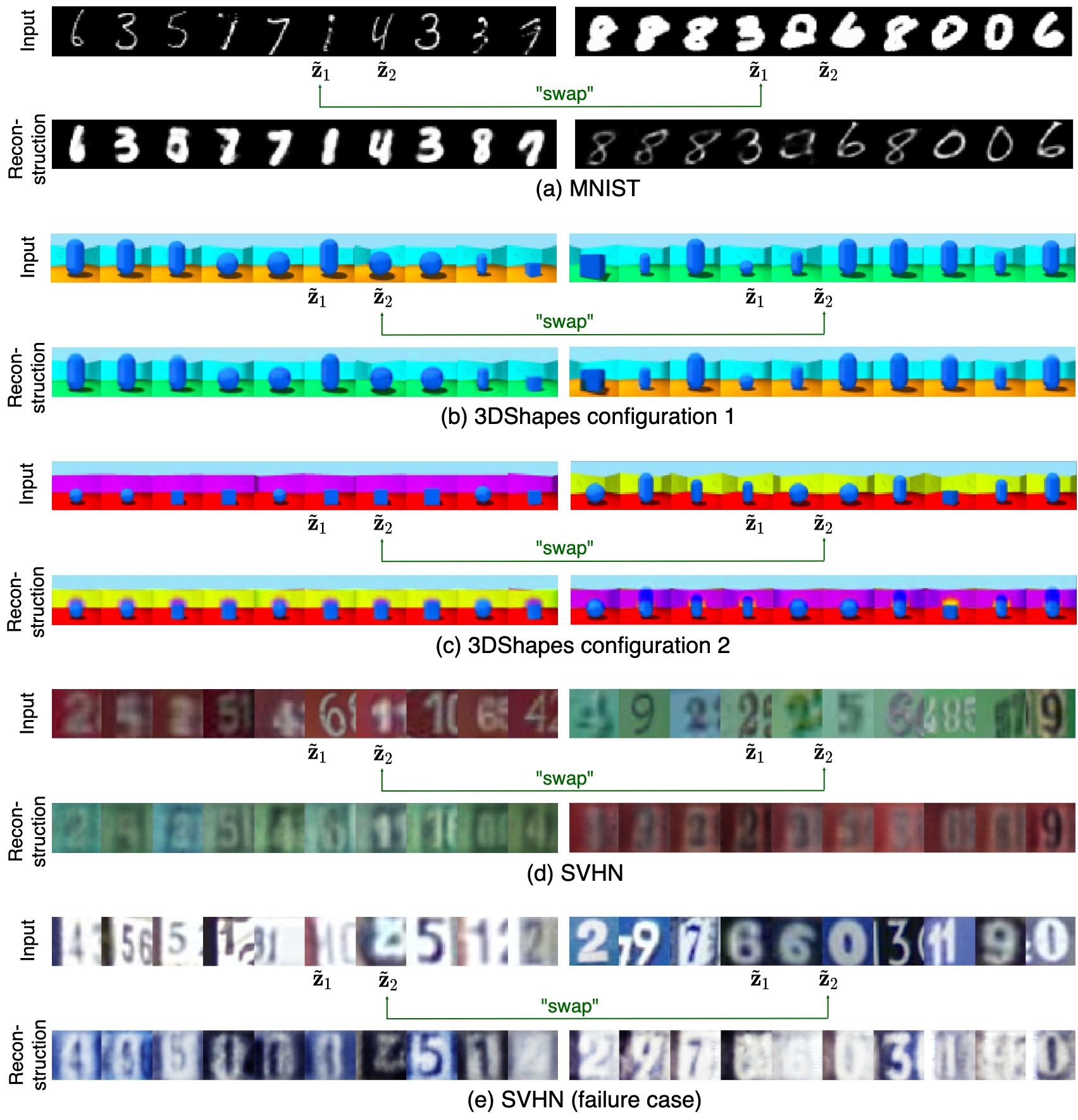}
    \caption{Input examples and reconstructions when swapping their style/colour facet's latent representation.}
    \label{app_fig:compositionality}
\end{figure}

\clearpage
\subsection{Generative, unsupervised classification}
\label{app:Generative, unsupervised classification}

To compare against assumed ground-truth clusterings imposed by the supervised class structure in our datasets, we report generative classification performance in terms of \textit{unsupervised clustering accuracy} on the test set. 
When the number of clusters in a facet is equal to the number of ground-truth classes compared against, one can use the Hungarian algorithm to find the optimal 1-to-1 mapping between clusters and classes~\cite{vade, kuhn1955hungarian}.
When the number of clusters in a facet is greater than the number of ground-truth classes compared against, as is common, one can simply assign each cluster in a facet to the most frequent ground-truth class found within that cluster~\cite{willetts2019disentangling}.

In Figs.~\ref{app_fig:test_acc_MNIST_SVHN} to \ref{app_fig:test_acc_3DShapes_configuration_2}, we plot generative classification performance on the test set measured in terms of unsupervised clustering accuracy over the training epochs. 
The blue shade is bounded by the mean accuracy plus and minus one standard deviation over the ten runs.
The sudden jumps of accuracy after $\approx 100$ epochs are caused by the progressive training algorithm which loops in a new facet at that point.

\clearpage

\begin{figure}[p]
    \centering
    \includegraphics[width=0.47\linewidth]{Figures/Appendix/test_acc_MNIST_final_sweep.png}
    \includegraphics[width=0.47\linewidth]{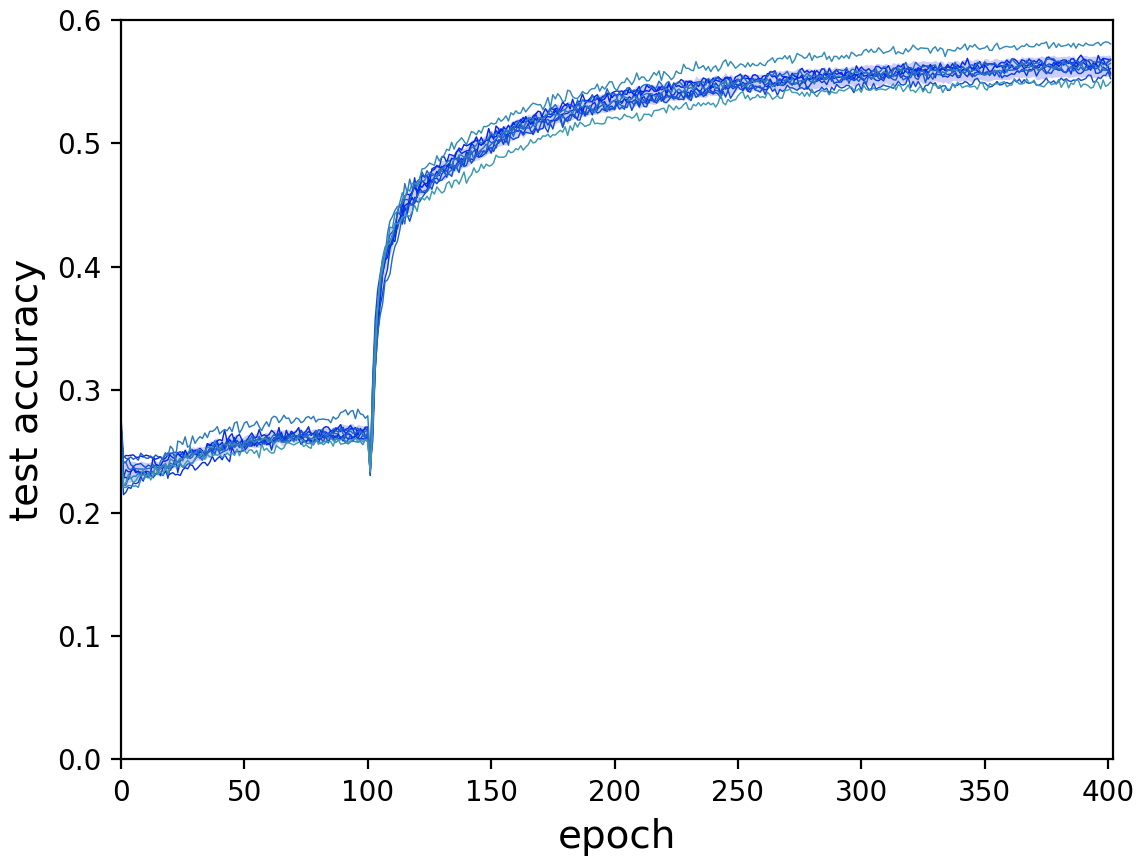}
    \caption{Unsupervised clustering accuracy on the test set w.r.t. the supervised label, when trained on [Left] MNIST, and [Right] SVHN.}
    \label{app_fig:test_acc_MNIST_SVHN}
\end{figure}

\begin{figure}[p]
    \centering
    \includegraphics[width=0.47\linewidth]{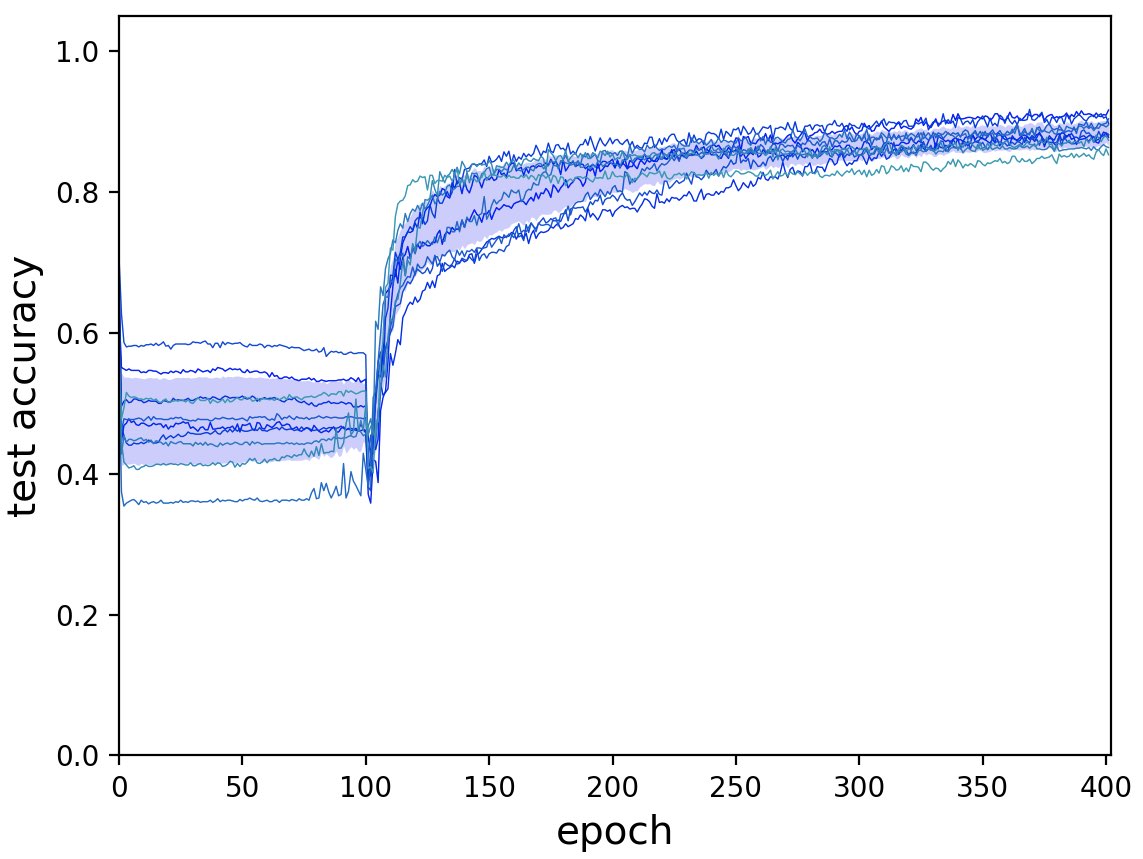}
    \includegraphics[width=0.47\linewidth]{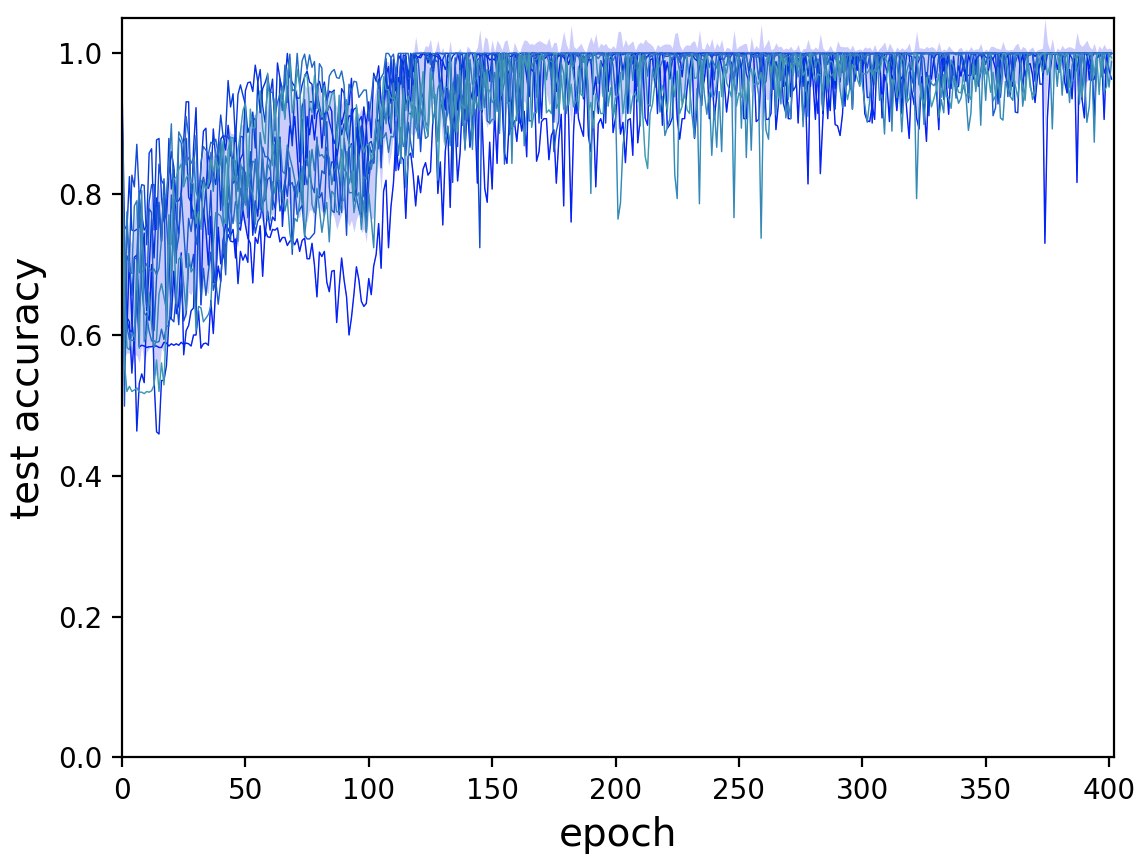}
    \caption{Unsupervised clustering accuracy on the test set for [Left] object shape and [Right] floor colour, when trained on 3DShapes (config. 1).}
    \label{app_fig:test_acc_3DShapes_configuration_1}
\end{figure}

\begin{figure}[p]
    \centering
    \includegraphics[width=0.47\linewidth]{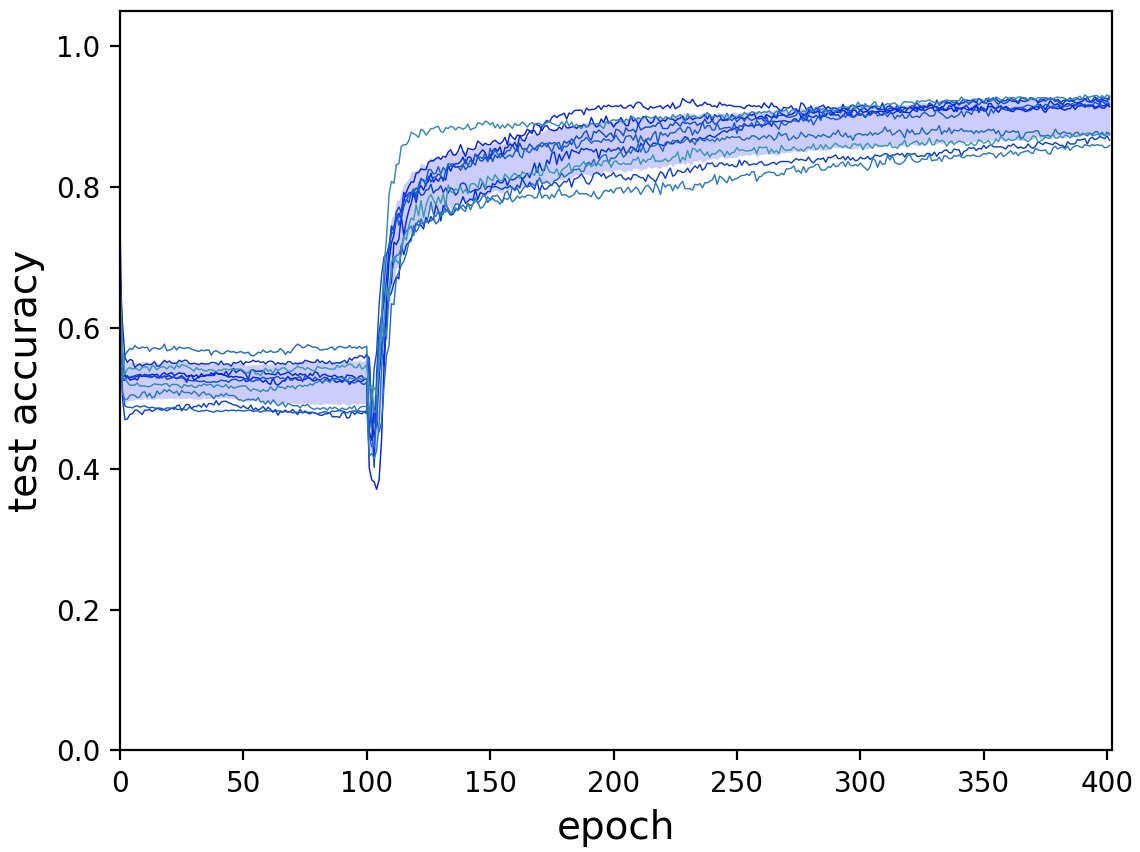}
    \includegraphics[width=0.47\linewidth]{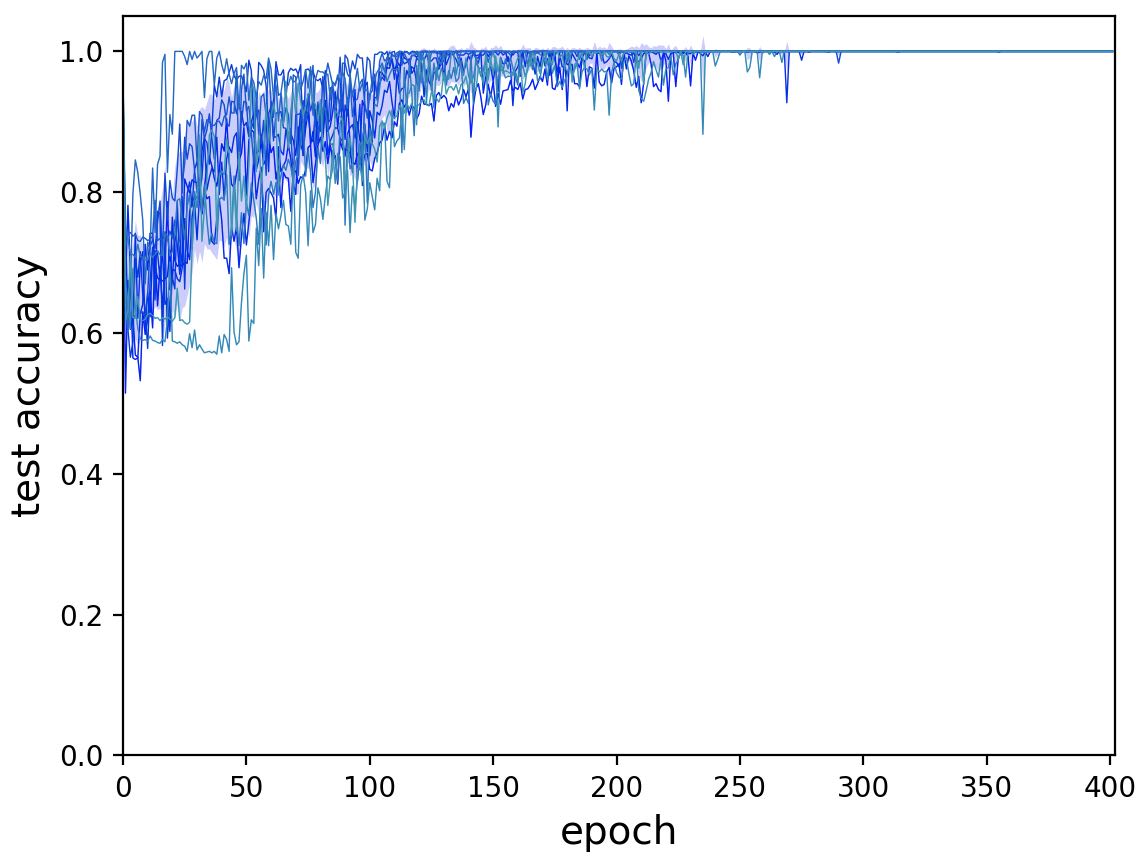}
    \caption{Unsupervised clustering accuracy on the test set for [Left] object shape and [Right] wall colour, when trained on 3DShapes (config. 2).}
    \label{app_fig:test_acc_3DShapes_configuration_2}
\end{figure}

\clearpage

\subsection{Diversity of generated samples}
\label{app:Diversity of generated samples}

First, we compare sample generation performance on MNIST between MFCVAE (a) and LTVAE (b) in Fig.~\ref{app_fig:sample_generation_full_MNIST_with_LTVAE}.
We sample MFCVAE as discussed in Section~\ref{sec:Diversity of generated samples}.
For LTVAE, while not fully clear, LTVAE samples from one Mixture-of-Gaussian, i.e. first samples from a categorical, then from the chosen component, to obtain $\v z$ (\cite{ltvae}, Section 4.5).
This sampling procedure resembles MFCVAE's reconstructions that illustrate its digit facet (e.g. Fig.~\ref{app_fig:sample_generation_full_MNIST_with_LTVAE} (a.1)).
When comparing (a.1) and (b), it can be observed that generation performance is comparable between the two models, with each row representing a certain digit identity cluster.
However, as shown in (a.2) and not demonstrated by LTVAE, our model additionally allows sample generation conditional on style clusters.
This demonstrates the advantage of MFCVAE in its intervention capability for each facet during sample generation, allowing a rich set of options for potential downstream tasks.

To quantitatively compare the sampling diversity of MFCVAE against VaDE, we compute the Learnt Perceptual Image Patch Similarity (LPIPS) \cite{zhang2018unreasonable} of 60,000 samples generated from models trained on MNIST and SVHN. 
We also report LPIPS computed on on real (i.e. non-synthetic) images, where we use the data from both training set and test set.
For LPIPS, higher is better, while it must be ensured that the true data distribution is modelled (and not just a distribution that  artificially maximises the metric).
As shown in Table \ref{app_tab:diversity}, our method generates samples with similar diversity to VaDE on both datasets.
Further, MFCVAE is close to the ``real image diversity'' for MNIST, yet is somewhat lower for SVHN.

\begin{table}[h]
\caption{LPIPS of real images and 60,000 samples generated from VaDE and MFCVAE for MNIST and SVHN.}
\medbreak
\centering
\begin{tabular}{ ccc  } \toprule
  &  MNIST & SVHN \\ \midrule
Real Images  & 0.112 & 0.227 \\
VaDE & 0.111 & 0.187 \\
MFCVAE (ours) & 0.116 & 0.182 \\ \bottomrule
\end{tabular}
\label{app_tab:diversity}
\end{table}

In addition, Figs.~\ref{app_fig:sample_generation_full_3DShapes_1} to \ref{app_fig:sample_generation_full_SVHN} show the complete plots of synthetic samples generated from MFCVAE where all clusters are visualised compared to the main text.
Again, we refer to Section~\ref{sec:Diversity of generated samples} for an explanation on the procedure of how these samples are generated, but provide here additional details: 
During sample generation, the variance of the distributions over latent variables $\v{z}_1$ and $\v{z}_2$ are scaled by a ``temperature'' factor $\tau > 0$. 
This is a common technique in likelihood-based deep generative models to improve the quality of generated samples~\cite{kingma2018glow, parmar2018image}. 
To formalise this, at sampling time, the covariance matrix $\tau \Sigma_{c_j}$ is used for $p(\v{z}_j | c_j)$, instead of  $\Sigma_{c_j}$.
In this set of experiments, temperature scaling is used for 3DShapes and SVHN, where we choose $\tau = 0.3$.
For MNIST, we do not use temperature scaling, i.e. $\tau = 1.0$.

For MNIST and 3DShapes, it can be observed that for clusters with a lower average assignment probability, synthetic samples remain homogeneous w.r.t. their characteristic value in each facet. 
For SVHN, the sample reconstruction quality drops for clusters with lower average assignment probabilities which we can attribute to a smaller separation of facets on this dataset as observed in Section~\ref{sec:Generative classification}.

\clearpage

\begin{figure}[p]
    \centering
    \includegraphics[width=.85\linewidth]{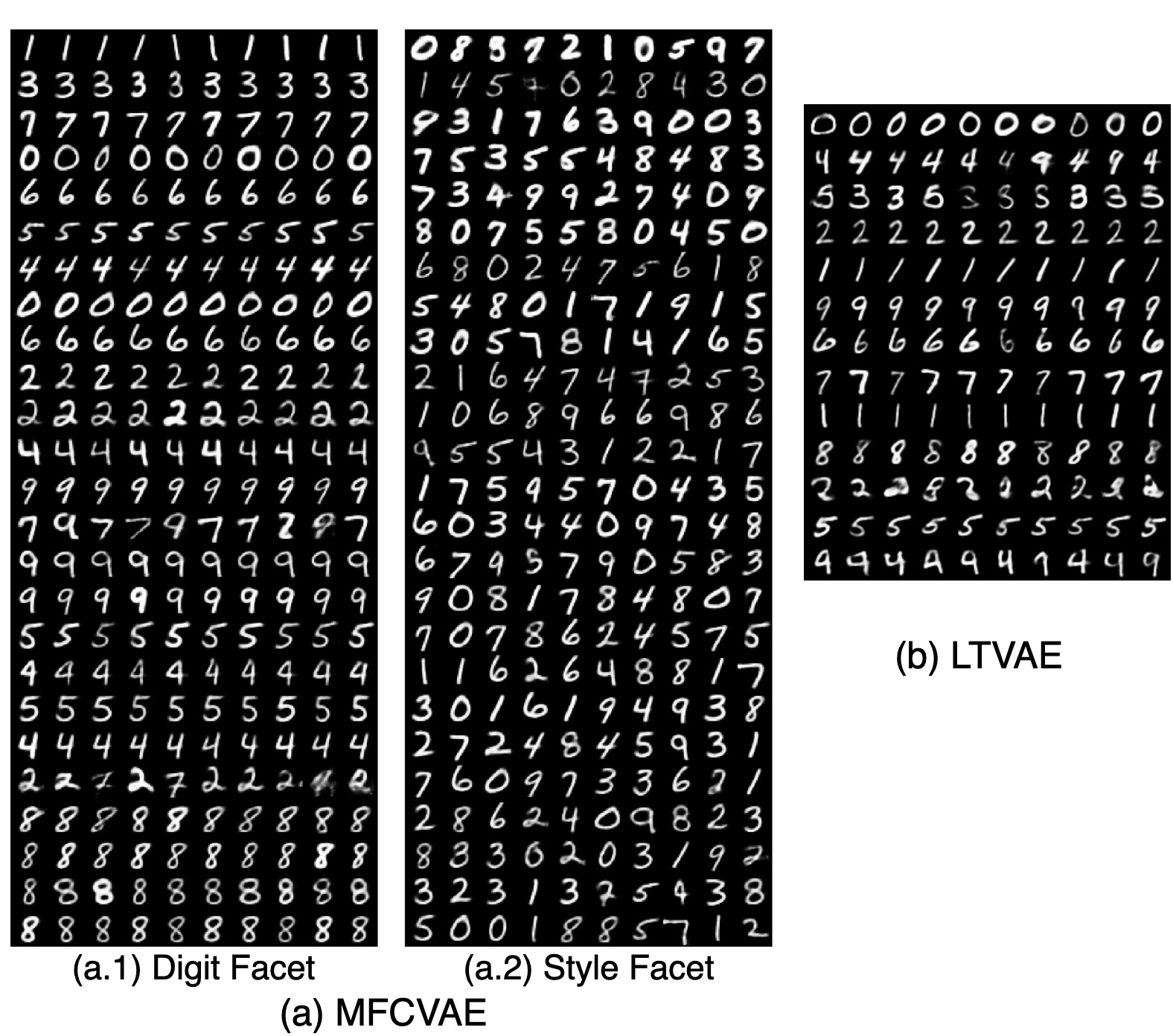}
    \caption{(a) Synthetic samples generated from MFCVAE with two facets ($J=2$) trained on MNIST, with all clusters visualised. 
    Rows are sorted as in Fig.~\ref{fig:sample_generation}.
    (b) Synthetic samples generated from LTVAE trained on MNIST. 
    Plot is taken as reported in \cite{ltvae}, Fig. 7.}
    \label{app_fig:sample_generation_full_MNIST_with_LTVAE}
\end{figure}

\clearpage

\begin{figure}[p]
    \centering
    \includegraphics[width=\linewidth]{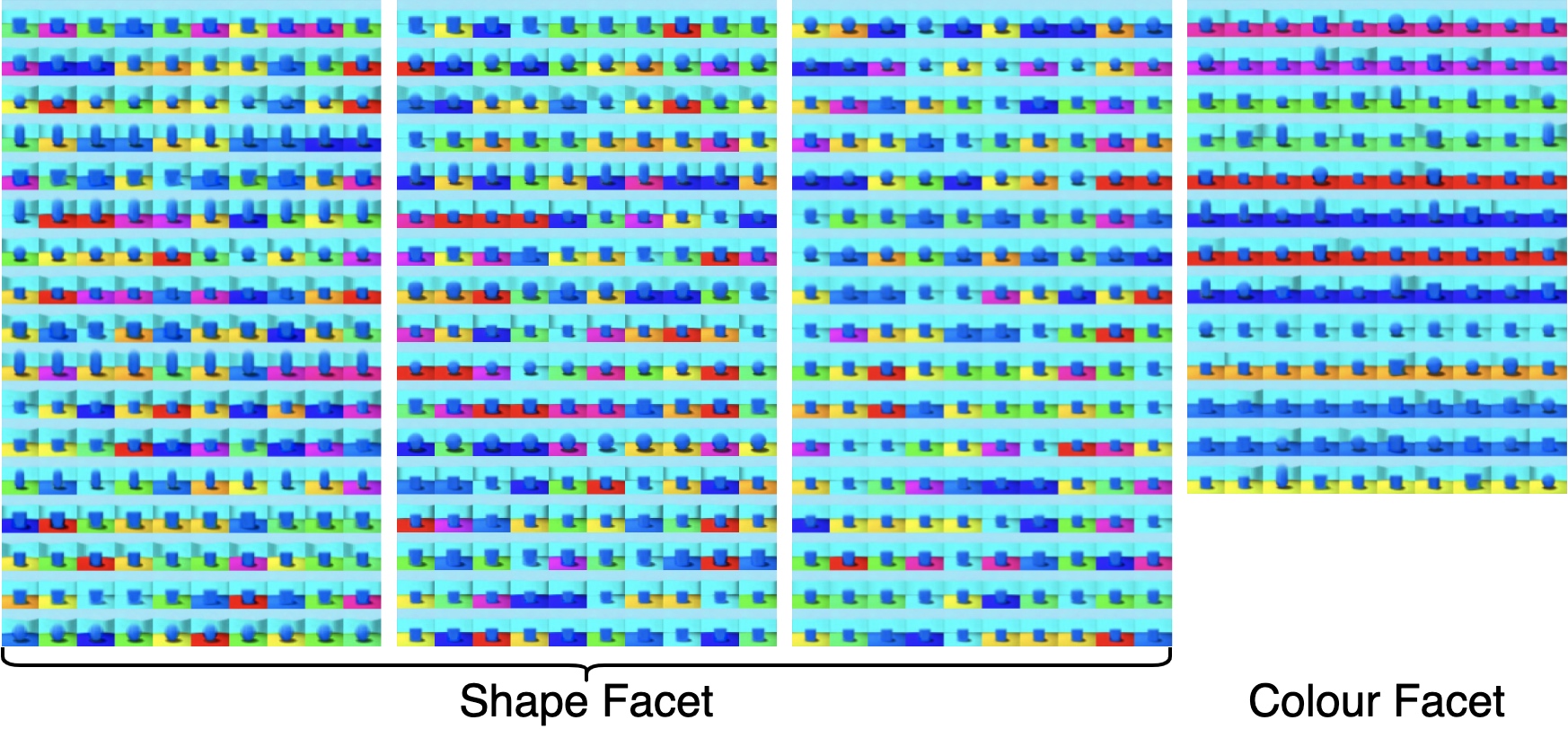}
    \caption{Synthetic samples generated from MFCVAE with two facets ($J=2$) trained on 3DShapes (configuration 1), with all clusters visualised. Rows are sorted as in Fig.~\ref{fig:sample_generation}.}
    \label{app_fig:sample_generation_full_3DShapes_1}
\end{figure}

\begin{figure}[p]
    \centering
    \includegraphics[width=\linewidth]{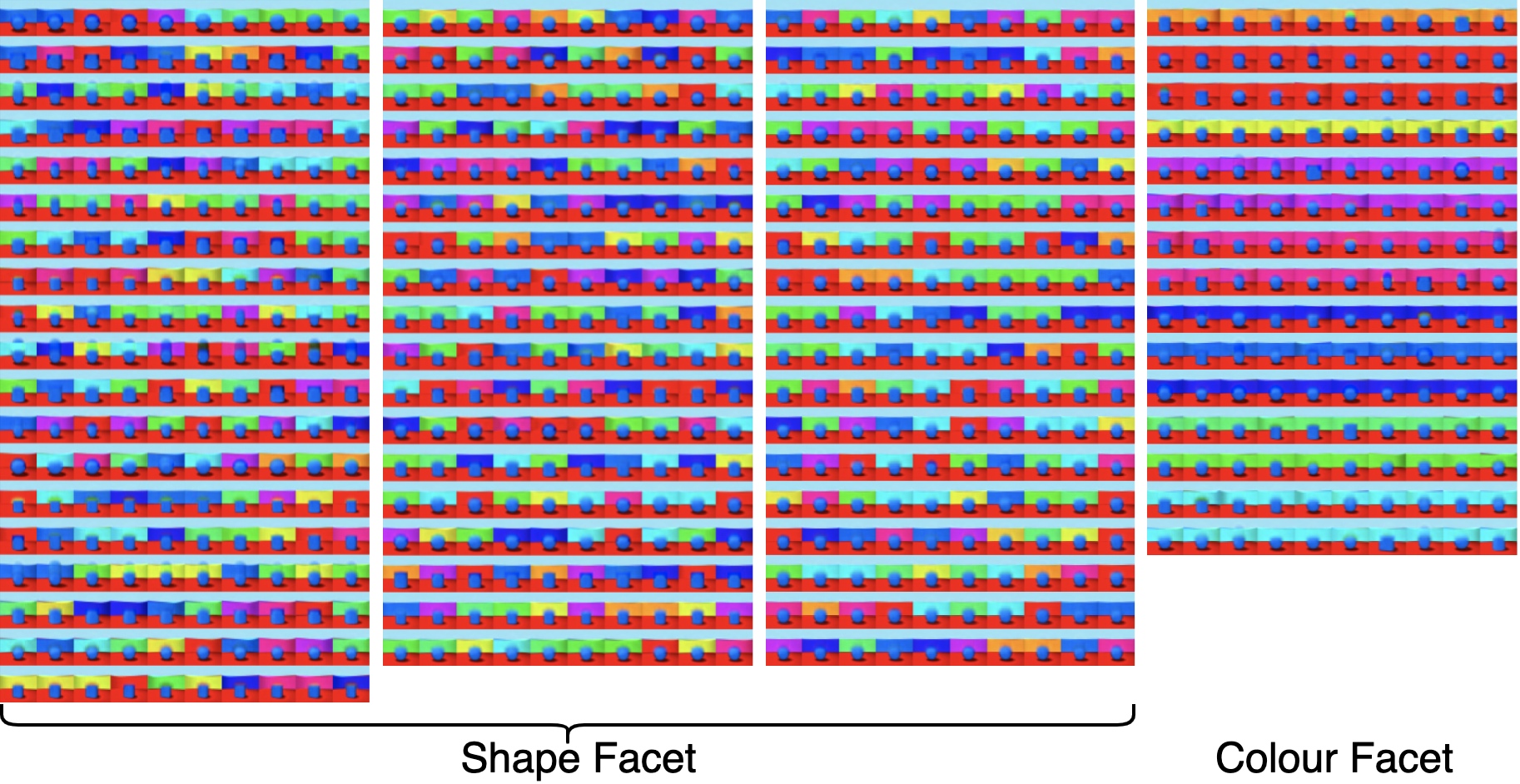}
    \caption{Synthetic samples generated from MFCVAE with two facets ($J=2$) trained on 3DShapes (configuration 2), with all clusters visualised. Rows are sorted as in Fig.~\ref{fig:sample_generation}.}
    \label{app_fig:sample_generation_full_3DShapes_2}
\end{figure}

\clearpage

\begin{figure}[p]
    \centering
    \includegraphics[width=\linewidth]{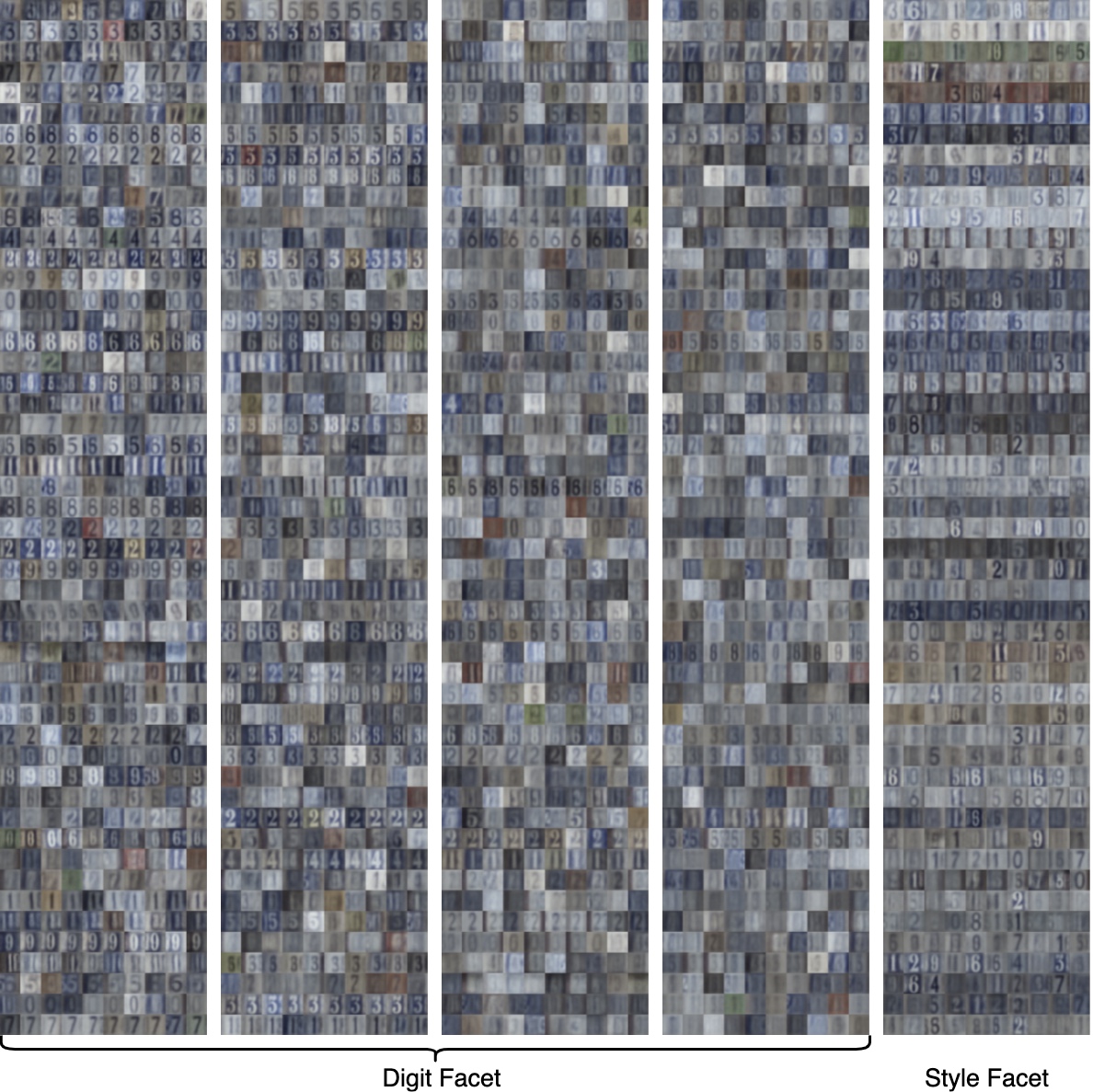}
    \caption{Synthetic samples generated from MFCVAE with two facets ($J=2$) trained on SVHN, with all clusters visualised. 
    Rows are sorted as in Fig.~\ref{fig:sample_generation}.}
    \label{app_fig:sample_generation_full_SVHN}
\end{figure}

\end{document}